%% file: main.tex
\documentclass{article}

\usepackage{natbib}
\usepackage{arxiv}

\input{macros}
\input{math_commands}

\usepackage[utf8]{inputenc} 
\usepackage[T1]{fontenc}    
\usepackage{url}            
\usepackage{booktabs}       
\usepackage{amsfonts}       
\usepackage{nicefrac}       
\usepackage{xcolor}         
\usepackage{subfigure}
\usepackage{amsthm,amssymb,amsmath}
\usepackage{listings}
\usepackage{cancel}
\usepackage{cleveref}
\usepackage{tikz}
\usepackage{bbm}
\usepackage{enumitem}
\usepackage{listings}
\usepackage{natbib}
\usepackage{geometry}
\usepackage[noend]{algorithmic}
\usepackage{algorithm}
\usepackage{algorithmic}
\usepackage{wrapfig}

\newtheorem{theorem}{Theorem}[section]

\newtheorem{lemma}[theorem]{Lemma}

\newtheorem{assumption}[theorem]{Assumption}

\crefdefaultlabelformat{#2#1#3}
\creflabelformat{equation}{#2#1#3}
\def\shownotes{1}  
\ifnum\shownotes=1
\newcommand{\authnote}[2]{[#1: #2]}
\else
\newcommand{\authnote}[2]{}
\fi

\def\shownotes{0}  
\ifnum\shownotes=1
\newcommand{\authnotev}[2]{[#1: #2]}
\else
\newcommand{\authnotev}[2]{}
\fi

\title{\ours: A Scalable Stochastic Second-order Optimizer for Language Model Pre-training}

\author{Hong Liu$\ \ \ \ \ $Zhiyuan Li$\ \ \ \ \ $David Hall$\ \ \ \ \ $Percy Liang$\ \ \ \ \ $Tengyu Ma \\
\\
	Stanford University \\
	\texttt{\{hliu99, zhiyuanli, dlwh, pliang, tengyuma\}@cs.stanford.edu}
}

\begin{document}
\maketitle
\input{intro}
\input{method}
\input{experiment}

\input{theoretical_analysis}
\input{relatedwork}

\bibliography{sample,all}
\bibliographystyle{icml2021}

\appendix
\input{appendix}

\end{document}

%% file: macros.tex
\newcommand{\clip}{\textup{clip}}
\newcommand{\ours}{Sophia}
\newcommand{\Cat}{\textup{Cat}}

%% file: math_commands.tex
\usepackage{amsmath,amsfonts,bm}

\def\1{\bm{1}}

\DeclareMathAlphabet{\mathsfit}{\encodingdefault}{\sfdefault}{m}{sl}
\SetMathAlphabet{\mathsfit}{bold}{\encodingdefault}{\sfdefault}{bx}{n}

\newcommand{\R}{\mathbb{R}}

\newcommand{\diag}{{\rm diag}}

\newcommand{\norm}[1]{\left\|#1\right\|}
\def\abs#1{\left| #1 \right|}

\newcommand{\inner}[2]{\left\langle #1,#2 \right\rangle}

\newcommand{\diff}{\mathrm{d}}

\newcommand{\eps}{\epsilon}



%% file: intro.tex
\newcommand{\name}{{Sophia}}
\begin{abstract}
  Given the massive cost of language model pre-training,
  a non-trivial improvement of the optimization algorithm would lead to a material reduction on the time and cost of training.
  Adam and its variants have been state-of-the-art for years, and more sophisticated second-order (Hessian-based) optimizers often incur too much per-step overhead. In this paper, we propose \ours, \textbf{S}econd-\textbf{o}rder Cli\textbf{p}ped Stoc\textbf{h}astic Opt\textbf{i}miz\textbf{a}tion, a simple scalable second-order optimizer that uses a light-weight estimate of the diagonal Hessian as the pre-conditioner. The update is the moving average of the gradients divided by the moving average of the estimated Hessian, followed by element-wise clipping. The clipping controls the worst-case update size and tames the negative impact of non-convexity and rapid change of Hessian along the trajectory.
  Sophia only estimates the diagonal Hessian every handful of iterations, which has negligible average per-step time and memory overhead. On language modeling with GPT models of sizes ranging from 125M to 1.5B, Sophia achieves a 2x speed-up compared to Adam in the number of steps, total compute, and wall-clock time, achieving the same perplexity with 50$\%$ fewer steps, less total compute, and reduced wall-clock time. 
  Theoretically, we show that Sophia, in a much simplified setting, adapts to the heterogeneous curvatures in different parameter dimensions, and thus has a run-time bound that does not depend on the condition number of the loss. 
\end{abstract}

\section{Introduction}
Language models (LLMs) have gained phenomenal capabilities as their scale grows~\citep{radford2019language,kaplan2020scaling,brown2020language,  zhang2022opt,   
 touvron2023llama,openai2023gpt}. However, pre-training LLMs is incredibly time-consuming due to the massive datasets and model sizes---hundreds of thousands of updates to the model parameters are required. For example, PaLM was trained for two months on 6144 TPUs, which costed 10 million dollars~\citep{chowdhery2022palm}. 

Pre-training efficiency is thus a major bottleneck in scaling up LLMs. 
This work aims to improve pre-training efficiency with a faster optimizer, which either reduces the time and cost to achieve the same pre-training loss, or alternatively achieves better pre-training loss with the same budget.

Adam~\citep{kingma2014Adam} (or its variants~\citep{loshchilov2017decoupled,shazeer2018adafactor,you2019large}) is the dominantly used optimizer for training LLMs, such as GPT~\citep{radford2019language,brown2020language}, OPT~\citep{zhang2022opt}, Gopher~\citep{rae2021scaling} and LLAMA~\citep{touvron2023llama}. 
Designing faster optimizers for LLMs is challenging. First, the benefit of the first-order (gradient-based) pre-conditioner in Adam is not yet well understood~\citep{ liu2020understanding,zhang2020adaptive,kunstner2023noise}. Second, the choice of pre-conditioners is constrained because we can only afford light-weight options whose overhead can be offset by the speed-up in the number of iterations. For example, the block-diagonal Hessian pre-conditioner in K-FAC is  expensive for LLMs~\citep{martens2015optimizing,grosse2016kronecker, ba2017distributed,martens2018kronecker}. On the other hand, ~\citet{chen2023symbolic} automatically search among the light-weight gradient-based pre-conditioners and identify Lion, which is substantially faster than Adam on vision Transformers and diffusion models but only achieves limited speed-up on LLMs~\citep{chen2023symbolic}.

This paper introduces \ours, \textbf{S}econd-\textbf{o}rder Cli\textbf{p}ped Stoc\textbf{h}ast\textbf{i}c Optimiz\textbf{a}tion, a light-weight second-order optimizer that uses an inexpensive stochastic estimate of the diagonal of the Hessian as a pre-conditioner and a clipping mechanism to control the worst-case update size. On pre-training language models such as GPT-2, {\ours} achieves the same validation pre-training loss with 50$\%$ fewer number of steps than Adam. Because {\ours} maintains almost the memory and average time per step, the speedup also translates to 50$\%$ less total compute and 50$\%$ less wall-clock time (See Figure~\ref{fig:head} (a)\&(b)). We also note that comparing the run-time to achieve the same loss is a correct way to compare the speed of optimizers for LLMs; see Section~\ref{sec:evaluation} for more details.

Moreover, the scaling law based on model size from 125M to 770M is in favor of Sophia over Adam---the gap between Sophia and Adam with 100K steps increases as the model size increases (Figure~\ref{fig:head} (c)). 
In particular, Sophia on a 540M-parameter model with 100K steps gives the same validation loss as  Adam on a 770M-parameter model with 100K steps. Note that the latter model needs 40\% more training time and 40\% more inference cost.  

\begin{figure}[t]
\begin{center}
\includegraphics[width=0.4\textwidth]{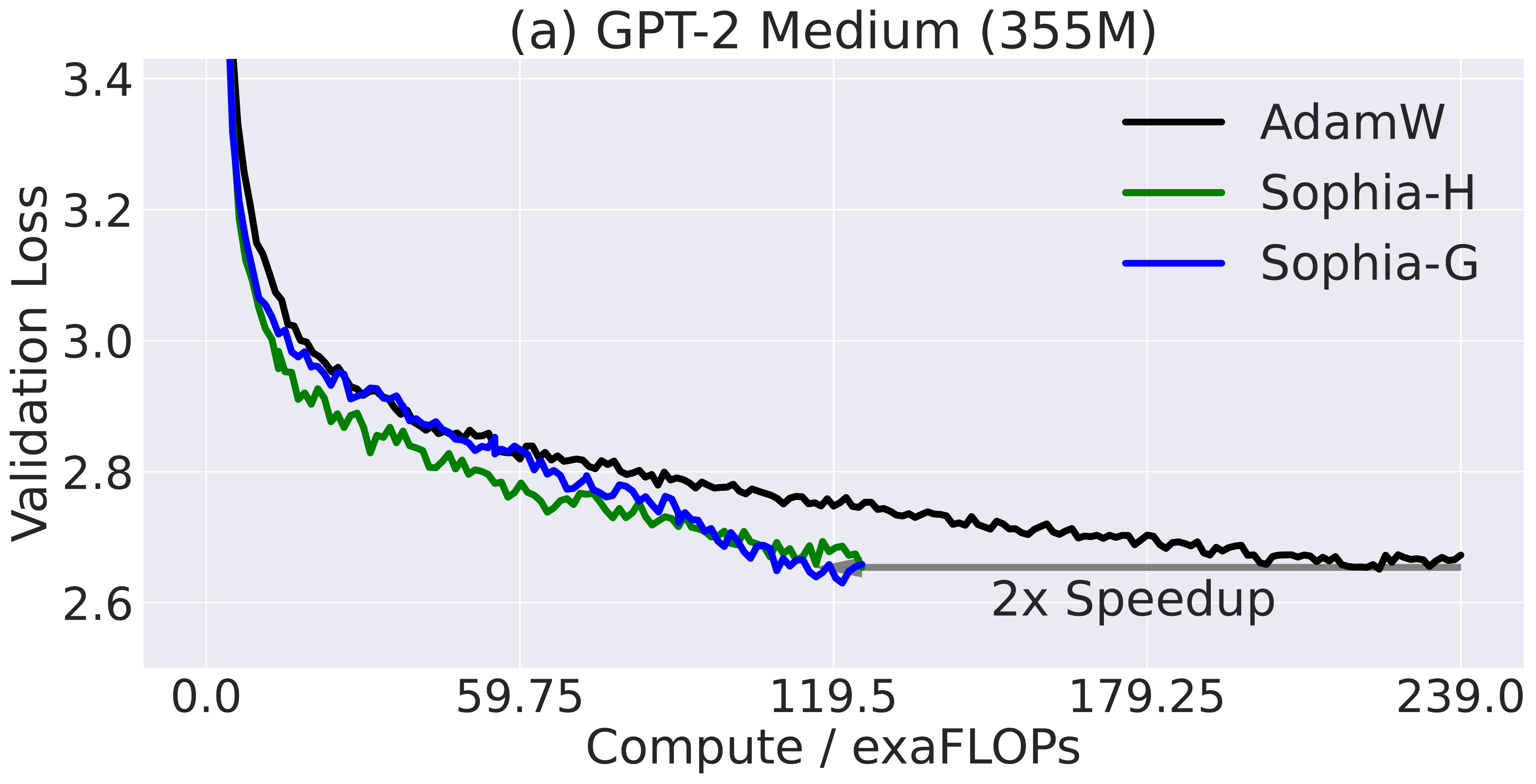}
\hspace{15pt}
\includegraphics[width=0.4\textwidth]{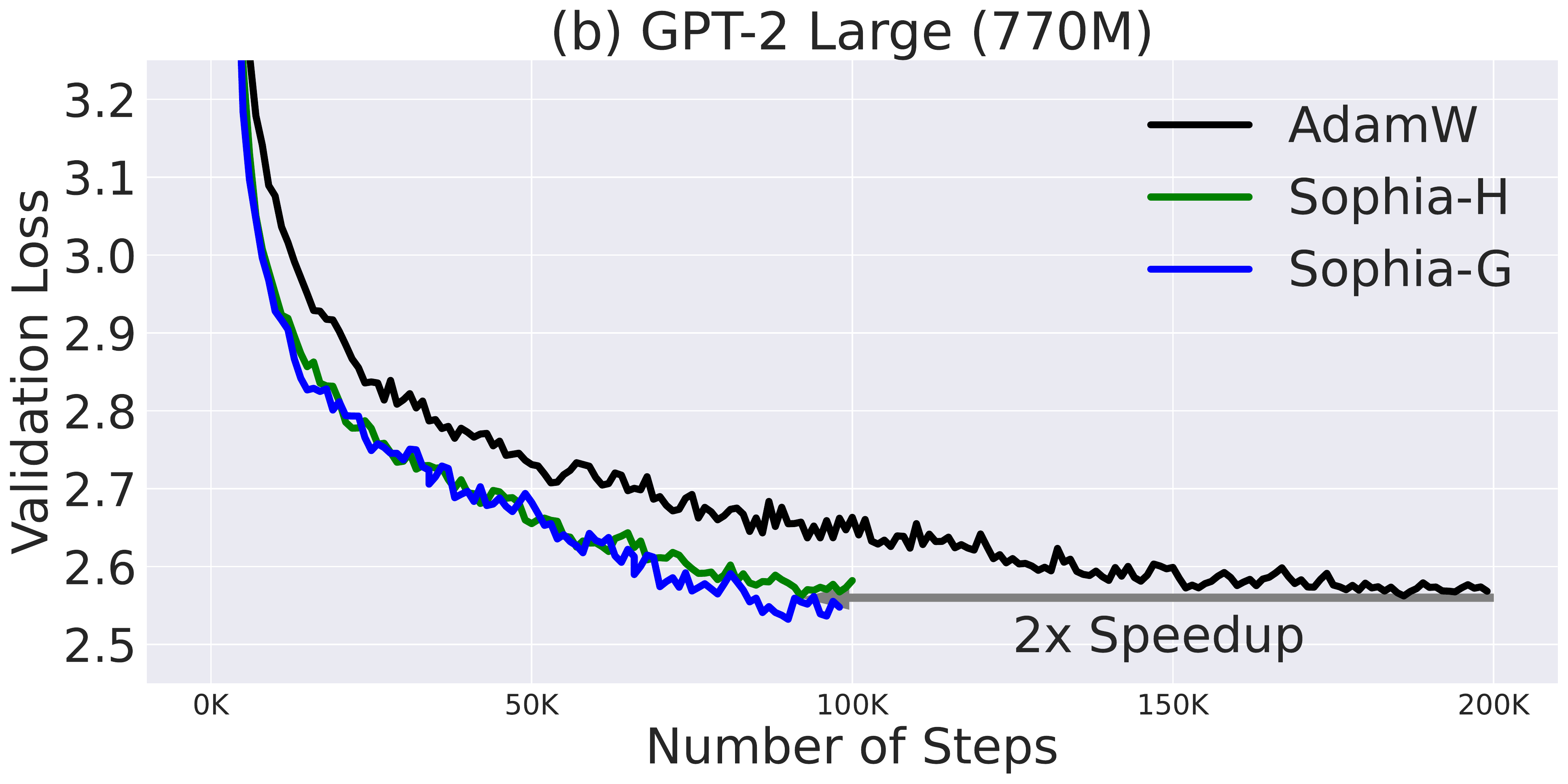}
\includegraphics[width=0.4\textwidth]{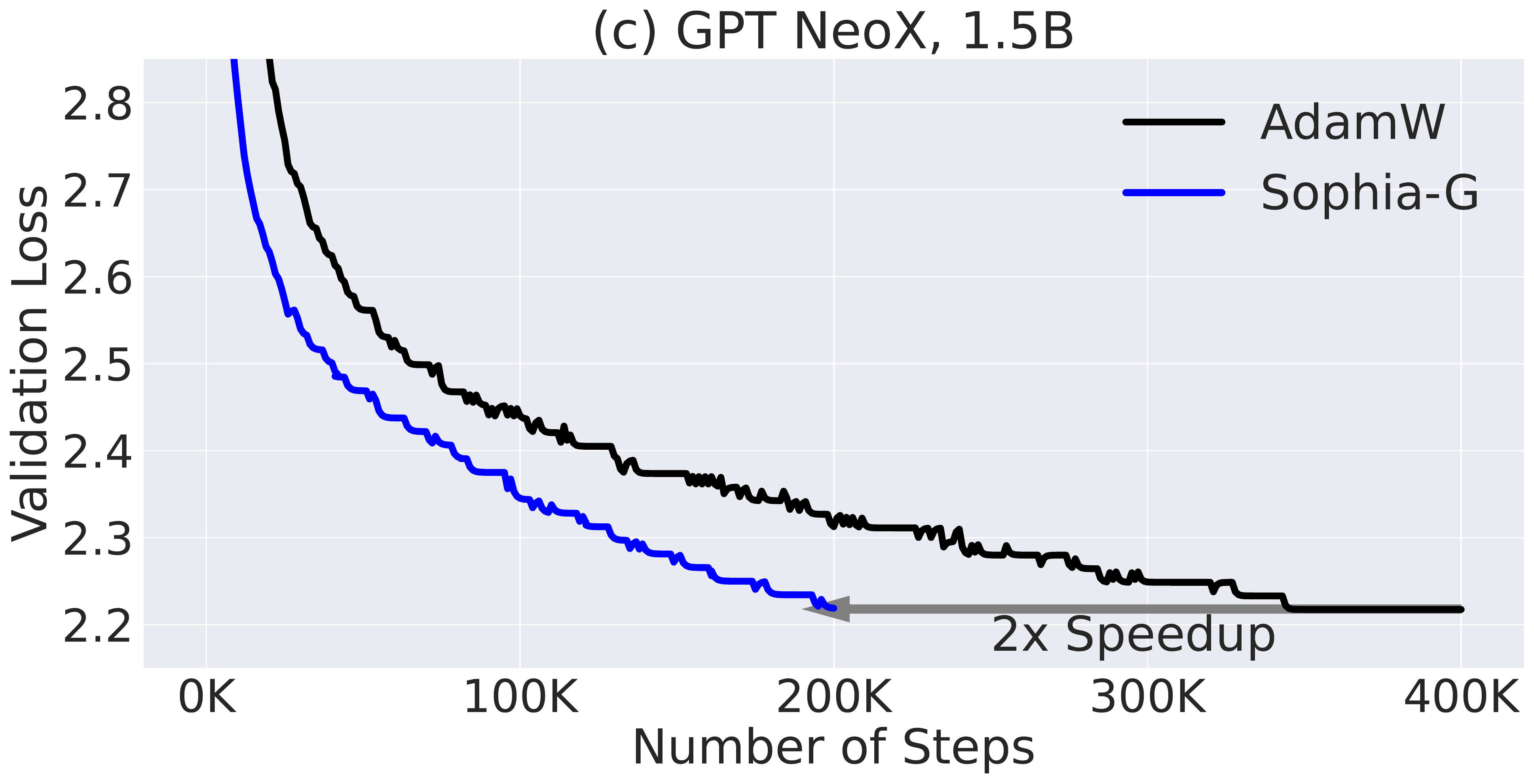}
\hspace{15pt}
\includegraphics[width=0.4\textwidth]{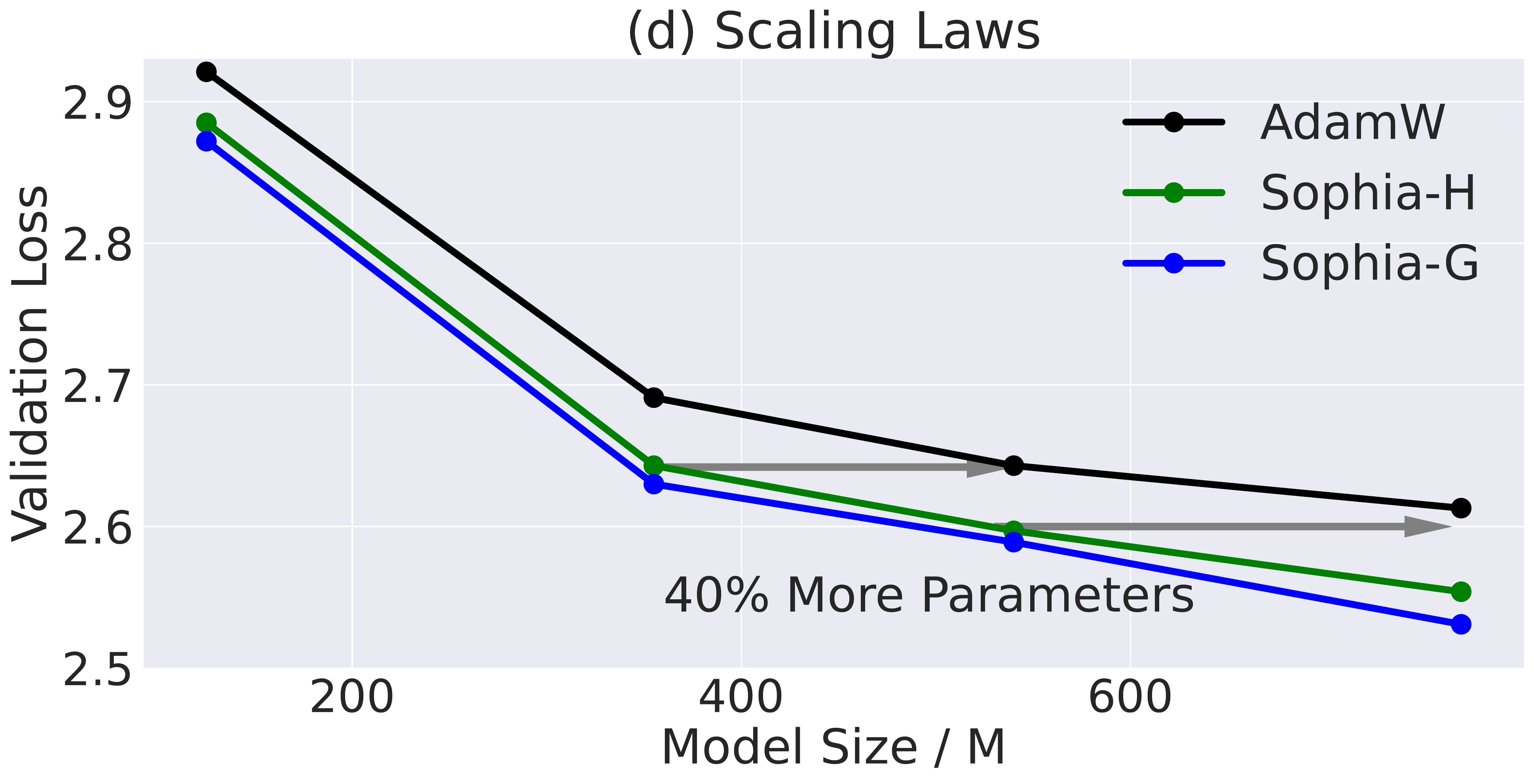}
\caption{{{\ours} achieves a 2x speedup over AdamW in GPT-2 pre-trained on OpenWebText and GPT NeoX pre-trained on the Pile. (a) (b) (c) Comparison of the number of steps needed to achieve the same level of validation loss on (a) GPT-2-medium (355M), (b) GPT-2-large (770M) and (c) GPT NeoX 1.5B. Across all model sizes, {\ours} needs 50$\%$ less time to reach the same validation loss as AdamW. (d) Validation losses of models with different sizes pre-trained for 100K steps. The gap between {\ours} and AdamW gets larger as models size grows. Notably, using {\ours} on a 540M-parameter model for 100K steps results in the same loss as using AdamW on a 770M-parameter model for 100K steps. See Section~\ref{sec:experiments} for details and more results. \label{fig:head}
}
}
\end{center}
\end{figure}

\begin{figure}
\begin{minipage}[b]{.46\textwidth}
\begin{minipage}[b]{\textwidth}
\begin{algorithm}[H]
	\caption{$\textup{Hutchinson}(\theta)$}
	\label{alg:sub}
	\begin{algorithmic}[1]
		\STATE {\bfseries Input:} parameter $\theta$.
            \STATE Compute mini-batch loss $L(\theta)$.
		\STATE Draw $u$ from $\mathcal{N}(0,\mathrm{I}_d)$.
        \RETURN $u \odot \nabla (\langle\nabla L(\theta), u\rangle)$.
	\end{algorithmic}
\end{algorithm}
\end{minipage}
\centering
\includegraphics[width=1.0\textwidth]{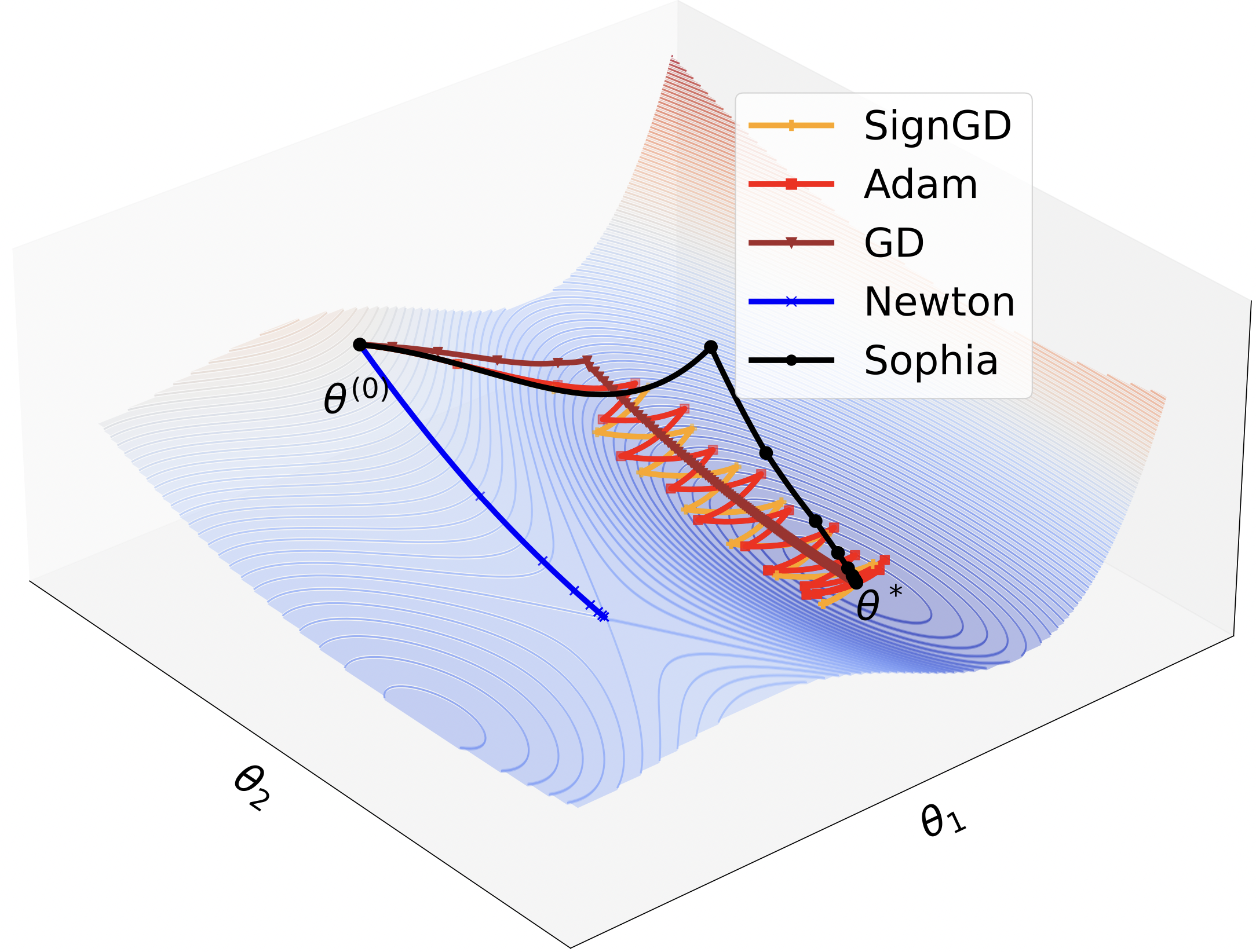}
\vspace{-20pt}
\caption{The motivating toy example. $\theta_{[1]}$ is the sharp dimension and $\theta_{[2]}$ is the flat dimension. GD's learning rate is limited by the sharpness in $\theta_1$, and makes slow progress along $\theta_{[2]}$. Adam and SignGD bounce along $\theta_{[1]}$ while making slow progress along $\theta_{[2]}$. Vanilla Newton's method converges to a saddle point. {\ours} makes fast progress in both dimensions and converges to the minimum with a few steps.\label{fig:toy}}

\vspace{-10pt}

\end{minipage}\hfill
\begin{minipage}[b]{.49\textwidth}
\begin{minipage}[b]{\textwidth}
\begin{algorithm}[H]
	\caption{ $\textup{Gauss-Newton-Bartlett}(\theta)$}
	\label{alg:sub_gn}
	\begin{algorithmic}[1]
		\STATE {\bfseries Input:} parameter $\theta$.
		\STATE Draw a mini-batch of input $\{x_b\}_{b=1}^{B}$.
            \STATE Compute logits on the mini-batch: $\{f(\theta, x_b)\}_{b=1}^{B}$.
            \STATE Sample $\hat{y}_b\sim \textup{softmax}(f(\theta, x_b)), \forall b\in [B]$.
            \STATE Calculate $\hat{g} = \nabla(1/B\sum\ell(f(\theta, x_b),\hat{y}_b))$.
        \RETURN $B\cdot \hat{g} \odot \hat{g}$.
	\end{algorithmic}
\end{algorithm}
\end{minipage}
\begin{minipage}[b]{\textwidth}
\begin{algorithm}[H]
	\caption{\ours~}
	\label{alg:hess_opt}
	\begin{algorithmic}[1]
		\STATE {\bf Input:} $\theta_1$, learning rate $\{\eta_t\}_{t=1}^T$,  hyperparameters $\lambda,\gamma, \beta_{1},\beta_{2}$, $\epsilon$, and estimator choice Estimator $\in\{$Hutchinson, Gauss-Newton-Bartlett$\}$\STATE Set $m_{0} = 0$, $v_{0} = 0$, $h_{1-k} = 0$
		\FOR{$t=1$ {\bf to} $T$}
		\STATE Compute minibach loss $L_t(\theta_t)$.
        \STATE Compute $g_t = \nabla L_t(\theta_t)$.
		\STATE  $m_{t} = \beta_{1} m_{t-1} + (1 - \beta_{1}) g_{t}$ 
        \IF{ $t\ \mathrm{mod}\  k = 1$}
        \STATE Compute $\hat{h}_{t} = \textup{Estimator}(\theta_t)$.
		\STATE  $h_t = \beta_2 h_{t-k} + (1 - \beta_2) \hat{h}_{t}$
        \ELSE \STATE  $h_t = h_{t-1}$
        \ENDIF
        \STATE $\theta_t = \theta_t - \eta_t\lambda \theta_t$ (weight decay)
		\STATE $\theta_{t+1} = \theta_{t} - \eta_t \cdot \clip(m_t / \max\{\gamma \cdot h_t,\epsilon\},1)$
		\ENDFOR
	\end{algorithmic}
\end{algorithm}
\end{minipage}
\end{minipage}
\end{figure}

Concretely, {\ours} estimates the diagonal entries of the Hessian of the loss using a mini-batch of examples every $k$ step (with $k = 10$ in our implementation). 
We consider two options for diagonal Hessian estimators: (a) an unbiased estimator that uses a Hessian-vector product with the same run-time as a mini-batch gradient up to a constant factor, and (b) a biased estimator that uses one mini-batch gradient calculated with resampled labels. Both the two estimators only introduce 5\% overheads per step (in average).
At every step, {\ours} updates the parameter with an exponential moving average (EMA) of the gradient divided by the EMA of the diagonal Hessian estimate, subsequently clipped by a scalar. (All operations are element-wise.) See Algorithm~\ref{alg:hess_opt} for the pseudo-code.

Additionally, {\ours} can be seamlessly integrated into existing training pipelines, without any special requirements on the model architecture or computing infrastructure. With the either of the Hessian estimators, {\ours} only require either standard mini-batch gradients, or Hessian-vector products which are supported in auto-differentiation frameworks such as PyTorch~\citep{pytorch2019} and JAX~\citep{jax2018github}.

Thanks to the Hessian-based pre-conditioner, {\ours} adapts more efficiently, than Adam does, to the heterogeneous curvatures in different parameter dimensions, which can often occur in the landscape of LLMs losses and cause instability or slowdown. 
Sophia has a more aggressive pre-conditioner than Adam---{\ours} applies a stronger penalization to updates in sharp dimensions (where the Hessian is large) than the flat dimensions (where the Hessian is small), ensuring a uniform \textit{loss} decrease across all parameter dimensions. In contrast, Adam's updates are mostly uniform across all parameter dimensions, leading to a slower loss decrease in flat dimensions. (See Section~\ref{sec:motivation} for more discussions.) 
These make Sophia converge in fewer iterations. Thanks to the light-weight diagonal Hessian estimate, the speed-up in the number of steps translates to a speed-up in total compute and wall-clock time.

\ours's clipping mechanism controls the worst-case size of the updates in all directions, safeguarding against the negative impact of inaccurate Hessian estimates, rapid Hessian changes over time, and non-convex landscape (with which the vanilla Newton’s method may converge to local maxima or saddle points instead of local minima). The safeguard allows us to estimate Hessian infrequently (every $k=10$ step) and stochastically. In contrast, prior second-order methods often update Hessian estimates every step~\citep{martens2015optimizing,grosse2016kronecker,  anil2020scalable,yao2021adahessian}.

We provide theoretical analyses of much simplified versions of Sophia on convex functions. 
The runtime bound does not depend on the local condition number (the ratio between maximum and minimum curvature at the local minimum) and the worst-case curvature (that is, the smoothness parameter), demonstrating the advantage of Sophia in adapting to heterogeneous curvatures across parameter dimensions. 

%% file: method.tex
\section{Method}\label{sec:method}
We first instantiate gradient descent (GD) and Adam on a simplified 2D problem and motivate the use of second-order information and per-coordinate clipping in Section~\ref{sec:motivation}. Then, we present \ours~in detail in Section~\ref{sec:sophia}, and the pseudo-code in Algorithm~\ref{alg:hess_opt}. We introduce two choices of estimators of diagonal Hessian used in Sophia in Section~\ref{sec:hess_estimator}. 

\subsection{Motivations}\label{sec:motivation}

\begin{wrapfigure}{r}{5.5cm}
	\centering
\includegraphics[width=0.32\columnwidth]{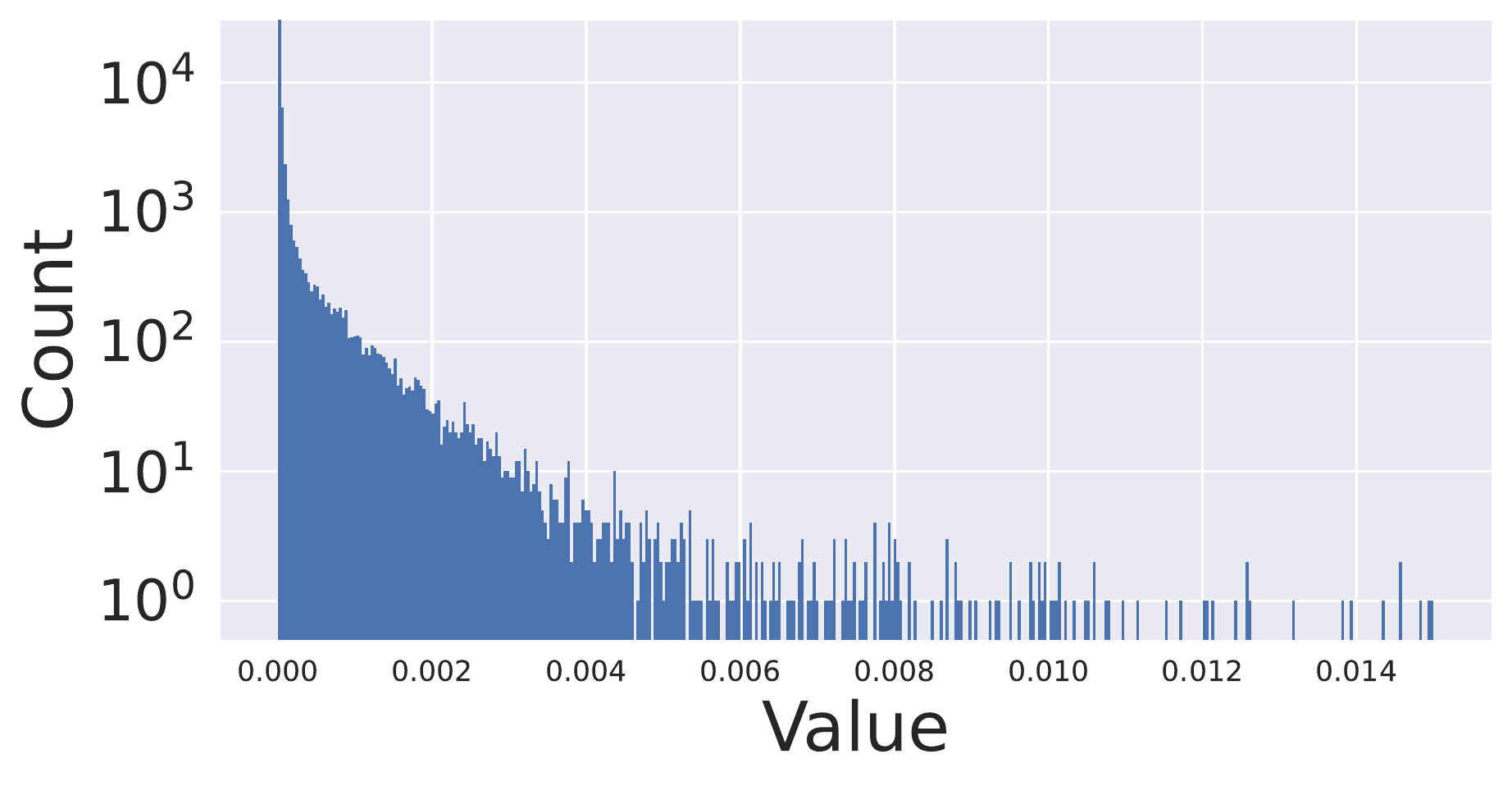}
 \caption{Histogram of positive entries of the diagonal Hessian of a 125M-parameter GPT-2.\label{fig:lt} 
 }
\end{wrapfigure}
\textbf{Heterogeneous curvatures.} The loss functions of modern deep learning problems often have different curvatures across different parameter dimensions~\citep{sagun2016eigenvalues,ghorbani2019investigation,zhang2020adaptive,yao2020pyhessian}. E.g., on a 125M-parameter GPT-2 model, Figure~\ref{fig:lt} shows that the distribution of positive diagonal entries of the Hessian is dispersed.

We demonstrate the limitations of Adam and GD on heterogeneous landscapes by considering a two-dimensional loss function $L(\theta_{[1]},\theta_{[2]}) = L_1(\theta_{[1]}) + L_2(\theta_{[2]})$ where $L_1$ is much sharper than $L_2$. We plot the loss landscape of $L(\theta_{[1]},\theta_{[2]})$ in Figure~\ref{fig:toy}.\footnote{Concretely, in Figure~\ref{fig:toy}, $L_{1}(\theta_{[1]}) = 8(\theta_{[1]} - 1)^2(1.3\theta_{[1]}^2 + 2\theta_{[1]} + 1)$ and $L_2(\theta_{[2]}) = 1/2 (\theta_{[2]} - 4)^2$.}  
For simplicity, we discuss GD and deterministic versions of Adam. Recall that GD’s update in this setting is:
\begin{align}
\theta_{[1]} \leftarrow \theta_{[1]} - \eta \cdot L'_1(\theta_{[1]})\ \ \textup{and}\ \ \theta_{[2]} \leftarrow \theta_{[2]} - \eta\cdot L'_2(\theta_{[2]})\,.\label{eqn:rule_gd_toy}
\end{align}
A common simplification of Adam that is more amenable to analysis~\citep{balles2018dissecting,bernstein2018signsgd,zhuang2020adabelief,kunstner2023noise} is SignGD, which dates back to RProp~\citep{braun1992rprop} that motivated RMSProp~\citep{hinton2012neural} and Adam. 
Observe that without using the EMA (for both the gradient and second moments of the gradient), Adam's update is simplified to $\eta \cdot \nabla L(\theta)/|\nabla L(\theta)| = \eta\cdot \textup{sign}(\nabla L(\theta))$ (where all operations are entry-wise), which is called SignGD. Applying the update rule to our setting gives: 
\begin{align}
\theta_{[1]} \leftarrow \theta_{[1]} - \eta \cdot \textup{sign}(L'_1(\theta_{[1]}))\ \ \textup{and}\ \ \theta_{[2]} \leftarrow \theta_{[2]} - \eta \cdot \textup{sign}(L'_2(\theta_{[2]}))\,.
\end{align}

\textbf{Limitations of GD and SignGD (Adam).} It is well known that the optimal learning rate of GD should be proportional to the inverse of the curvature, that is, the Hessian/second derivative at the local minimum. More precisely, let $h_1$ and $h_2$ be the curvatures of $L_1$ and $L_2$  at the local minimum (and thus $h_1 \gg h_2$). The optimal learning rate for the update of $\theta_{[1]}$ in equation~\eqref{eqn:rule_gd_toy} is $\asymp 1/h_1$, which is much smaller than the optimal learning rate that the update of $\theta_{[2]}$ needs, which is $\asymp 1/h_2$. As a result, the largest shared learning rate can only be $1/h_1$; consequently, the convergence in $\theta_{[2]}$ dimension is slow as demonstrated in the brown curve in Figure~\ref{fig:toy}.

The update size of SignGD is the learning rate $\eta$ in all dimensions. The same update size translates to less progress in decreasing the loss in the flat direction than in the sharp direction. As observed from the yellow curve in Figure~\ref{fig:toy}, the progress of SignGD in the flat dimension $\theta_{[2]}$ is slow because each step only decreases the loss $L_2(\theta_{[2]})$ slightly. On the other hand, along the direction $\theta_{[1]}$, the iterate quickly travels to the valley in the first three steps and then starts to bounce. To fully converge in the sharp dimension, the learning rate $\eta$ needs to decay to 0, which will exacerbate the slow convergence in the flat dimension $\theta_{[2]}$. The trajectory of Adam in this example is indeed similar to SignGD, which is also plotted as the red curve in Figure~\ref{fig:toy}.

The behavior of SignGD and Adam above indicates that a more aggressive pre-conditioning is needed---sharp dimensions should have relatively smaller updates than flat dimensions so that the decrease of loss is equalized in all dimensions. As suggested by well-established literature on second-order optimization~\citep{boyd2004convex}  for convex functions, the optimal pre-conditioner should be the Hessian, which captures the curvature on each dimension; as in Newton's method, the update is the gradient divided by the Hessian in each dimension: 
\begin{align}
\theta_{[1]} \leftarrow \theta_{[1]} - \eta \cdot L'_1(\theta_{[1]}) / h_1\ \ \textup{and}\ \ \theta_{[2]} \leftarrow \theta_{[2]} - \eta \cdot L'_2(\theta_{[2]})/ h_2\,.
\end{align}

\textbf{Limitations of Newton's method.} Nevertheless, Newton’s method has known limitations as well. For non-convex functions, vanilla Newton’s method could converge to a global maximum when the local curvature is negative. In the blue curve of Figure~\ref{fig:toy}, Newton’s method quickly converges to a saddle point instead of a local minimum. The curvature might also change rapidly along the trajectory, making the second-order information unreliable. To address these limitations, we propose considering only pre-conditioners that capture positive curvature, and introduce a pre-coordinate clipping mechanism to mitigate the rapid change of Hessian (more detail in Section~\ref{sec:sophia}). Applying our algorithm on the toy case results in the following update: 
\begin{align}
\theta_{[1]} \leftarrow \theta_{[1]} - \eta \cdot \clip(\nicefrac{ L'_1(\theta_{[1]})}{\max\{h_1,\epsilon\}} ,\rho)\ \textup{and}\ \theta_{[2]} \leftarrow \theta_{[2]} - \eta \cdot \clip(\nicefrac{ L'_2(\theta_{[2]})}{\max\{h_2,\epsilon\}},\rho)\,,\label{eqn:ours_update_toy}
\end{align}
where $\rho$ is a constant to control the worst-case update size, $\epsilon$ is a very small constant (e.g., 1e-12), which avoids dividing by 0. When the curvature of some dimension is rapidly changing or negative and thus the second-order information is misleading and possibly leads to a huge update before clipping, the clipping mechanism kicks in and the optimizer defaults to SignGD (even though this is sub-optimal for benign situations).
Numerous prior methods such as trust region~\citep{conn2000trust}, backtracking line search~\citep{boyd2004convex}, and cubic regularization~\citep{nesterov2006cubic} also tackle the same issue of Newton's method, but the clipping mechanism is much simpler and more efficient. 

As shown in the black curve in Fig.~\ref{fig:toy}, the update in equation~\eqref{eqn:ours_update_toy} starts off similarly to SignGD due to the clipping mechanism in the non-convex region, making descent opposed to converging to a local maximum. Then, in the convex valley, it converges to the global minimum with a few steps. Compared with SignGD and Adam, it makes much faster progress in the flat dimension $\theta_{[2]}$ (because the update is bigger in dimension $\theta_{[2]}$), while avoiding boucing in the sharp dimension $\theta_{[1]}$ (because the update is significantly shrunk in the sharp dimension $\theta_{[1]}$).

\subsection{Sophia: \textit{S}econd-\textit{o}rder Cli\textit{p}ped Stoc\textit{h}astic Opt\textit{i}miz\textit{a}tion}\label{sec:sophia}

Section~\ref{sec:motivation} demonstrates that Adam does not sufficiently adapt to the heterogeneous curvatures. On the other hand, vanilla Newton’s method has a pre-conditioner optimal for convex functions, but is vulnerable to negative curvature and rapid change of Hessian. With these insights, we design a new optimizer, \ours, which is more adaptive to heterogeneous curvatures than Adam, more resistant to non-convexity and rapid change of Hessian than Newton's method, and also uses a low-cost pre-conditioner.

We use $\theta_t$ to denote the parameter at time step $t$. At each step, we sample a mini-batch from the data distribution and calculate the mini-batch loss, denoted by $L_t(\theta_t)$. We denote by $g_t$ the gradient of $L_t(\theta_t)$, i.e. $g_t = \nabla L_t(\theta_t)$. Let $m_t$ be the EMA of gradients, $m_{t} \leftarrow \beta_{1} m_{t-1} + (1 - \beta_{1}) g_{t}$, which is the numerator of the update. 

\textbf{EMA of diagonal Hessian estimates.} \ours~uses a diagonal Hessian-based pre-conditioner, which directly adjusts the update size of different parameter dimensions according to their curvatures. We will present two options in detail in Section~\ref{sec:hess_estimator} for estimating the diagonal Hessian efficiently. 
To mitigate the overhead, we only estimate the Hessian every $k$ steps ($k=10$ in our implementation). At time step $t$ with $t \ \mathrm{mod} \ k = 1$, the estimator returns an estimate $\hat h_t$ of the diagonal of the Hessian of the mini-batch loss. 

Similar to the gradient of the mini-batch loss function, the estimated diagonal Hessian can also have large noise. Inspired by the EMA of moments of gradients in Adam, we also denoise the diagonal Hessian estimates with EMA across iterations. We update the EMA every $k$ steps, resulting in the following update rule for the diagonal Hessian estimate:
\begin{align}
h_t = \beta_2 h_{t-k} + (1 - \beta_2) \hat{h}_{t} \ \textup{ if } t \ \mathrm{mod} \ k = 1 ; \textup{ else } h_t = h_{t-1}\,.
\end{align}

\textbf{Per-coordinate clipping.} As discussed in Section~\ref{sec:motivation}, on nonconvex functions, vanilla Newton's method, which uses Hessian as the pre-conditioner, may converge to local maxima instead of local minima. In addition, the inaccuracy of Hessian estimates and the change of Hessian along the trajectory can make the second-order information unreliable. To this end, we (1) only consider the positive entries of the diagonal Hessian and (2) introduce per-coordinate clipping to the update. For a clipping threshold $\rho > 0$, let the clipping function be $\clip(z,\rho) = \max\{\min\{z,\rho\},-\rho\}$ where all operations are applied coordinate-wise. The update rule is written as: 
\begin{align}
    \theta_{t+1} \leftarrow  \theta_{t} - \eta_t \cdot \clip(m_t / \max\{\gamma \cdot h_t,\epsilon\}, 1),
\end{align}
where $\epsilon > 0$ is a very small constant to avoid dividing by $0$. Note that $\eta_t \cdot \clip(m_t / \max\{\gamma \cdot h_t,\epsilon\}, 1) = \left(\eta_t/\gamma\right) \cdot \clip(m_t / \max\{h_t,\epsilon/\gamma\}, \gamma)$, and thus we essentially clip the original update $m_t/h_t$ by $\gamma$ coordinate-wise, and then re-adjust the final update size by a factor of $\gamma$. The re-adjustment makes the scale of the update less dependent on $\gamma$, because now $\gamma$ only controls the fraction of clipped entries but all clipped entries will eventually be set to $\eta_t$ in the udpate---e.g., when $\gamma$ is extremely small, all entries of $\left(\eta_t/\gamma\right) \cdot \clip(m_t / \max\{h_t,\epsilon/\gamma\}, \gamma)$ will be $\eta_t$. In practice, the choice of $\gamma$ should be tuned based on the the fraction of clipped entries (See Section~\ref{sec:exp_setup} for details). 
We present the pseudo-code of the \ours~in Algorithm~\ref{alg:hess_opt}.

When any entry of $h_t$ is negative, e.g., $h_t[i] < 0$, the corresponding entry in the pre-conditioned gradient $m_t[i]/\max\{\gamma \cdot h_t[i],\epsilon\} = m_t[i]/\epsilon$ is extremely large and has the same sign as $m_t[i]$, and thus $\eta\cdot \clip(m_t[i] / \max\{\gamma\cdot h_t[i],\epsilon\},1) = \eta\cdot \textup{sign}(m_t[i])$, which is the same as stochastic momentum SignSGD.
In other words, Sophia uses stochastic momentum SignSGD as a backup when the Hessian is negative (or mistakenly estimated to be negative or very small.)  We also note that the clipping mechanism controls the worst-case size of the updates in all parameter dimensions to be at most $\rho$, which also improves the stability (which could be a severe issue for second-order methods). Moreover, because for many parameter dimensions, the clipping is not activated and the update is automatically adjusted, our worst-case update size $\eta\rho$ can be chosen to be larger than the worst update size $\eta$ in stochastic momentum SignSGD.

Several previous works~\citep{becker1988improving, chapelle2011improved,schaul2013no}, including the recent work AdaHessian~\citep{yao2021adahessian}, use diagonal Hessian as a pre-conditioner in optimizers for training neural networks. However, they use more frequent Hessian estimations, which leads to significant per-step computation overhead (more than two gradient computations), most likely because of the lack of the clipping mechanism that safeguards against inaccurate and changing Hessian. In general, to the best of our knowledge, there has not been previous reports that showed second-order optimizers achieve a speed-up on decoder-only large language models in wall-clock time or total compute (see more related work and discussions in Section~\ref{sec:rw}).

\subsection{Diagonal Hessian Estimators}\label{sec:hess_estimator}
 We introduce two diagonal Hessian estimators, both of which have memory and run-time costs similar to computing a gradient (up to constant factors).

\textbf{Option 1: Hutchinson's unbiased estimator.} For any loss function $\ell(\theta)$ on parameters $\theta\in \R^d$,  the Hutchinson's estimator  ~\citep{hutchinson1989stochastic,roosta2015improved,yao2021adahessian} first draws $u\in \R^d$ from the spherical Gaussian distribution $\mathcal{N}(0,\mathrm{I}_d)$, and then outputs $\hat{h} = u \odot (\nabla^2 \ell(\theta) u)$, where $\odot$ denotes the element-wise product, and $\nabla^2 \ell(\theta) u$ is the product of the Hessian with the vector $u$. The Hutchinson's estimator is an unbiased estimator for the diagonal of the Hessian, because
\begin{align}
\mathbb{E}[\hat h] = \mathrm{diag}(\nabla^2 \ell(\theta))\,.
\end{align}
The estimator only requires a Hessian-vector product (i.e., $\nabla^2 \ell(\theta) u$), which have efficient implementations in PyTorch and JAX, instead of the full Hessian matrix. 

\newcommand{\ce}{\ell_{\textup{ce}}}
\textbf{Option 2: Gauss-Newton-Bartlett (GNB) estimator.} We leverage the structure of the loss to design a biased stochastic estimator for the diagonal Hessian, following~\citet{schraudolph2002fast,martens2020new,wei2020implicit}. Suppose $\ell(\theta, (x,y))$ is a loss function on an example $(x,y)$ of the form $\ell(\theta, (x,y)) = \ce(f(\theta,x), y)$ where $\ce$ is the cross-entropy loss and $f(\theta, x)\in \R^{V}$ is the logits, and $V$ is the number of items/classes in a multi-class classification problem (e.g., the vocabulary size in LLMs). First, the Hessian of $\ell(\theta, (x,y))$ (w.r.t to variable $\theta$) has the well-known Gauss-Newton decomposition~\citep{ortega2000iterative,schraudolph2002fast} (which is a simple consequence of the chain rule), 
\begin{align}
    \nabla_\theta^2~\ell(\theta) = J_\theta f(\theta,x) \cdot S \cdot  J_\theta f(\theta, x)^\top + J_{\theta\theta}f(\theta, x)[q] \label{eqn:GN}
\end{align}
where $J_\theta f(\theta,x)$ is the Jacobian of $f$ w.r.t to $\theta$ viewed as a matrix in $\R^{d\times V}$,  $S = \frac{\partial^2 \ce(t,y)}{\partial t^2}\Big\vert_{t=f(\theta, x)}\in \R^{V\times V}$ is the second-order derivatives of the loss w.r.t to the logits, $q = \frac{\partial \ce(t,y)}{\partial t}\Big\vert_{t=f(\theta, x)} \in \R^V$ is the first-order derivatives of the loss w.r.t to the logits, and $J_{\theta\theta}f(\theta,x)$ is the second-order derivatives of the multi-variate function $f(\theta, x)$ w.r.t $\theta$, viewed as a linear map from $\R^{V}$ to $\R^{d \times d}$, where $d$ is the dimension of the parameter $\theta$.

In the context of neural networks, past works have found that the second term $J_{\theta\theta}f(\theta, x)[q]$ in \Cref{eqn:GN} is often relative smaller than the first term $J_\theta f(\theta,x) \cdot S \cdot  J_\theta f(\theta, x)^\top$~\citep{sankar2021deeper}, which is often referred to as the Gauss-Newton matrix~\citep{dennis1996numerical,ortega2000iterative,schraudolph2002fast,chen2011hessian} and used as pre-conditioners in second-order optimizers~\citep{botev2017practical,martens2020new,gargiani2020promise}. 
Following this line of work, we build an unbiased estimator for \textit{the diagonal} of the Gauss-Newton matrix, which is a biased estimator for the diagonal of the Hessian. 

We first claim that $S$ only depends $f(\theta, x)$ but not $y$, even though the loss depends on $y$.\footnote{Denote by $p(\theta,x)=\textup{softmax}(f(\theta,x))\in \R^V$ the probability vector obtained by applying softmax on the logits. Indeed, a simple derivation shows that $S = \textup{diagonal}(p(\theta,x)) - p(\theta,x)p(\theta,x)^\top$, where $\textup{diagonal}(p(\theta,x))$ is the matrix with the vector $p(\theta,x)$ residing on the diagonal. In fact, this is a general property of exponential families---the Hessian of the negative log-likelihood of any exponential family distribution only depends on the parameters of that exponential family, but not on the example on which the likelihood is evaluated.} 
Thus, $S = \frac{\partial^2 \ce(t,\hat{y})}{\partial t^2}\Big\vert_{t=f(\theta, x)}$ for any $\hat{y} \in \{1,\dots, V\}$, which implies that
$S = \mathbb{E}_{\hat{y}\sim p(\theta,x)}\left[\frac{\partial^2 \ce(t,\hat{y})}{\partial t^2}\Big\vert_{t=f(\theta, x)}\right].$ Because $\ce(t,y)$ is the negative log-probability of the probabilistic model defined by the categorical distribution $\Cat(t)$ with parameter $t$, by Bartlett's second identity~\citep{bartlett1953approximate}, we have that, 
\begin{align}
S = \mathop{\mathbb{E}}_{\hat{y}\sim \Cat(t)}\left[\frac{\partial^2 \ce(t,\hat{y})}{\partial t^2}\right] = \mathop{\mathbb{E}}_{\hat{y}\sim \Cat(t)}\left[\frac{\partial \ce(t,\hat{y})}{\partial t}\frac{\partial \ce(t,\hat{y})}{\partial t}^\top\right]\,,
\end{align}
where the first equality holds for $t = f(\theta,x)$ and the second equality holds for all $t$ by Bartlett's second identity.
Therefore, the Gauss-Newton matrix satisfies

\begin{align}
    J_\theta f(\theta,x) \cdot S \cdot  J_\theta f(\theta, x)^\top & =  \mathop{\mathbb{E}}_{\hat{y}\sim \Cat(t)}\left[J_\theta f(\theta,x) \frac{\partial \ce(t,\hat{y})}{\partial t}\frac{\partial \ce(t,\hat{y})}{\partial t}^\top J_\theta f(\theta, x)^\top\right]   \nonumber\\
    & = \mathop{\mathbb{E}}_{\hat{y}\sim \Cat(t)}\left[
        \nabla_\theta \ce(f(\theta,x),\hat{y})   \nabla_\theta \ce(f(\theta,x),\hat{y})^\top 
    \right], \label{eqn:24}
\end{align}
which implies that $\textup{diag}(J_\theta f(\theta,x) \cdot S \cdot  J_\theta f(\theta, x)^\top) = \mathop{\mathbb{E}}_{\hat{y}\sim \Cat(t)}\left[
        \nabla_\theta \ce(f(\theta,x),\hat{y}) \odot  \nabla_\theta \ce(f(\theta,x),\hat{y})
    \right]$. Hence, the quantity $\ce(f(\theta,x),\hat{y}) \odot  \nabla_\theta \ce(f(\theta,x),\hat{y})$ is an unbiased estimator of the Gauss-Newton matrix for the Hessian of a one-example loss $\ell(f(\theta,x),y)$. 

\textit{Mini-batch version.} Given a mini-batch of inputs $\{(x_b, y_b)\}_{b=1}^B$. The most natural way to build an estimator for the diagonal of the Gauss-Newton matrix for the Hessian of the mini-batch loss is using 
\begin{align}
\frac{1}{B}\sum_{b=1}^B \nabla \ce(f(\theta,x_b),\hat{y}_b) \odot  \nabla_\theta \ce(f(\theta,x_b),\hat{y}_b)\,, \label{eqn:23}
\end{align}
where $\hat{y}_b$'s are labels sampled from the model on inputs $x_b$'s respectively. However, as noted by~\citet{grosse2022neural}, implementing this estimator is inconvenient under the current auto-differentiation frameworks, where the users only have access to the average gradient over a mini-batch (as opposed to the individual ones). Fortunately, by the Bartlett's first identity~\citep{bartlett1953approximate} (which generally holds for the negative log-likelihood loss of any probabilistic model), we have:
\begin{align}
\forall b, ~~\mathbb{E}_{\hat{y}_b}\nabla \ce(f(\theta,x_b),\hat{y}_b) = 0\,.
\end{align}
Let $\widehat L(\theta) = \frac{1}{B}\sum_{b=1}^B \ce(f(\theta, x_b), \hat{y}_b)$ be the mini-batch loss on the \textit{sampled} labels (as opposed to the original labels). 
Observing that $\hat{y}_b$'s are independent with each other, we have
\begin{align}
\mathbb{E}_{\hat{y}_b's} \left[B \cdot \nabla_\theta \widehat L(\theta) \odot \nabla_\theta \widehat L(\theta)\right] & = 
\mathbb{E}_{\hat{y}_b's} \left[\frac{1}{B} \sum_{b=1}^B \nabla \ce(f(\theta,x_b),\hat{y}_b) \odot \sum_{b=1}^B \nabla \ce(f(\theta,x_b),\hat{y}_b)\right] \nonumber \\
& = \mathbb{E}_{\hat{y}_b's} \left[\frac{1}{B} \sum_{b=1}^B \nabla \ce(f(\theta,x_b),\hat{y}_b) \odot \nabla \ce(f(\theta,x_b),\hat{y}_b)\right]\,.\label{eqn:20}
\end{align}
Note that the RHS of \Cref{eqn:20} is the same as the expectation of \Cref{eqn:23}, which, by \Cref{eqn:24}, also equals to the diagonal of the Gauss-Newton matrix for the mini-batch loss.
Hence, we use $B \cdot \nabla_\theta \widehat L(\theta) \odot \nabla_\theta \widehat L(\theta)$ as the estimator.

\textit{GNB estimator for exponential family.} If $y$ is drawn from an exponential family $p(y;\eta)$ where the natural parameter $\eta$ is set to be $f(\theta, x)$ and the loss function $\ell(f(\theta, x), y)$ is the negative log-likelihood loss for the corresponding probabilistic distribution, then all the derivations above still follow because (1) $S$ still only depends on $f(\theta, x)$ but not $y$, and (2) both the first and second Bartlett's identities still hold. 

\textit{GNB estimator for squared loss.} When $y, f(\theta,x)\in \R$ and $\ell(f(\theta, x), y)= \frac{1}{2}(f(\theta,x)-y)^2$, the $S$ matrix is identity, and thus one can simply use 
$J_\theta f(\theta, x) J_\theta f(\theta,x)^\top$ as the estimator.\footnote{One can verify that the GNB estimator gives the same quantity in expectation if we use a probabilistic model $y \sim \mathcal{N}(f(\theta, x), \sigma^2)$ and go through the derivation of the GNB estimator.}

To the best of our knowledge, ~\cite{wei2020implicit} is the first paper that uses this estimator of Gauss-Newton matrix. Given the use Bartlett's first and second identities that are central to the estimator, we call it Gauss-Newton-Bartlett (GNB) estimator.

\textbf{Comparisons of Hessian estimators.}
The Hutchinson's estimator does not assume any structure of the loss, but requires a Hessian-vector product. The GNB estimator only estimates the Gauss-Newton term but always gives a positive semi-definite (non-negative) diagonal Hessian estimate. The PSDness ensures that the pre-conditioned update is always a descent direction~\citep{dennis1996numerical}. The GNB estimator can also be easily extended to the negative log-likelihood loss of any exponential family distribution, and be adapted to estimating the trace of the Gauss-Newton matrix as in~\citet{wei2020implicit} or efficiently implementing the product of Gauss-Newton matrix with a vector. The authors suspect the GNB estimator has a smaller variance than the Hutchinson's estimator, but more empirical and theoretical investigation are needed to support the hypothesis.  

%% file: experiment.tex
\section{Experiments}\label{sec:experiments}

We name the algorithm using the Hutchinson's estimator and the GNB estimator Sophia-H and Sophia-G, respectively. We evaluate \ours~on auto-regressive language modeling with GPT-2~\citep{radford2019language} of model sizes ranging from 125M to 770M, and GPT NeoX~\citep{black2022gpt} of sizes 1.5B and 6.6B. Results indicate that \ours~is 2x faster than AdamW~\citep{loshchilov2017decoupled} in number of steps, total compute, and wall-clock time across all model sizes. Moreover, the scaling law is in favor of \ours~over AdamW. 

\subsection{Experimental Setup} \label{sec:exp_setup}

\textbf{Language modeling.} We train autoregressive models on OpenWebText~\citep{Gokaslan2019OpenWeb} and the Pile~\citep{gao2022pile}. Following standard protocol, we set the context length of GPT-2 to 1024, and the context length of GPT-2 NeoX~\citep{black2022gpt} to 2048. We consider GPT-2 with 125M (small), 355M (medium), and 770M (large) parameters, and GPT NeoX with 1.5B and 6.6B parameters, respectively. Detailed model configurations are deferred to Section~\ref{sec:imp_detail}.

\textbf{Baselines.} We mainly compare {\ours} and Adam with decoupled weight decay (AdamW)~\citep{loshchilov2017decoupled} which is the dominantly used optimizer on language modeling tasks, AdaHessian~\citep{yao2021adahessian} which uses the EMA of the square of the diagonal Hessian estimate in its denominator, and Lion~\citep{chen2023symbolic}, which is an first-order adaptive optimizer discovered by symbolic search. For the 30M model, all hyperparameters are tuned with grid search. 
For other models, all hyperparmeters but the peak learning rate are configured as identical to those found on the 30M model. For models with size 125M and 355M, the peak learning rates are obtained through grid search. 
For larger models, we gradually increase the peak learning rate to search for the largest possible peak learning rate such that the training does not blow up, and ensure that the chosen learning rate is approximately the largest in the sense that 1.25 times the chosen learning rate will lead to a blow-up. For AdamW we found the well-established practice ($\beta_1 = 0.9$ and $\beta_2 = 0.95$) works consistently better than other choices~\citep{radford2019language,mistral}. For Lion, we use $\beta_1 = 0.95$ and $\beta_2 = 0.98$ following \citet{chen2023symbolic}. Although~\citet{chen2023symbolic} suggests using $0.1$ times the learning rate (LR) of AdamW for vision tasks, we find out the LR should be larger on LMs from the grid search. For AdaHessian, we found $\beta_1 = 0.92$ and $\beta_2 = 0.99$ works the best in the grid search. Details on hyperparameter tuning are deferred to Section~\ref{sec:hyper_param}.

\textbf{Implementation.} We set batch size to 480 for GPT-2 and 2048 for GPT NeoX. We use cosine LR schedule with the final LR equal to 0.05 times the peak LR,  following~\citet{rae2021scaling}. We use the standard gradient clipping (by norm) threshold 1.0. We adopt a fixed 2k steps of LR warm-up. For \ours, we use $\beta_1 = 0.96$, $\beta_2 = 0.99$, $\epsilon=$1e-12 and update diagonal Hessian every 10 steps. For Sophia-H, we use only a subset of 32 examples from the mini-batch to calculate the diagonal Hessian to further reduce overhead. For Sophia-G, we use a subset of 240 examples from the mini-batch to calculate the diagonal Gauss-Newton. We implement the algorithms in PyTorch~\citep{pytorch2019} and JAX~\citep{jax2018github} and train all the models in bfloat16. The 125M and 355M models are trained on A5000 GPUs, while the 770M models are trained on A100 GPUs. We use a TPU v3-128 slice to train the 1.5B and 6.6B GPT NeoX.

\paragraph{Hyperparamter tuning strategy.} We refer to Section~\ref{sec:hyper_param} for the details on hyperparameters and only discuss two key hyperparameters, $\gamma$ and the peak learning rate $\eta$ in the main text. Similar to the protocol of baselines, all other hyperparameters are tuned on a 30M model and remain fixed for all the model sizes. For the peak learning rate and $\gamma$, we found the following strategy general works well, and delivers almost the same performance as those found by grid search. 
\begin{itemize}
    \item On a small model, tune $\gamma$ to make the proportion of coordinates where the update is not clipped (i.e., $|m_t / \max\{\gamma \cdot h_t,\epsilon\}| < 1$) in the range of $10\%-50\%$. If the proportion is too large (or too small), multiply $\gamma$ by $0.5$ (or $2$) and restart. The same $\gamma$ likely can be transferred to models with the same architecture and data but different number of parameters. We use $\gamma = 0.01$ for Sophia-H and $\gamma=0.05$ for Sophia-G in this paper.
    \item Suppose we already find a suitable $\gamma$ following the above procedure. We can then set the learning rate of Sophia to be either 3-5 times the learning rate that one would have used for Lion, or 0.8 times the learning rate that one would have used for AdamW. 
\end{itemize}

\subsection{Evaluation} \label{sec:evaluation}

\paragraph{Methodology for comparing the optimizers for LLMs.} This paper argues that one correct (and arguably preferred) way of claiming optimizer $O_2$ is 2x faster than optimizer $O_1$ is  comparing the following two experiments (for a reasonable variety of $T$):
\begin{itemize}
    \item[1.] running optimizer $O_1$ (e.g. Adam) with $T$ steps, with the optimal learning rate and learning rate schedule (tuned for running for $T$ steps),
    \item[2.] running optimizer $O_2$ (e.g., Sophia) with $T/2$ steps, with any learning rate schedule.
\end{itemize}
If Experiment 2 achieves a loss that is smaller than or equal to the loss of Experiment 1 (for a reasonable sets of choices of $T$), then we say optimizer $O_2$ is 2x faster than optimizer $O_1$.

\newcommand{\eval}{\textup{Eval}}
In more formal language, suppose the loss (or any minimization metric) of an optimizer $O$ with runtime budget $T$ and hypeparameter $H$ (excluding the runtime $T$) is denoted by $\eval(O, T, H)$. Then, we say optimizer $O_2$ is $k$-times faster than $O_1$ if 
\begin{align}
    \exists H_2, ~~\min_{H_1} ~\eval(O_1, T, H_1) \ge \eval(O_2, T/k, H_2) \,. \label{eq:compare}
\end{align}
Note that the equation above implies $\min_{H_1}\eval(O_1, T, H_1) \ge \min_{H_2} \eval(O_2, T/k, H_2)$, but is simpler to verify than the latter. 
We note that many modern learning rate schedulers such as cosine learning rate~\citep{loshchilov2016sgdr} are highly sensitive to a pre-specific total number of steps. The optimal hyperparameters are also sensitive to the total budget $T$. Thus, tuning the hyperparameters of the baseline $O_1$ with a pre-specified budget $T$ (aka, minimizing $H_1$ given $T$) is critical to ensure fair comparison. (On the other hand, the tuning strategy for $H_2$ is less important because we just need to show the existence of a good $H_2$.)

Moreover, we note that, in equation~\eqref{eq:compare}, we need to insist $\ge$ without any approximations. The difference in losses of various experiments might be seemingly small--- $\eval(\textup{Adam}, T/2, H)$ is likely to be only slightly higher than $\eval(\textup{Adam}, T, H)$ (assuming the peak learning rate is a hyperparameter), and thus allowing any approximation in the criterion~\eqref{eq:compare} may lead to the fallacy statement that ``Adam is 2x faster than Adam''. In fact, Figure~\ref{fig:gpt2-1} (a) shows that even with the same peak learning rate, the learning rate in a run with $T/2$ steps decays faster than the learning rate in a run with $T$ steps. Moreover, the $T/2$-steps run tends to have a initial faster decay of loss than the $T$-steps run but stop at a higher final loss, and the latter is not a continuation of the former.  

Finally, a tempting comparison would be comparing the following experiment with Experiment 1:
\begin{itemize}
    \item[2'.] running optimizer $O_2$ (e.g., Sophia) with $T$ steps, with any learning rate schedule, and recording the performance of the $T/2$-th checkpoint. 
\end{itemize}

We note that if one were able to find an optimizer $O_2$ that is 2x faster than $O_1$ under the proposed criterion (comparing 1\&2), then one can also design another optimizer $O_2'$ (which run for $T$ steps) that is 2x faster than $O_1$ under the second criterion (comparing 1\&2'): first running optimizer $O_2$ for $T/2$ steps and do nothing for the rest of the $T/2$ steps. 
Therefore, even though we will also provide comparison between Experiment 1\&2' in Figure~\ref{fig:gpt2} as an alternative, we use the comparison between Experiment 1\&2 quantitatively in most of the paper.

\paragraph{Technical details.} Following the methodology above, we train baselines and Sophia for 100K, 200K, or 400K, and mainly compare 400k-steps baseline vs 200K-steps Sophia, and 200k-steps baseline vs 100k-steps Sophia. 
We primarily evaluate the models with their log perplexity and plot the loss curves. We also report in-context learning results (with 2-shot exemplars and greedy decoding) on SuperGLUE~\citep{wang2019superglue}. We average the results of 5 prompts (Section~\ref{sec:eval}).

\begin{figure}[t]
\begin{center}
\includegraphics[width=0.235\textwidth]{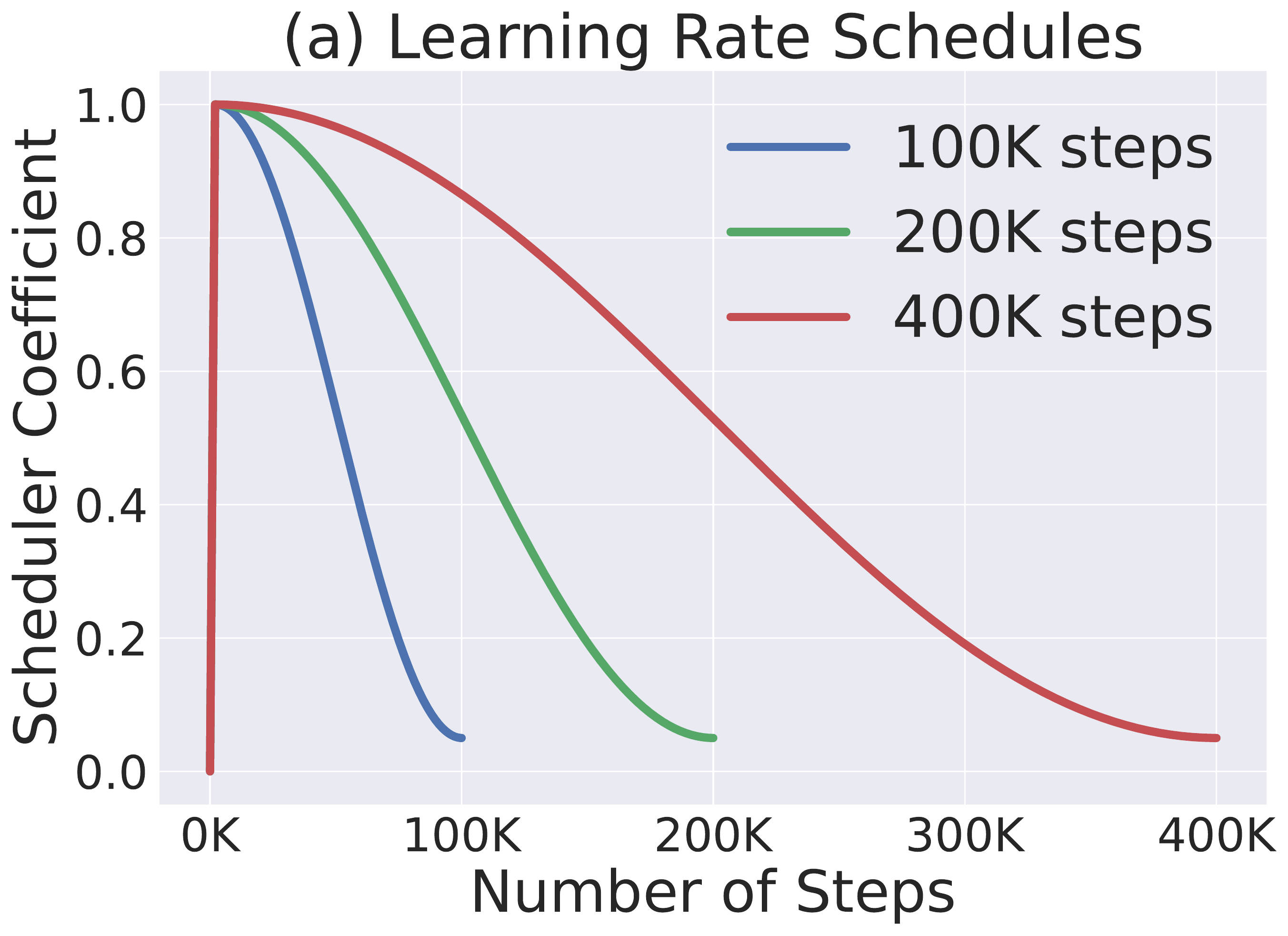}
\includegraphics[width=0.235\textwidth]{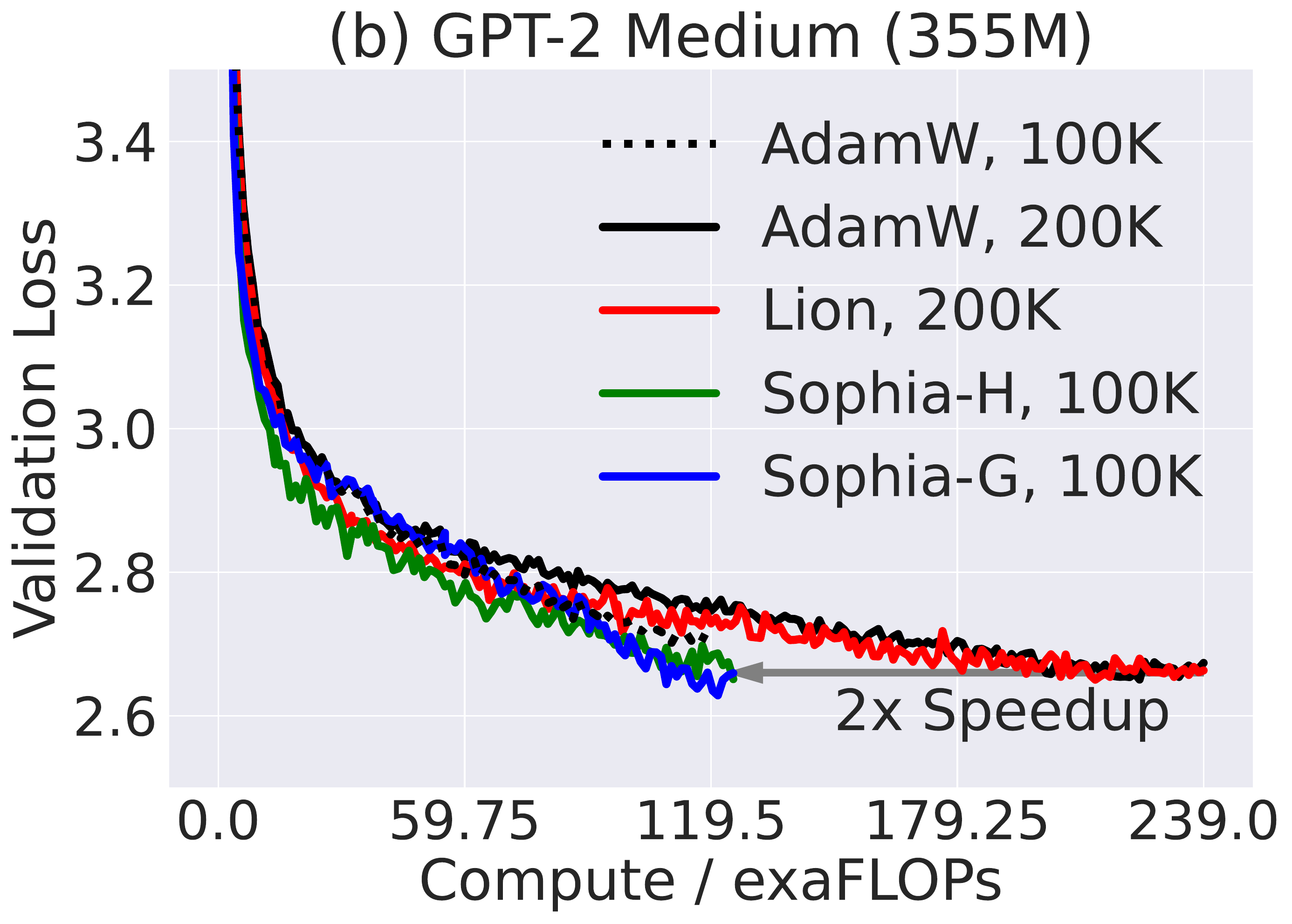}
\includegraphics[width=0.235\textwidth]{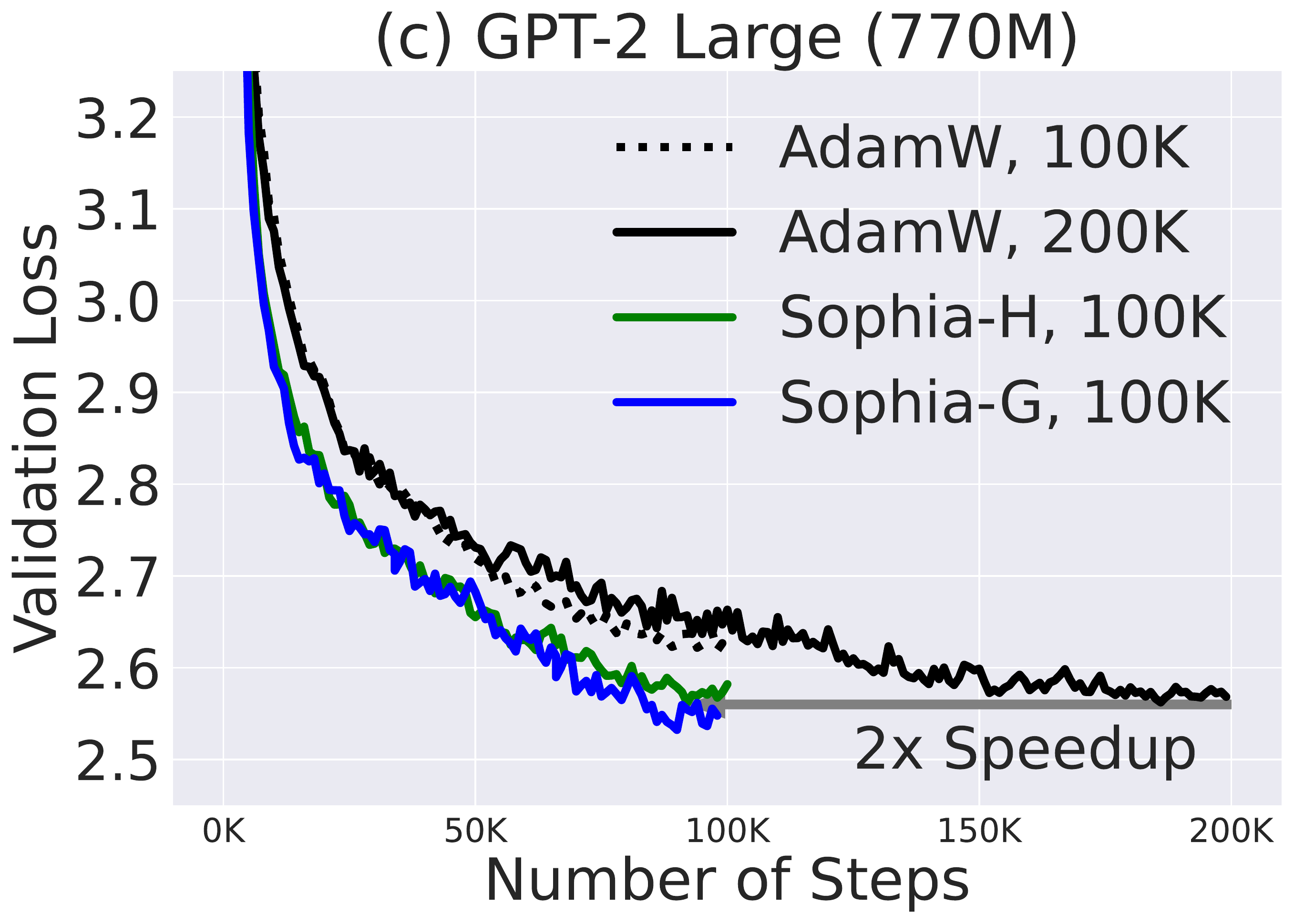}
\includegraphics[width=0.235\textwidth]{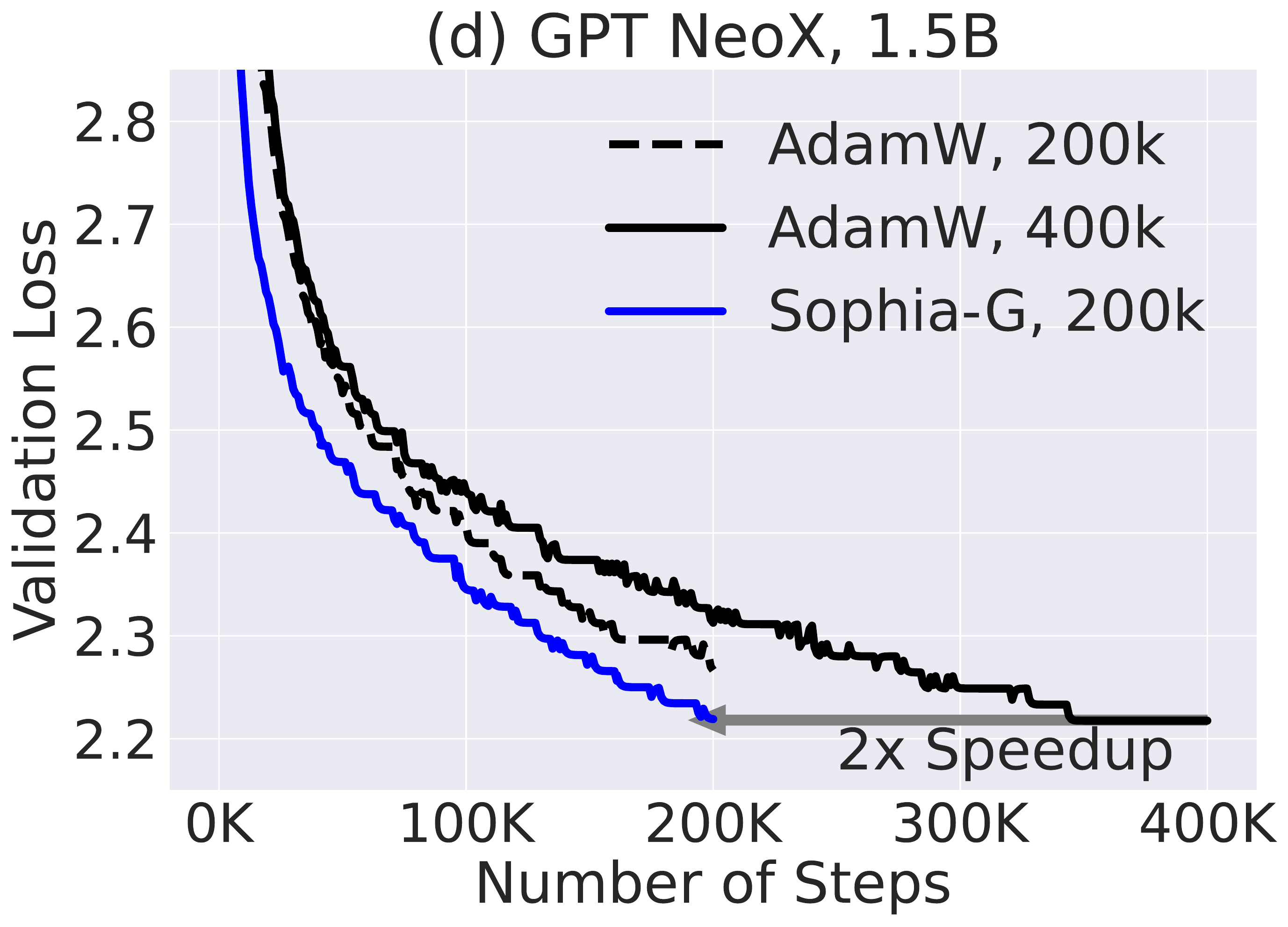}
\caption{ Comparison of numbers of steps to reach the same validation loss. (a) Learning rate schedules. (b) GPT-2 Medium (355M). (c) GPT-2 Large (770M). (d) GPT NeoX 1.5B. Across all model sizes, {\ours} achieves a 2x speedup. \label{fig:gpt2-1} 
}
\end{center}
\end{figure}

\begin{figure}[t]
\begin{center}
\includegraphics[width=0.19\textwidth]{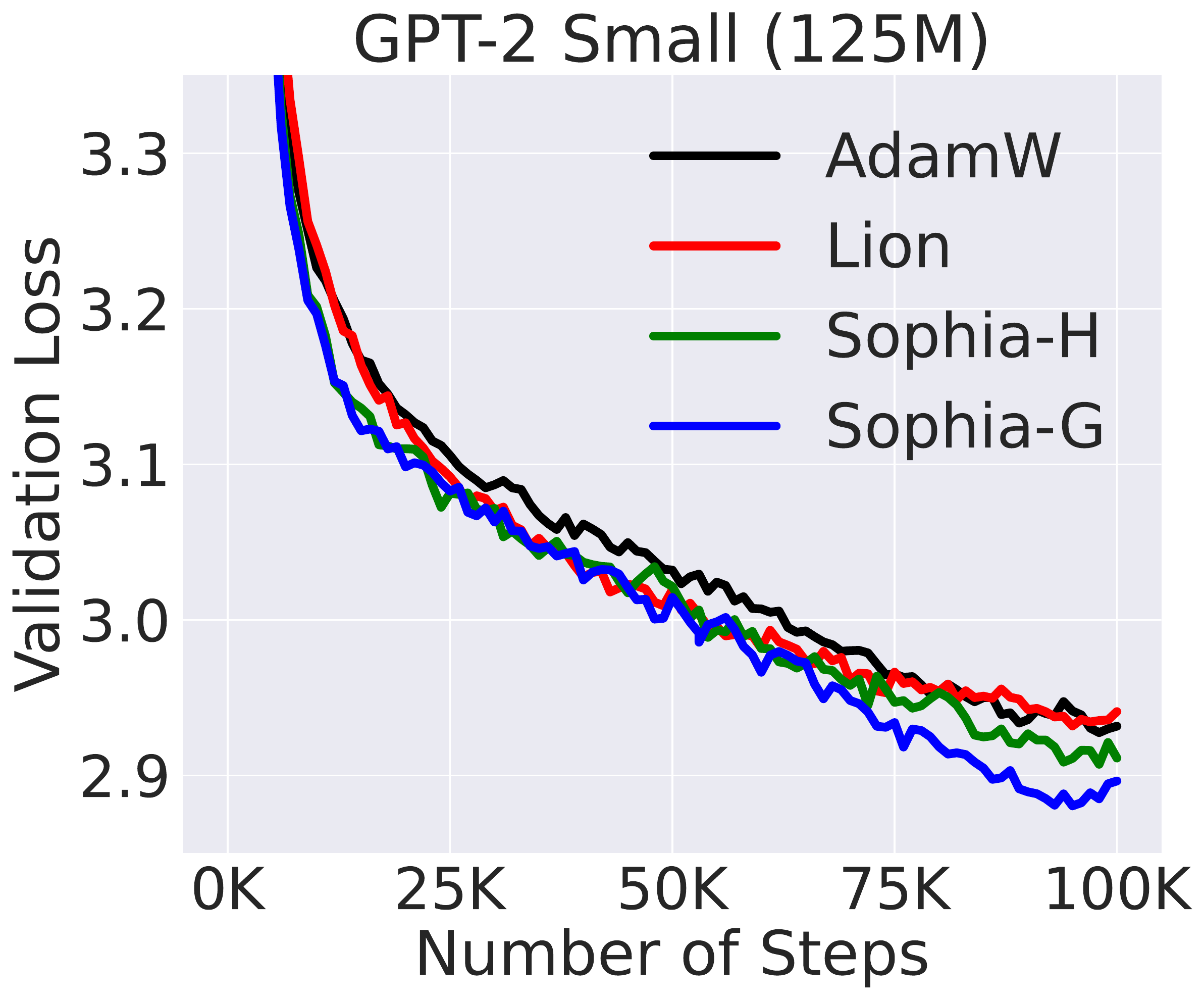}
\includegraphics[width=0.19\textwidth]{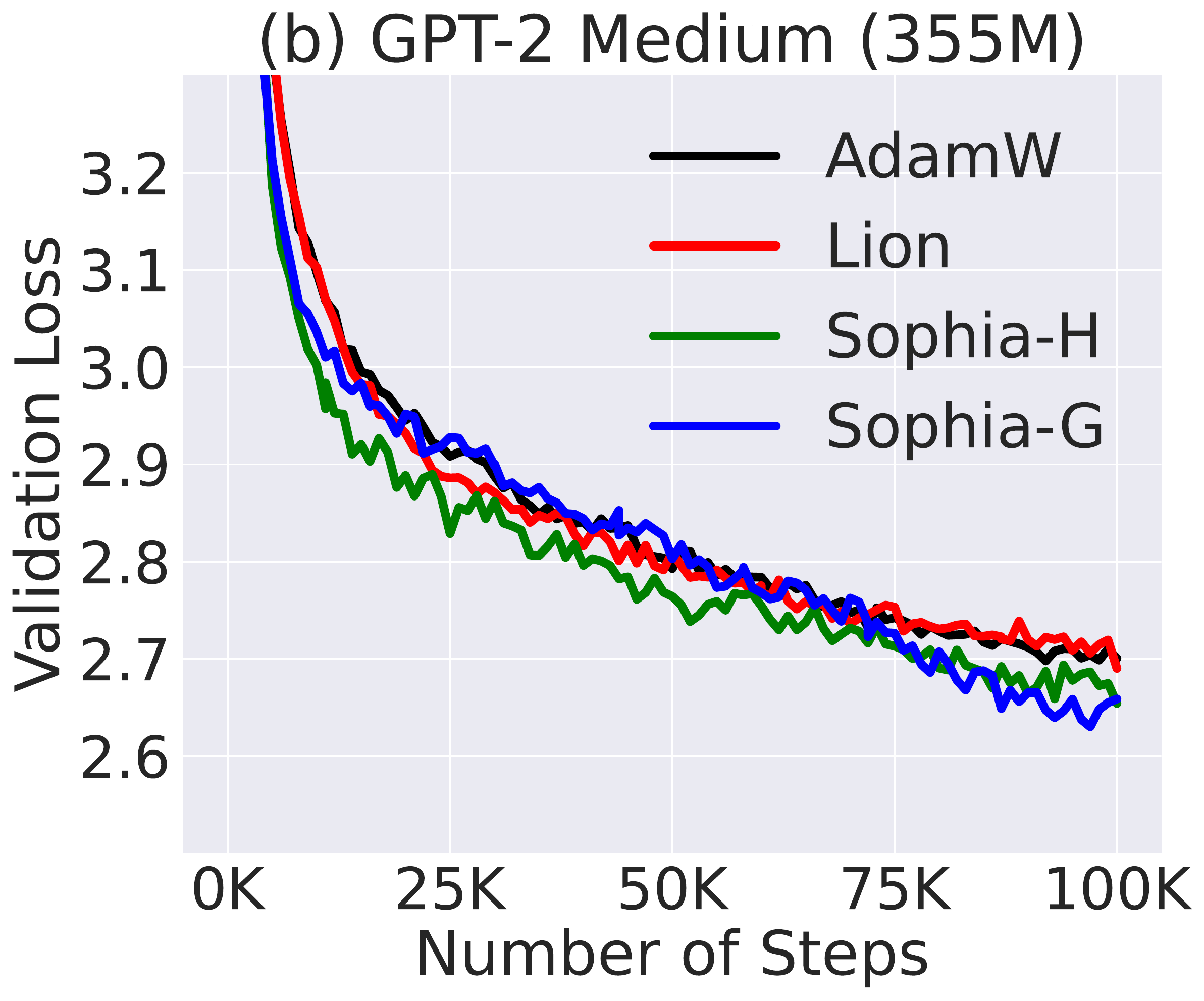}
\includegraphics[width=0.19\textwidth]{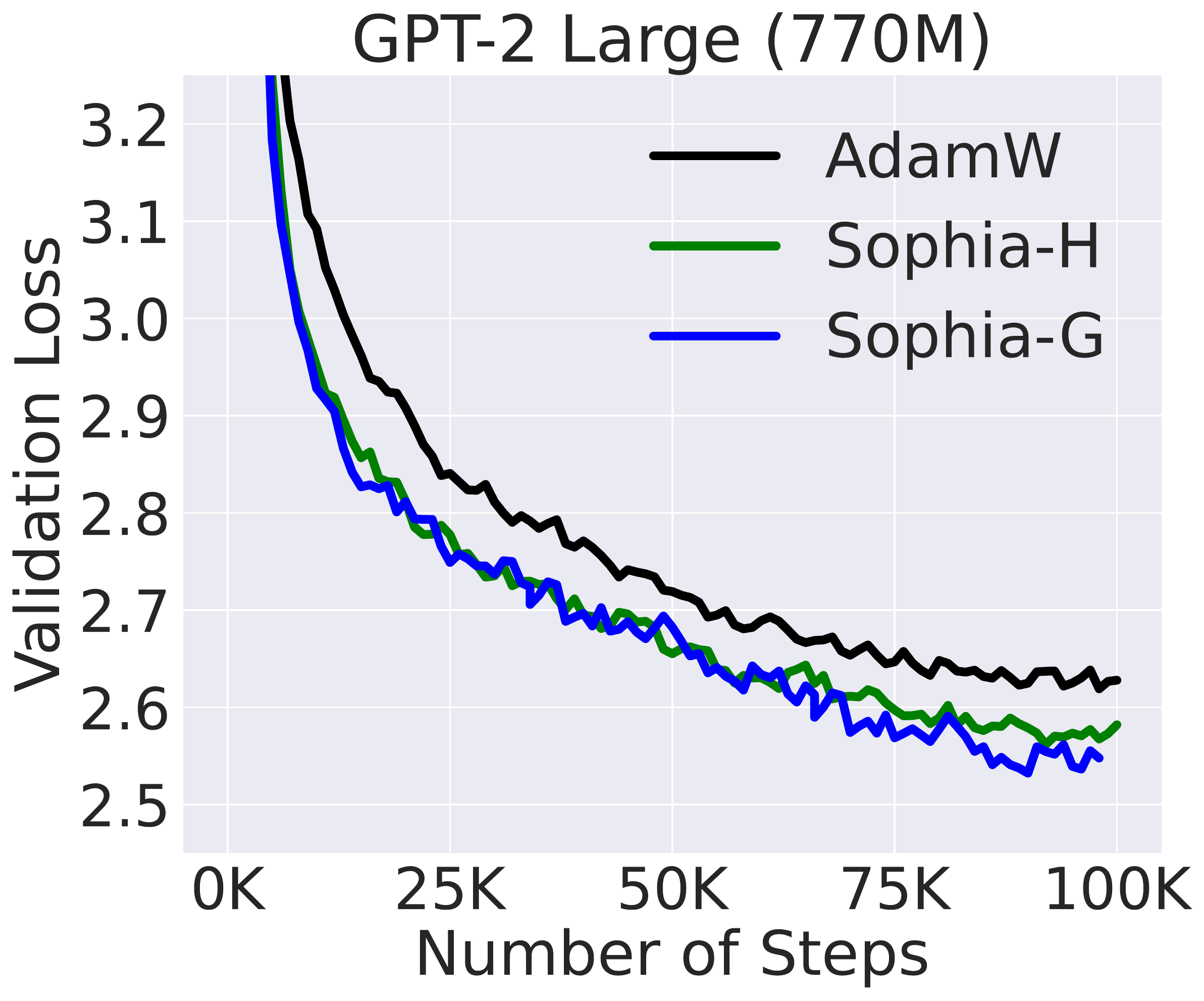}
\includegraphics[width=0.19\textwidth]{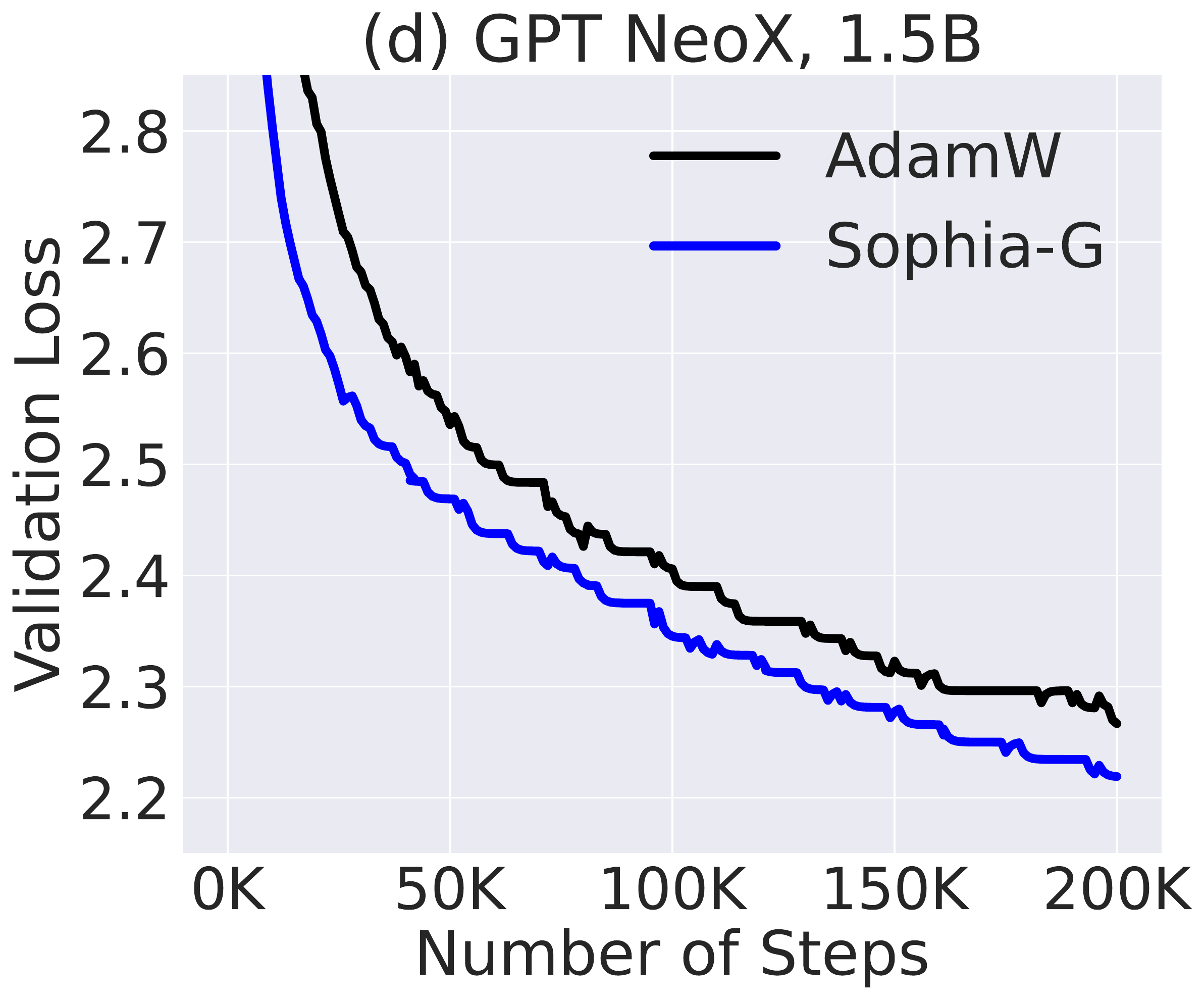}
\includegraphics[width=0.19\textwidth]{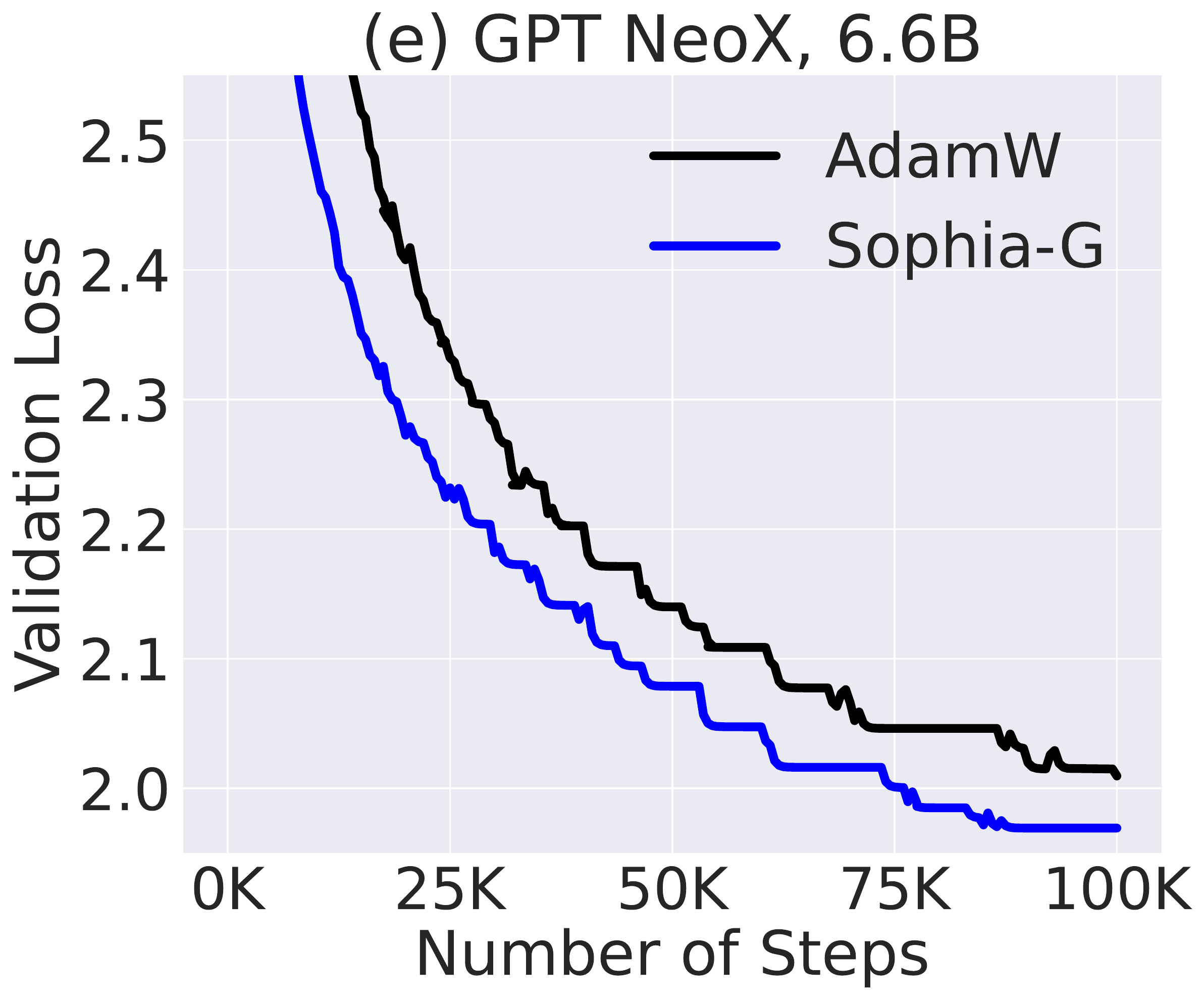}
\caption{ (a) (b) (c) Validation loss on OpenWebText. (a) GPT-2 Small (125M). AdamW: 2.921, Lion: 2.924, \ours-H: 2.901, Sophia-G: 2.875 (b) GPT-2 Medium (355M). Adam: 2.691, Lion: 2.678, \ours-H: 2.645, Sophia-G: 2.627. (c) GPT-2 Large (770M). AdamW: 2.613, \ours-H: 2.554, \ours-G: 2.524. (d) (e) Validation loss on the Pile (d) GPT NeoX 1.5B. AdamW: 2.250, \ours-G: 2.218.(e) GPT NeoX 6.6B. AdamW: 1.992, \ours-G: 1.969. \label{fig:gpt2}}
\end{center}
\end{figure}

\begin{figure}[t]
\begin{center}
\includegraphics[width=0.32\textwidth]{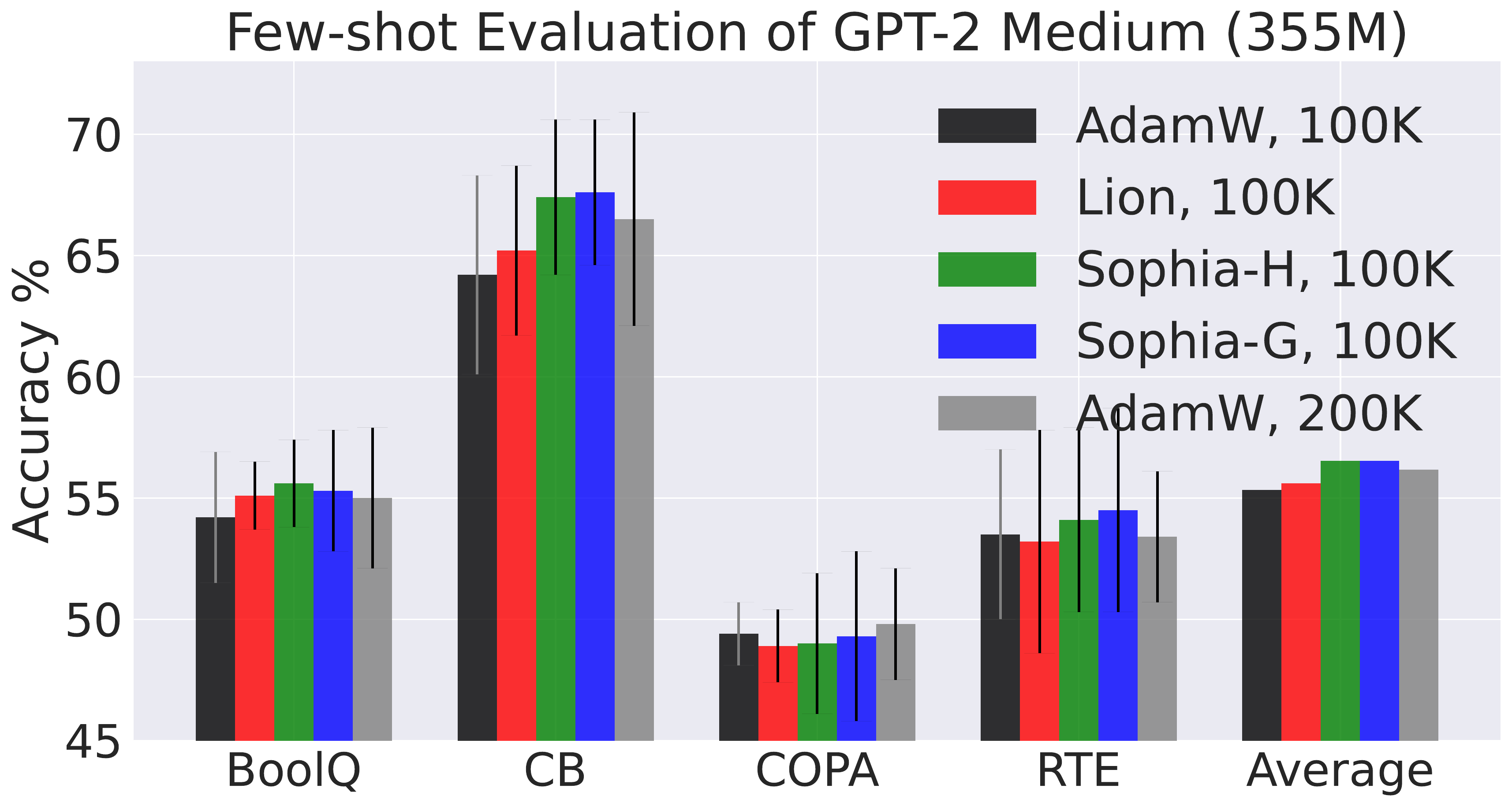}
\includegraphics[width=0.32\textwidth]{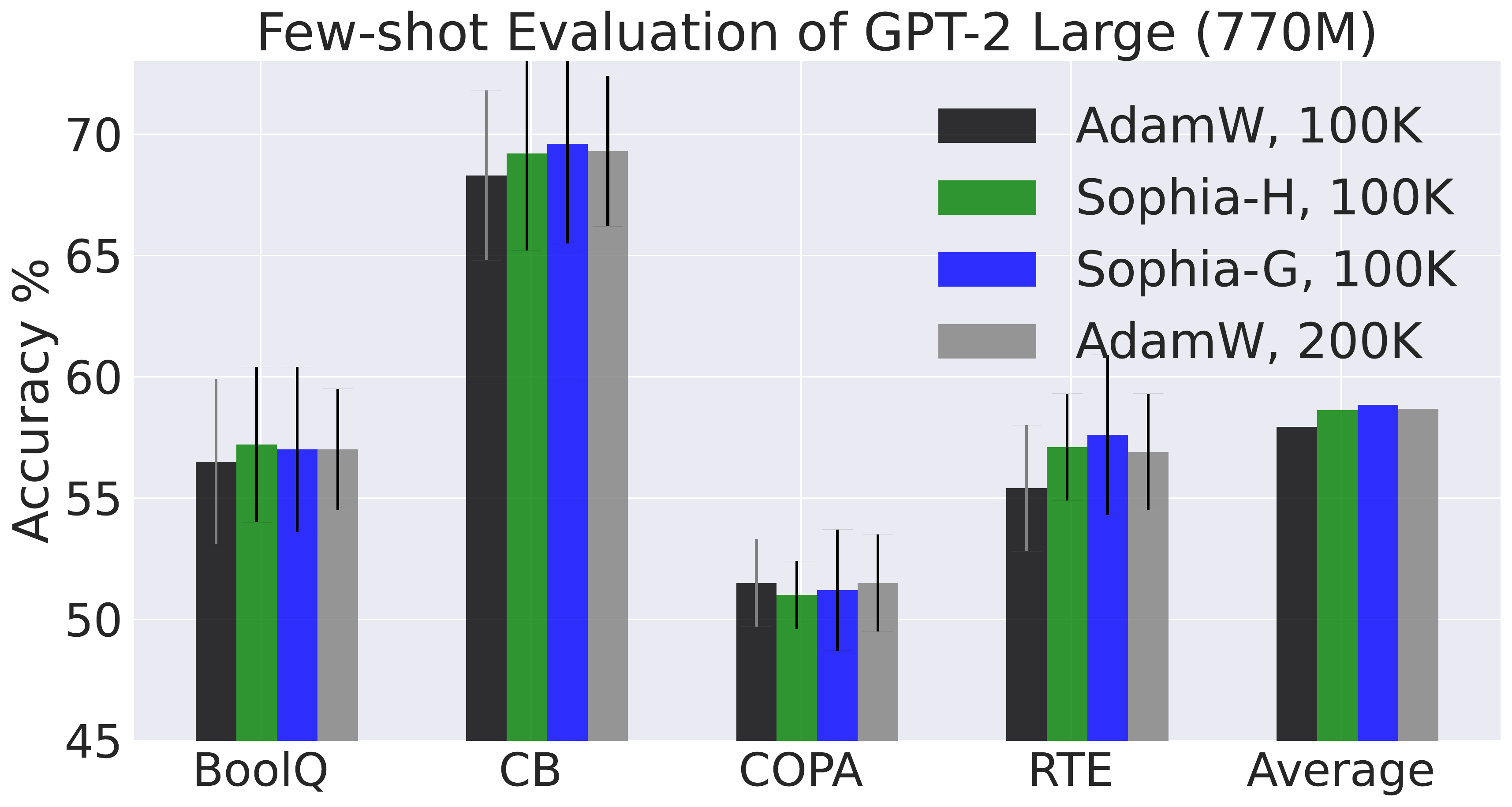}
\includegraphics[width=0.32\textwidth]{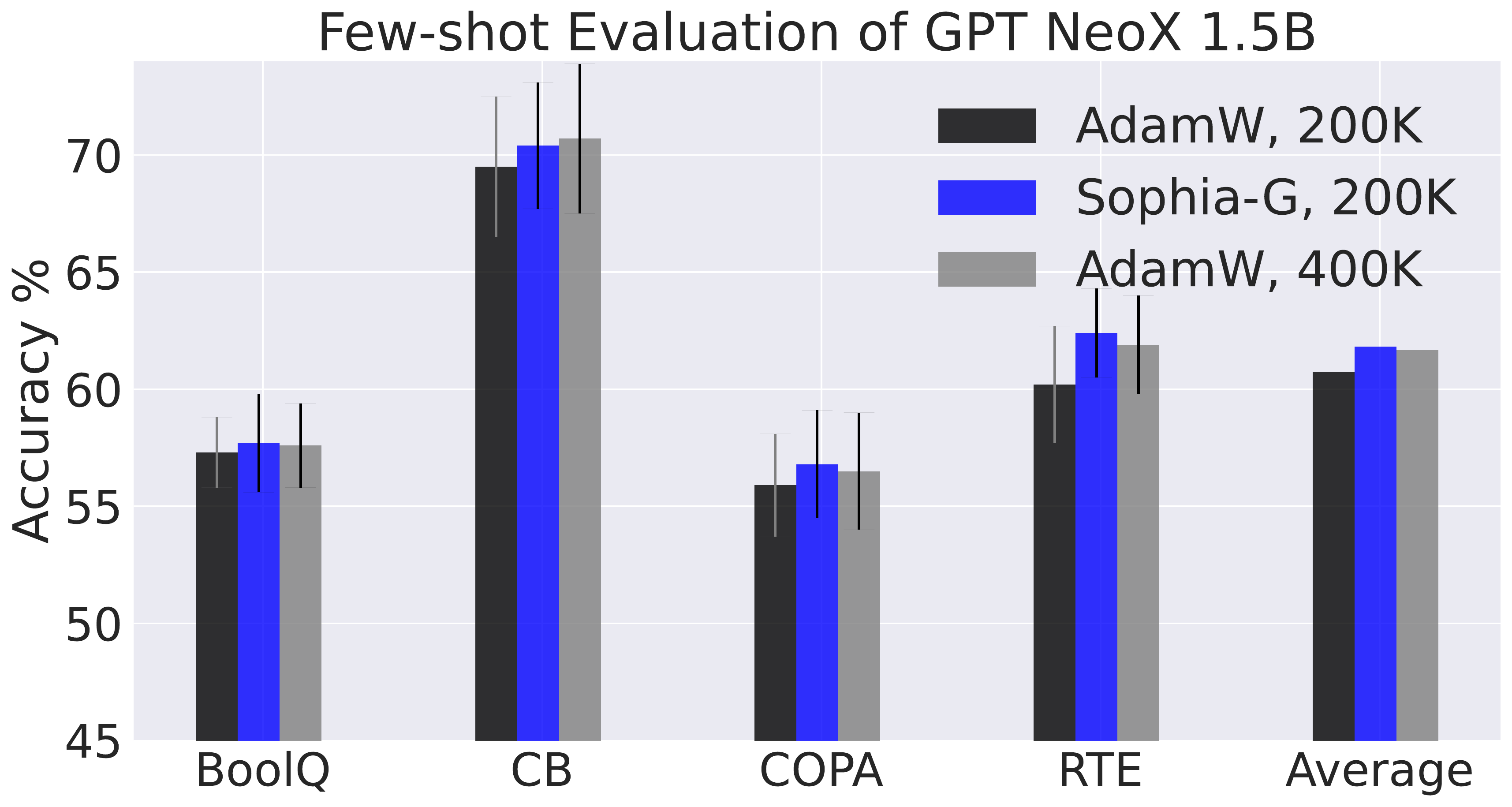}
\caption{{Few-shot evaluation on SuperGLUE. With the same number of steps, models pre-trained with \ours~outperforms models pre-trained with AdamW and Lion on most tasks. Models pre-trained with \ours~for 200K steps have comparable performance as models pre-trained with AdamW for 400K steps.\label{fig:fewshot}}
}
\end{center}
\end{figure}

\subsection{Results}\label{sec:results}

Figure~\ref{fig:gpt2} illustrates the validation loss curve (token-level log perplexity) on OpenWebText or the Pile with the same number of steps. Sophia-H and Sophia-G
consistently achieve better validation loss than AdamW, Lion, and AdaHessian, while Sophia-G is better than Sophia-H. As the model size grows, the gap between \ours~and baselines also becomes larger. Note that the cost of each step AdaHessian is more than 2x that of AdamW, while Sophia only introduces a per step compute overhead of less than 5$\%$. Sophia-H and Sophia-G both achieve a 0.04 smaller validation loss on the 355M model (Figure~\ref{fig:gpt2} (b)). Sophia-G and Sophia-H achieve a 0.05 smaller validation loss on the 770M model (Figure~\ref{fig:gpt2} (c)), with the same 100K steps. This is a significant improvement since according to scaling laws in this regime~\citep{kaplan2020scaling,hoffmann2022training} and results in Figure~\ref{fig:gpt2-1}, an improvement in loss of 0.05 is equivalent to 2x improvement in terms of number of steps or total compute to achieve the same validation loss. 

\textbf{\ours~is 2x faster in terms of number of steps, total compute and wall-clock time.} The improvement in validation loss brought by \ours~can be translated into reduction of number of steps or total compute. In Figure~\ref{fig:head} (a)(b)(c) and Figure~\ref{fig:gpt2-1}, we evaluate the optimizers by comparing the number of steps or total compute needed to achieve \textit{the same validation loss level}. As can be observed in Figure~\ref{fig:head} (a)(b)(c), Sophia-H and Sophia-G achieve a 2x speedup compared with AdamW and Lion across different model sizes.

\textbf{The scaling law is in favor of \ours~over AdamW.} In Figure~\ref{fig:head} (d), we plot the validation loss on OpenWebText of models of different sizes pre-trained for 100K steps. The gap between \ours~and AdamW grows as we scale up the models. Moreover, the 540M model trained by \ours-H has smaller loss than the 770M model trained by AdamW. The 355M model trained by \ours-H has comparable loss as the 540M model trained by AdamW. 

\textbf{Few-shot Evaluation on Downstream Tasks (SuperGLUE).} As shown in Figure~\ref{fig:fewshot}, as expected, the improvement in validation loss transfers to an improvement in downstream task accuracy. 
With the same number of steps in pre-training, GPT-2 medium, GPT-2 large and GPT NeoX 1.5B pre-trained with {\ours} have better few-shot accuracy on most subtasks. Also, models pre-trained with {\ours} have comparable few-shot accuracy as models pre-trained with AdamW for 2x number of steps. 

\subsection{Analysis}\label{sec:analysis}
\begin{wraptable}{r}{7.5cm}
	\centering
	\caption{\small Wall-clock time and compute.}
	\label{table2}
	\begin{small}
 \addtolength{\tabcolsep}{-5pt} 
	\begin{tabular}{l|c|c|c|c}
	\toprule
	Algorithm & Model Size  & T(step) & T(Hessian) & Compute \\
	\midrule
	AdamW & 770M  & 3.25s & -- & 2550 \\
	\ours-H & 770M  & 3.40s & 0.12s & 2708\\
        Sophia-G & 770M  & 3.42s & 0.17s & 2678\\
        AdamW & 355M  & 1.77s & -- & 1195\\
	\ours-H & 355M  & 1.88s & 0.09s & 1249\\
        Sophia-G & 355M  & 1.86s & 0.09s & 1255\\
 \bottomrule
	\end{tabular}
	\end{small}
\end{wraptable}
\textbf{Comparison of wall-clock time and amount of compute.} We compare the total compute (TFLOPs) per step and the wall-clock time on A100 GPUs in Table~\ref{table2}. We report the average time per step (T(step)), the time spent in Hessian computation (T(Hessian)) and the total compute following~\citet{chowdhery2022palm}. Since we calculate the diagonal Hessian estimate with a reduced batch size every 10 steps, the computation of the Hessian accounts for 6$\%$ of the total compute, and the overall wall-clock time overhead is less than 5$\%$ compared with AdamW. In terms of memory usage, our optimizer has two states, $m$ and $h$, which results in the same memory cost as AdamW.

\begin{figure}[t]
\begin{center}
\includegraphics[width=0.31\textwidth]{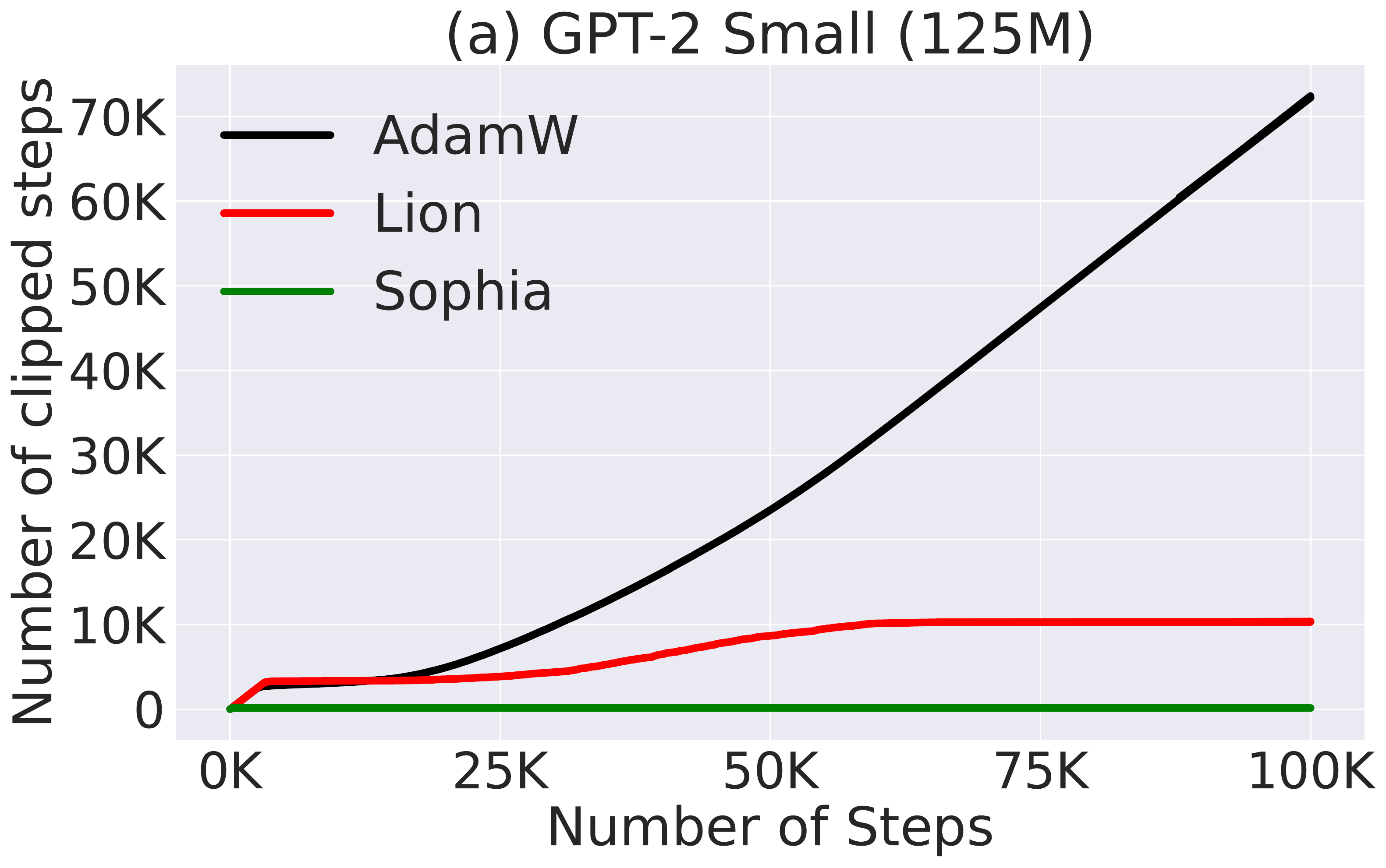}
\includegraphics[width=0.31\textwidth]{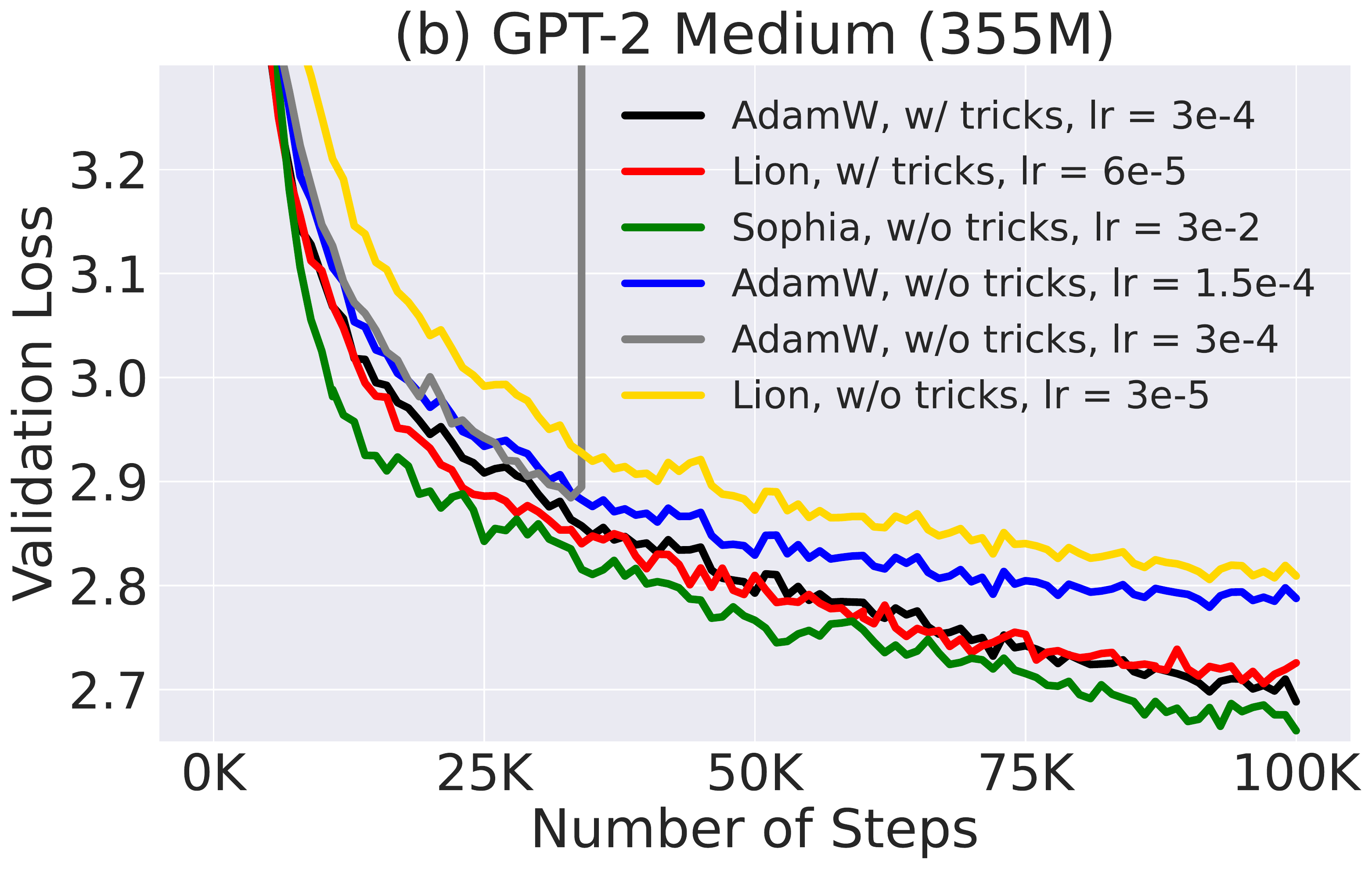}
\includegraphics[width=0.33\textwidth]{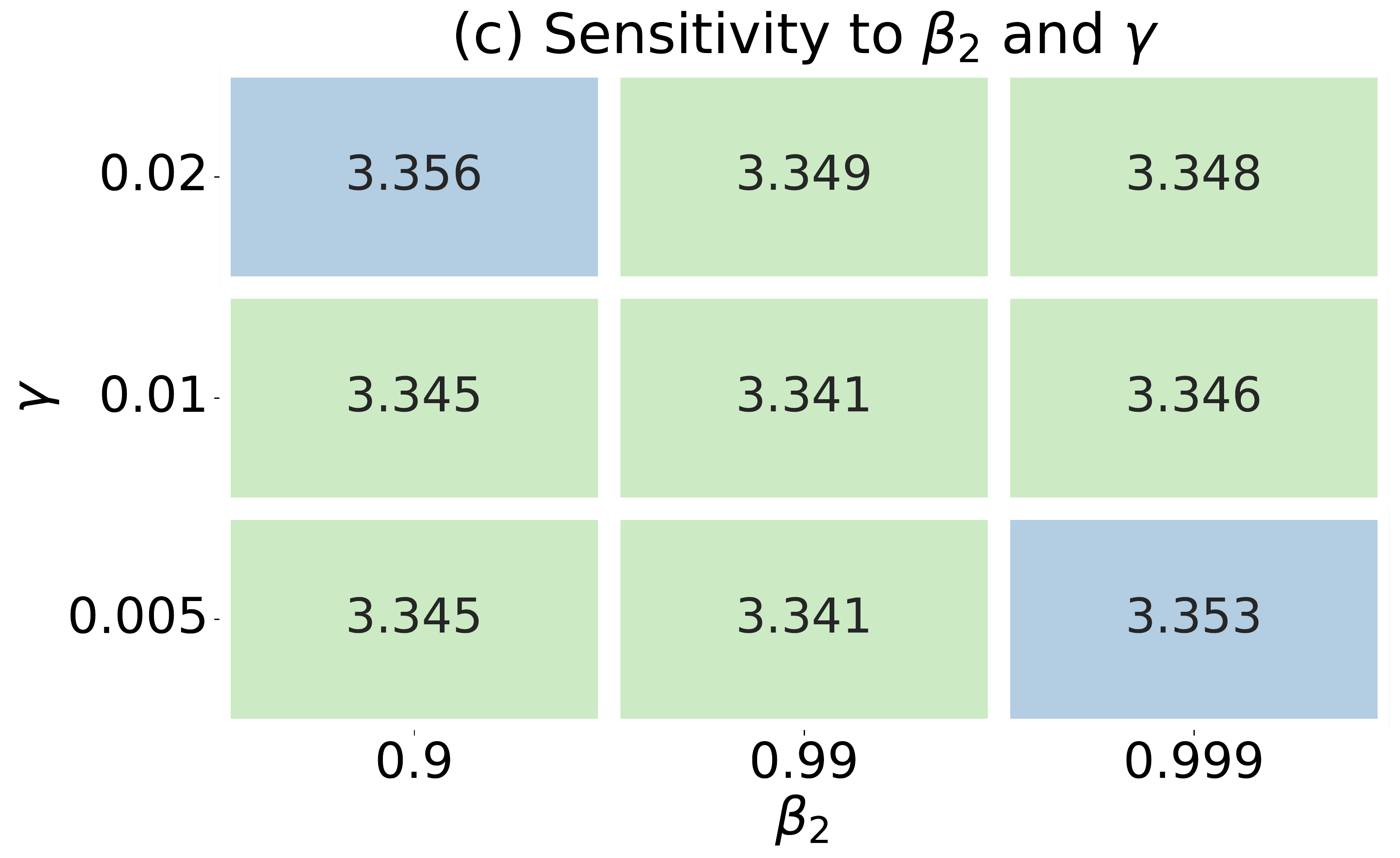}
\caption{\ours~improves pre-training stability and is insensitive to hyperparameters. (a) With AdamW and Lion, gradient clipping is triggered frequently. With \ours, gradient clipping rarely happens. (b) AdamW and Lion require the trick of re-parameterizing the attention with a temperature that is the inverse of the layer index~\citep{mistral}. The plot shows the largest LR that AdamW and Lion without the trick can use to be stable, which is much smaller than with the trick. 
In contrast, \ours~does not need this trick. (c) \ours~is not sensitive to hyperparameter choice. \label{fig:stability}}
\end{center}
\end{figure}

\textbf{Sensitivity to $\rho$ and $\beta_2$, and transferability of hyperparameters.} On a 30M model, we perform a grid search to test the sensitivity of \ours-H to hyperparamters (Figure~\ref{fig:stability} (c)). All combinations have a similar performance, while $\beta_2=0.99$ and $\gamma=0.01$ performs the best. Moreover, this hyperparameter choice is transferable across model sizes. For all the experiments on 125M, 355M and 770M, we use the hyperparameters searched on the 30M model, which is $\gamma=0.01$, $\beta_2=0.99$.

\textbf{Training Stability.} \ours-H has better stability in pre-training compared to AdamW and Lion. Gradient clipping (by norm) is an important technique in language model pre-training as it avoids messing up the moment of gradients with one mini-batch gradient computed from rare data~\citep{zhang2020adaptive}. In practice, the frequency that gradients clipping is triggered is related to the training stability---if the gradient is frequently clipped, the iterate can be at a very instable state. We compare the proportion of steps where gradient clipping is triggered on GPT-2 small (125M) in Figure~\ref{fig:stability} (a). Although all methods use the same clipping threshold 1.0, \ours-H seldomly triggers gradient clipping, while AdamW and Lion trigger gradient clipping in more than 10$\%$ of the steps.

Another common trick of pre-training deep Transformers is scaling the product of keys and values by the inverse of the layer index as implemented by Mistral~\citep{mistral} and Huggingface~\citep{wolf-etal-2020-transformers}. This stabilizes training and increases the largest possible learning rate. Without this trick, the maximum learning rate of AdamW and Lion on GPT-2 medium (355M) can only be 1.5e-4, which is much smaller than 3e-4 with the trick (the loss will blow up with 3e-4 without the trick). Moreover, the loss decreases much slower without the trick as shown in Figure~\ref{fig:stability} (b). In all the experiments, \ours-H does not require scaling the product of keys and values by the inverse of the layer index.

\begin{figure}[t]
\begin{center}
\includegraphics[width=0.315\textwidth]{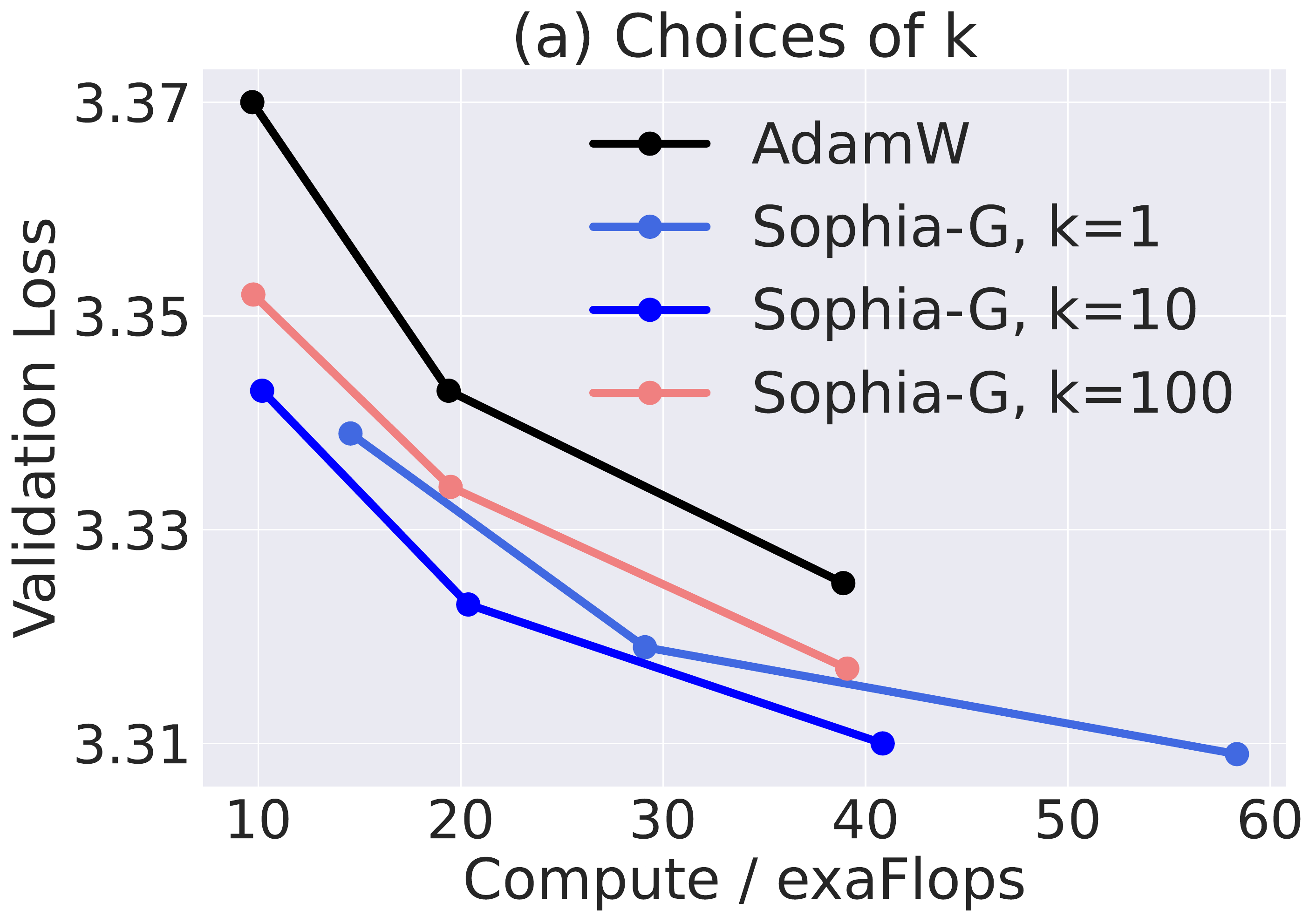}
\includegraphics[width=0.315\textwidth]{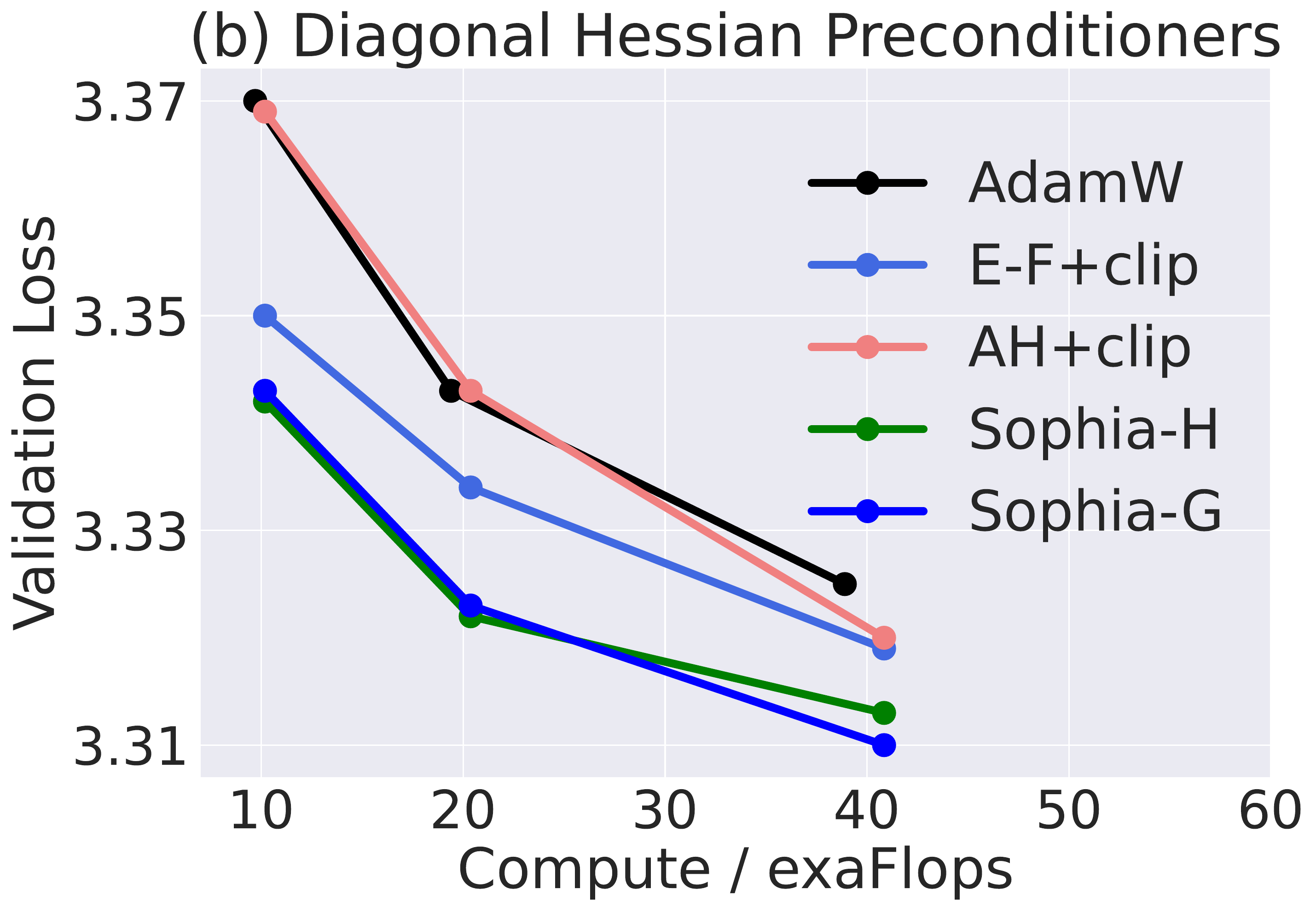}
\includegraphics[width=0.323\textwidth]{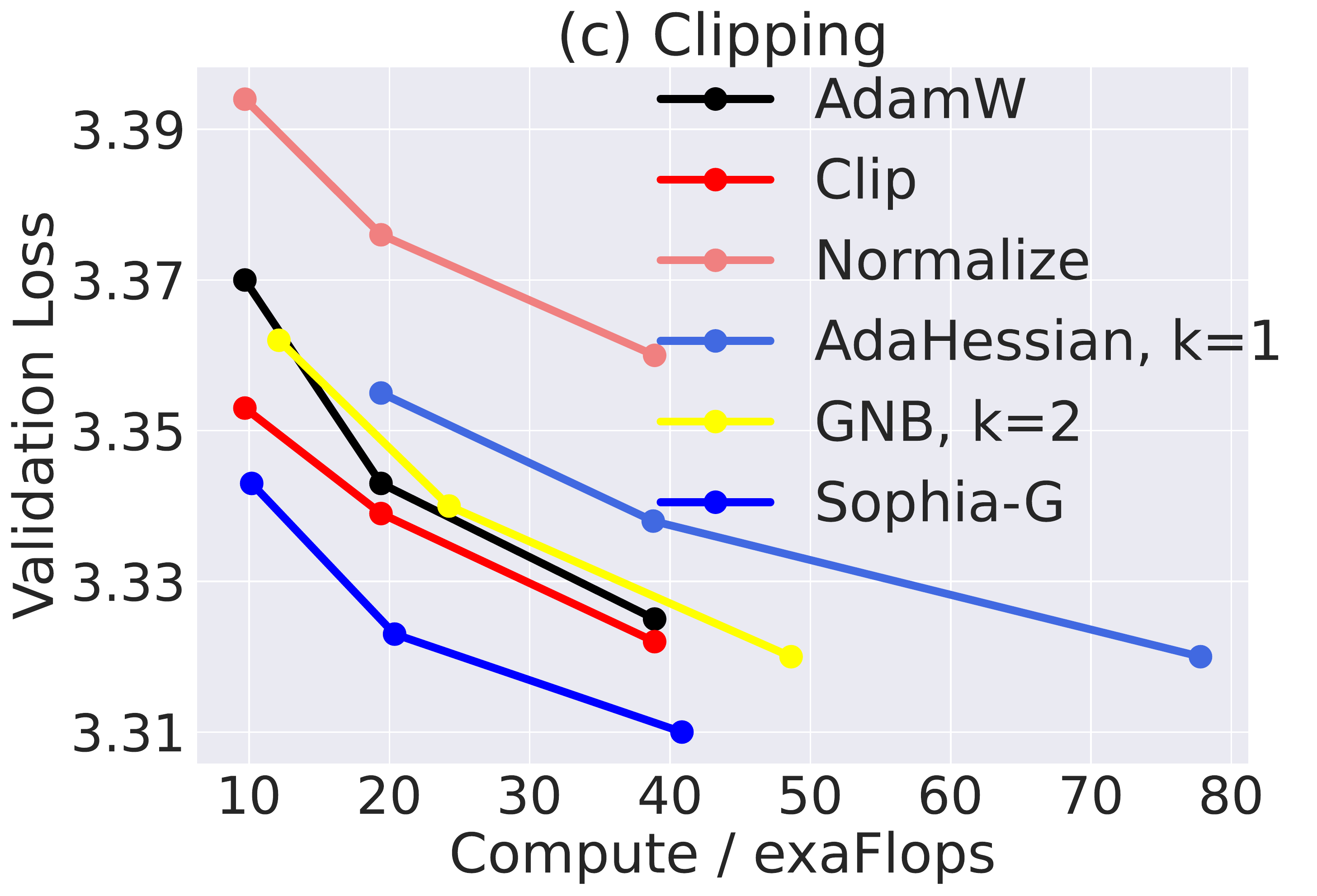}
\caption{Ablation study. (a) Hessian update frequency $k$. (b) Diagonal Hessian pre-conditioners. (c) Element-wise Update clipping. \label{fig:ablation}}
\end{center}
\end{figure}

\subsection{Ablation Study}\label{sec:ablation}

\textbf{Choices of Hessian update frequency $k$.} We study the effect of Hessian update frequency $k$ of Sophia-G on computational overhead and validation loss on a 30M GPT-2 model. We consider $k=1,10,100$ and run each method for 100k, 200k, and 400k steps. All other hyperparameters are fixed, and we tune the peak learning rate with a grid search. We plot the amount of compute spent and the validation loss of each run in Figure~\ref{fig:ablation} (a). While $k=1$ has better validation loss with the same number of steps, the computational overhead is 50$\%$ and the convergence speed with respect to amount of compute is worse than $k=10$. The choice of $k=100$ still outperforms AdamW, but is not as good as $k=10$.

\textbf{Diagonal Hessian pre-conditioners.} We compare different diagonal Hessian pre-conditioners (with the same $k=10$ and $\gamma$ found by grid search): Empirical Fisher (E-F+clip), AdaHessian (AH+clip), Hutchinson (Sophia-H), and GNB (Sophia-G). Note that empirical Fisher is the EMA of squared gradients, which differs from GNB in label sampling. We run each method for 100k, 200k, and 400k steps and plot the results in Figure~\ref{fig:ablation} (b). Results indicate that GNB is better than Empirical Fisher, which is consistent with \citet{kunstner2019limitations}. Sophia-H is also consistently better than AdaHessian. We hypothesize that the difference stems from that the EMA of the diagonal Hessian estimates (used in Sophia-H ) has more denoising effect than the EMA of the second moment of Hessian estimates (used in AdaHessian). 

\textbf{Element-wise Clipping.} We compare the role of different update clipping strategy in Figure~\ref{fig:ablation} (c). We include element-wise clipping without pre-conditioners (Clip), update normalization without pre-conditioners (Normalize), AdaHessian and Sophia-G without clipping (GNB). The learning rate is found by grid search. Note that clipping without pre-conditioner is essentially the same as sign momentum, or Lion with a single $\beta$. Without element-wise clipping, we find that AdaHessian will diverge with $k=2$ and GNB will diverge with $k=5$, thus we use $k=1$ for AdaHessian and $k=2$ for GNB. Results indicate that per-coordinate clipping itself is already better than AdamW. Further adding the GNB pre-conditioner makes Sophia-G much better than baselines.

%% file: theoretical_analysis.tex
\section{Theoretical Analysis}\label{sec:theoretical_analysis}

\newcommand{\sub}[1]{_{[#1]}}

This section provides runtime bounds for the deterministic version of Sophia that does not depend on the local condition number (the ratio between maximum and minimum curvature at the local minimum) and the worst-case curvature (that is, the smoothness parameter), demonstrating the advantage of Sophia in adapting to heterogeneous curvatures across parameter dimensions.   

We start with standard assumptions on the differentiability and uniqueness of the minimizer.
\begin{assumption}\label{assum:existence_local_minimizer_w_hessian_pd}
$L:\R^d\to\R$ is a twice continuously differentiable, strictly convex function with $\theta^*$ being its minimizer. For convenience, we denote $\lambda_{\min}(\nabla^2 L(\theta^*))$ by $\mu$.
\end{assumption}
The following assumptions state that the Hessian has a certain form of continuity---within a neighborhood of size $R$, the ratio between the Hessians, $\nabla^2 L(\theta')^{-1} \nabla^2 L(\theta)$, is assumed to be bounded by a constant 2. 
\begin{assumption}\label{assum:hessian_multiplicative_lipschitz}
	There exists a constant $R>0$, such that 
	\begin{align}
		\forall \theta,\theta'\in \R^d, \norm{\theta-\theta'}_2\le R \implies \norm{\nabla^2 L(\theta')^{-1} \nabla^2 L(\theta)}_2\le 2
	\end{align}
\end{assumption}

We analyze the convergence rate of the deterministic version of the Sophia on convex functions,  
\begin{align}\label{eq:clipped_newton_for_convex_loss}
	\theta_{t+1} = \theta_t - \eta V_t^\top\clip(V_t(\nabla^2L(\theta_t))^{-1}\nabla L(\theta_t),\rho),
\end{align}
where $\nabla^2 L(\theta_t) = V_t^\top \Sigma_t V_t$ is an eigendecomposition of $\nabla^2 L(\theta_t)$. Here, we use the full Hessian as the pre-conditioner because the diagonal Hessian pre-conditioner cannot always work for general functions which may not have any alignment with the natural coordinate system. Moreover, the  matrix $V_t$ transforms $(\nabla^2L(\theta_t))^{-1}\nabla L(\theta_t)$ into eigenspace and thus the clipping can be done element-wise in the eigenspace. We do not need the max between Hessian and $\epsilon$ in the original version of Sophia because the Hessian is always PSD for convex functions. Finally, the matrix $V_t^\top$ transforms the update back to the original coordinate system for the parameter update. 

\begin{theorem}\label{thm:convex_main}
Under \Cref{assum:existence_local_minimizer_w_hessian_pd} and \Cref{assum:hessian_multiplicative_lipschitz},  let $\eta=1/2,\rho = \frac{R}{2\sqrt{d}}$, 
the update in~\Cref{eq:clipped_newton_for_convex_loss} reaches a loss at most $\epsilon$ in $
	T\lesssim d\cdot \frac{ L(\theta_0)-\min L}{\mu R^2} + \ln \frac{\mu R^2}{32d\eps}
	$ steps. 
\end{theorem}	
The first term in the runtime bound is a burn-in time before reaching a local region, where the error decays exponentially fast so that the runtime bound is logarithmic in $1/\epsilon$ as the second term in the runtime bound shows. We remark that the bound does not depend on the condition number (the ratio between the maximum and minimum eigenvalue of Hessian), as opposed to the typical dependency on the maximum eigenvalue of the Hessian (or the smoothness parameter) in standard analysis of gradient descent in convex optimization~\citep{boyd2004convex}. Moreover, even on simple quadratic functions, the convergence rate of simplified Adam (SignGD) depends on the condition number (Appendix~\ref{sec:proofs:lowbound}). This demonstrates the advantage of Sophia in adapting to heterogeneous curvatures across parameter dimensions.

%% file: relatedwork.tex
\section{Related work} 
\label{sec:rw}
\textbf{Stochastic Adaptive First-order Optimizers in Deep Learning.} The idea of adaptive first-order optimizers dates back to RProp~\citep{braun1992rprop}. AdaGrad~\citep{duchi2011adaptive} adapted the learning rate of features by estimated geometry and assign larger learning rate to infrequent features. 
RMSProp~\citep{hinton2012neural} generalized RProp and is capable to work with smaller batch sizes. Adam~\citep{kingma2014Adam} improved RMSProp by introducing a running average of gradients, and has so far become the dominant approach to solve optimization problems in deep learning, especially for training Transformers~\citep{vaswani2017attention}. Many follow-up works proposed variants of Adam~\citep{dozat2016incorporating,shazeer2018adafactor,reddi2019convergence,loshchilov2017decoupled, zhuang2020adabelief, you2019large}. \citet{chen2023symbolic} performed a search over adaptive first-order algorithms and discovered Lion, which is a improved version of sign momentum SGD.

\textbf{Second-order Optimizers in Deep Learning.} Second-order optimizers are believed to have the potential to outperform adaptive first-order optimizers. Classical second-order optimization algorithms pre-condition the gradient with curvature information~\citep{broyden1970convergence,nesterov2006cubic,conn2000trust}. Over the years, people have developed numerous ways to adapt these methods to deep learning. To the best of our knowledge, \citet{becker1988improving} was the first to use diagonal Hessian as the pre-conditioner. \citet{martens2010deep} approximated the Hessian with conjugate gradient. \citet{schaul2013no} automatically tuned learning rate of SGD by considering diagonal Hessian. \citet{pascanu2013revisiting} considered Gaussian Newton's approximation of Hessian and Fisher information matrix. \citet{martens2015optimizing} and follow-up works \citep{ba2017distributed,george2018fast,martens2018kronecker,zhang2022eva} proposed to approximate the Hessian based on the structure of neural networks. \citet{yao2021adahessian, jahani2021doubly} proposed to use the EMA of diagonal Hessian estimator as the pre-conditioner. 

Despite these progress on deep learning applications, for decoder-only large language models, Adam still appears to the most popular optimizer. The authors of this paper suspect that many previous second-order optimizers face the challenge that the computational / memory overhead due to frequent Hessian computation hinders improvements in wall-clock time~\citep{martens2015optimizing,gupta2018shampoo}. Some of them also depend on specific model architecture or hardware structures, e.g., \citet{anil2020scalable} offloads hessian computation to CPUs, and \citet{george2018fast} needs ResNets and very large batch size to approximate the Fisher information matrix. To the best of our knowledge, there was no previous report that second-order optimizers can achieve a speed-up on large language models in total compute.

\textbf{Gradient Clipping.} Global gradient clipping has been a standard practice in pre-training language models~\citep{merity2017regularizing,radford2019language,izsak2021train,zhang2022opt}. It helps stabilizes training and avoids the effect of rare examples and large gradient noise. \citet{zhang2019gradient,mai2021stability} showed that global gradient clipping is faster than standard SGD when global smoothness does not hold. \citet{zhang2020adaptive,crawshaw2022robustness} found out per-coordinate gradient clipping can function as adaptivity. In addition to gradient clipping, \ours~is the first to clip the update (coordinate-wise) in second-order methods to avoid the effect of Hessian's changing along the trajectory and the inaccuracy of Hessian approximation. 

\textbf{Optimization Algorithms in LM Pre-training.} Adam~\citep{kingma2014Adam} (with decoupled weight decay~\citep{loshchilov2017decoupled}) has become the dominant approach for language model pre-training~\citep{vaswani2017attention,devlin2018bert,radford2019language,brown2020language,zhang2022opt,touvron2023llama}. Different from vision tasks with CNNs~\citep{he2016deep} where models trained with SGD generalize better than models trained with Adam, Adam outperforms SGD by a huge margin on language modeling tasks with Transformers~\citep{anil2019memory,liu2020understanding,kunstner2023noise}. \citet{raffel2020exploring,chowdhery2022palm} trained Transformers with AdaFactor~\citep{shazeer2018adafactor}, which is a low rank version of Adam. \citet{you2019large} proposed to make the update of Adam proportional to per-layer paramter norm to stably train LLMs. 

\section{Conclusion}
We introduced Sophia, a scalable second-order optimizer for language model pre-training. Sophia converges in fewer steps than first-order adaptive methods, while maintaining almost the same per-step cost. On language modeling with GPT models, Sophia achieves a 2x speed-up compared with AdamW in the number of steps, total compute, and wall-clock time.

\subsection*{Acknowledgements}

We thank Jeff Z. HaoChen, Neil Band, Garrett Thomas for valuable feedbacks.
HL is supported by Stanford Graduate Fellowship. The authors
would like to thank the support from NSF IIS 2211780.

%% file: appendix.tex
\newpage
\section{Additional Experiment Results}\label{sec:add_result}

\textbf{Dynamics of \ours~in training.} We measure the $\ell_2$ norm of the EMA of the diagonal Hessian $h_t$, and the proportion of parameters where clipping happens (that is, $|m_t[i]/\max\{\gamma h_t[i], \epsilon\}|$ is larger than 1) during pre-training in Figure~\ref{fig:dynamics}. After the initial stage, the norm of the Hessian steadily grows. The proportion of parameters where clipping happens approaches 60$\%$, which corroborates the importance of per-coordinate clipping in the algorithm. We also note that the clipping proportion generally should be between 50\% and 90\% for Sophia to be effective, and the users are recommended to tune $\gamma$ to achieve some a clipping proportion (See Section~\ref{sec:exp_setup} for details). 

\begin{figure}[h]
\begin{center}
\includegraphics[width=0.325\textwidth]{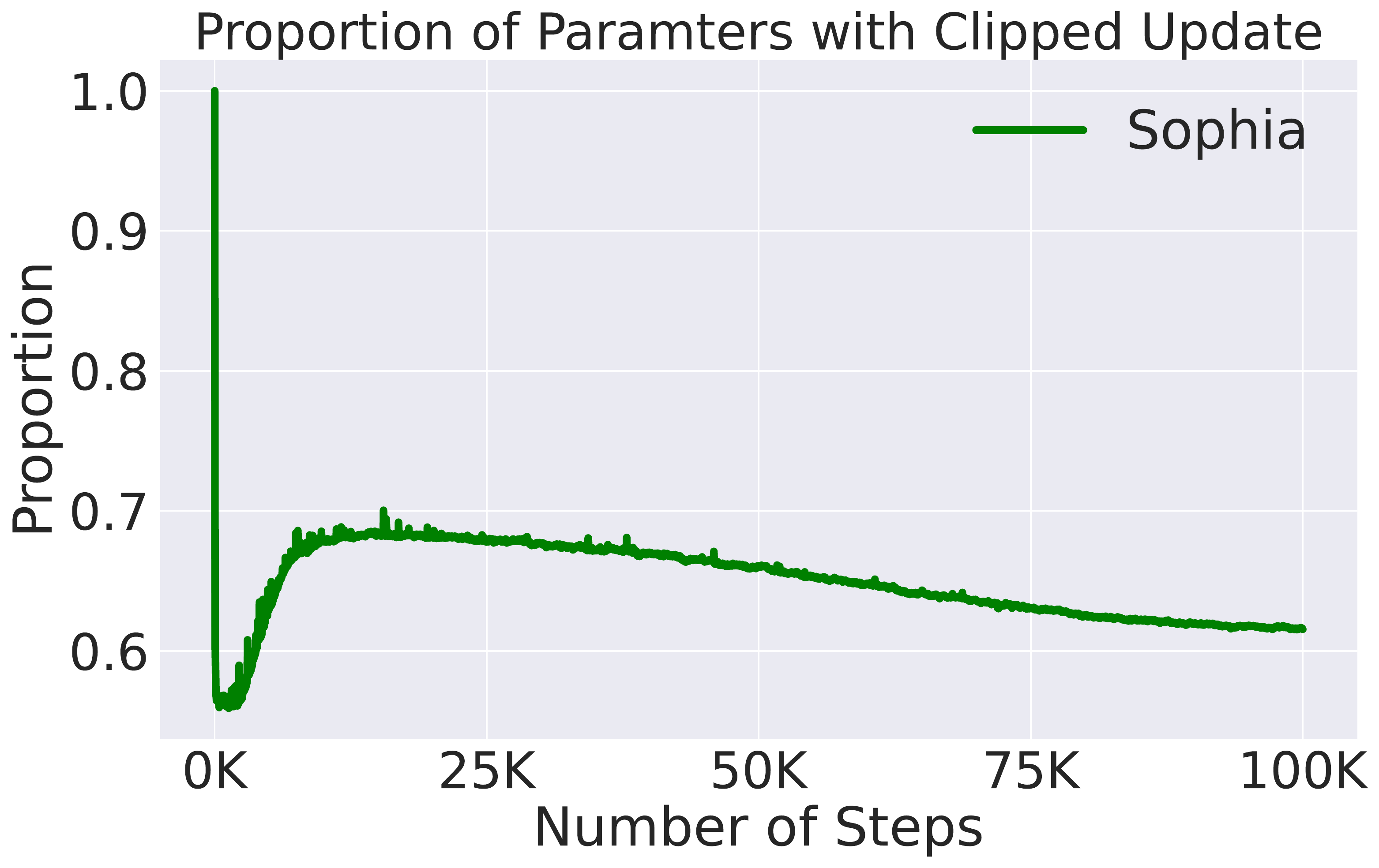}
\includegraphics[width=0.33\textwidth]{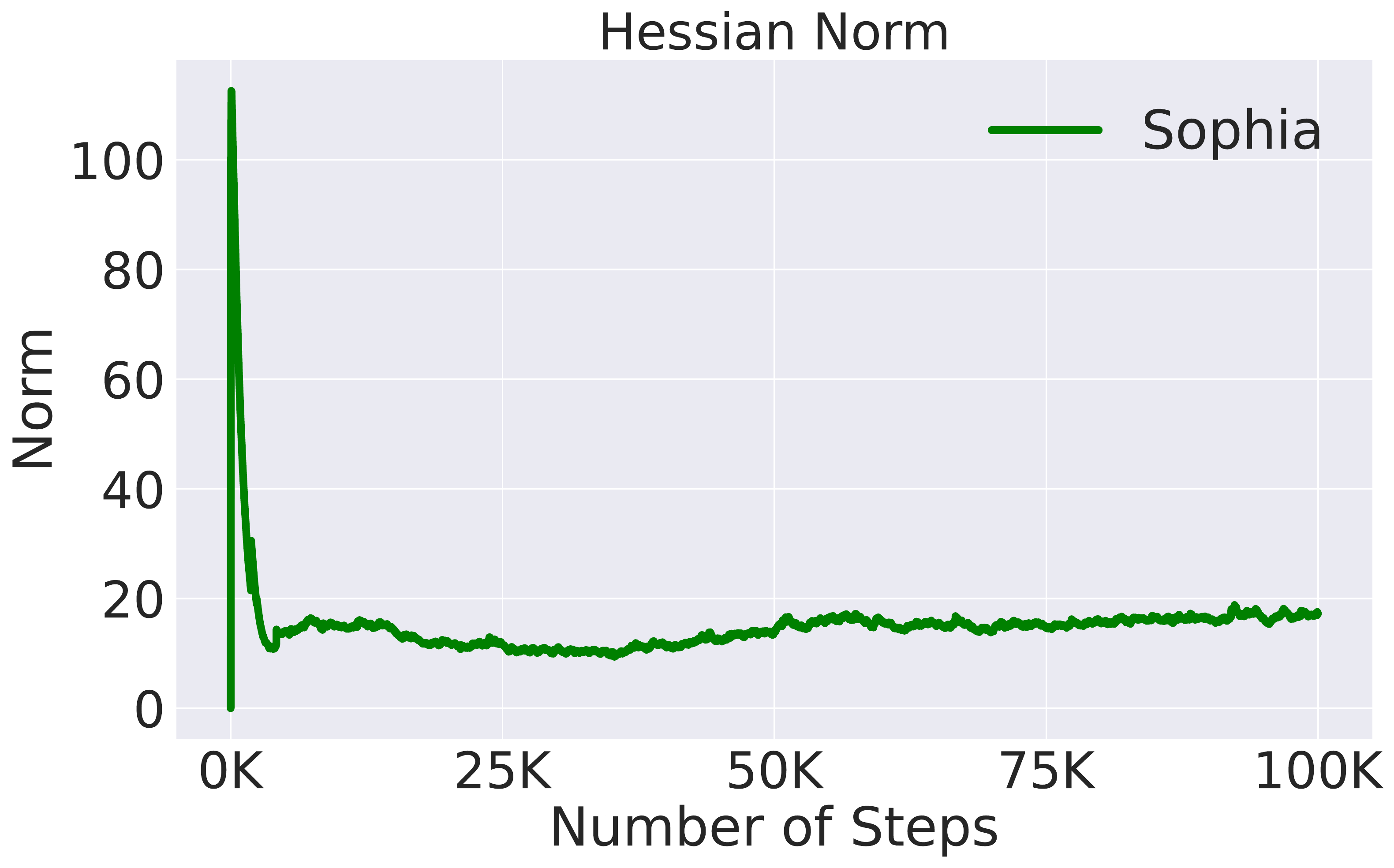}
\caption{Visualization of training statistics. (a) The proportion of parameters whose update is clipped. (b) $\ell_2$ norm of the EMA of Hessian $h_t$. \label{fig:dynamics} 
}
\end{center}
\end{figure}

\textbf{Results with different number of steps.} Runs with different number of steps and their comparison are provided in Figure~\ref{fig:add_res}. Across different choices of the total number of steps, Sophia outperforms AdamW and Lion with a large margin as the main experiments we presented in Section~\ref{sec:results}.

\begin{figure}[h]
\begin{center}
\includegraphics[width=0.452\textwidth]{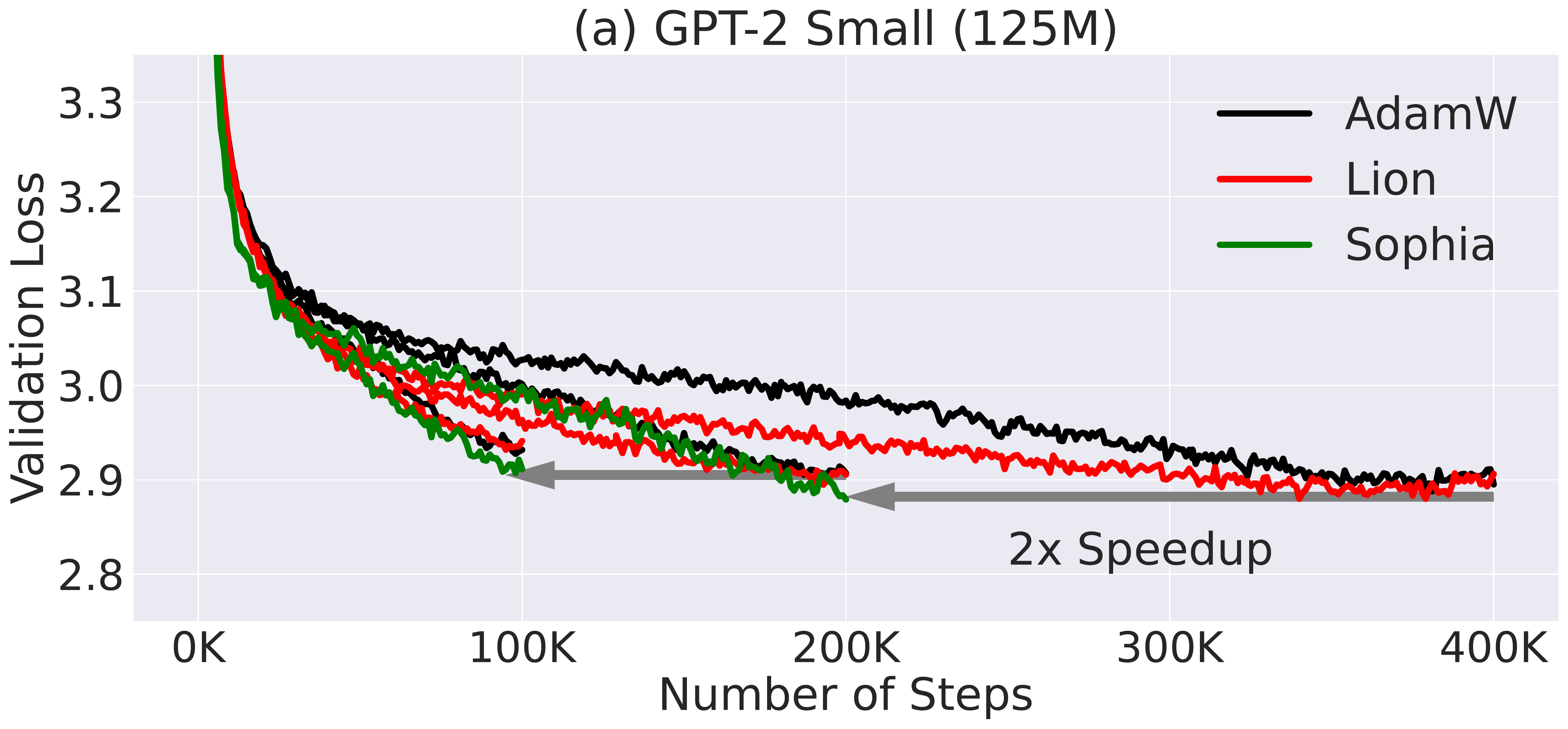}
\includegraphics[width=0.33\textwidth]{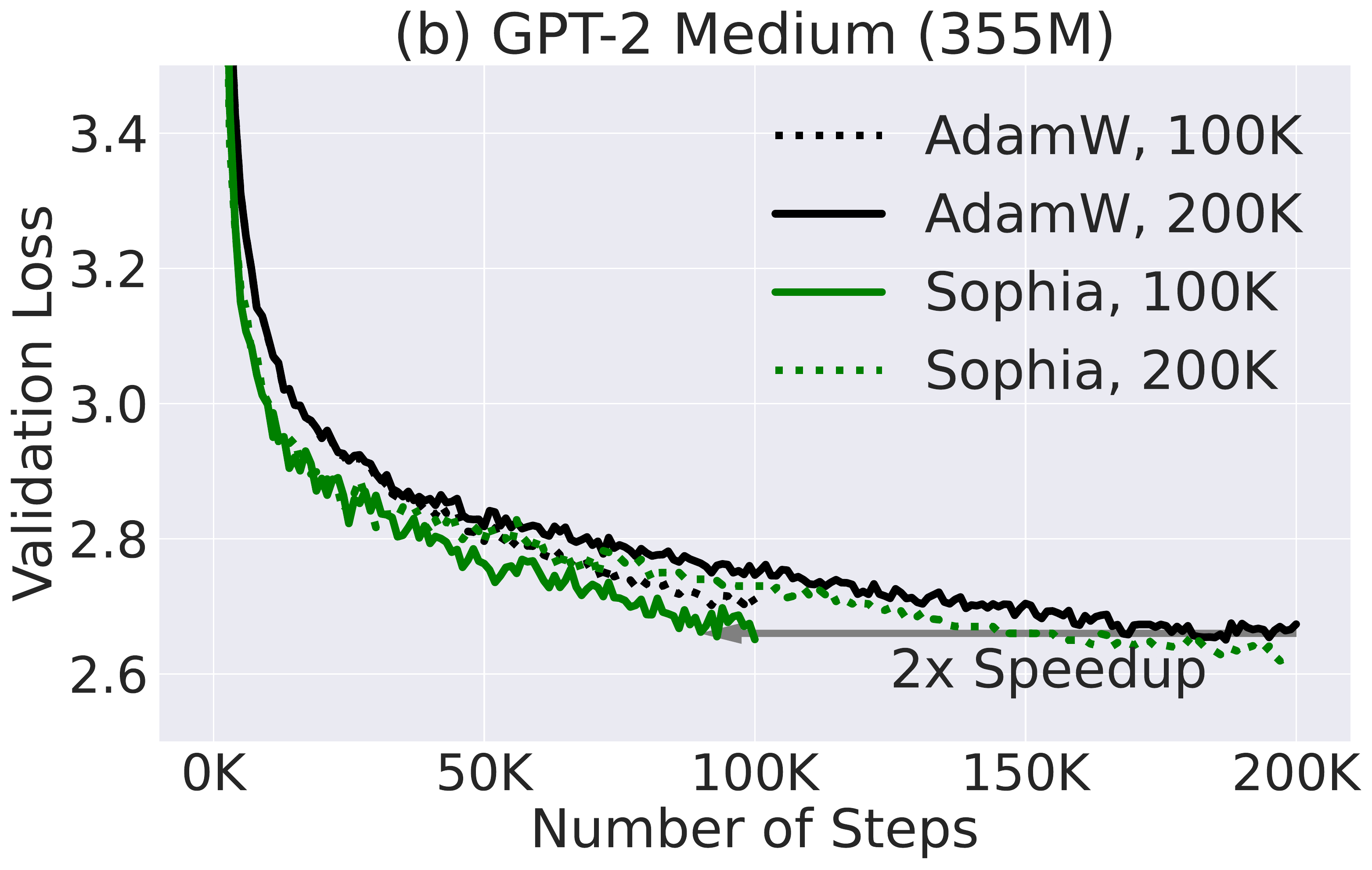}
\caption{Results of training for different steps. \label{fig:add_res} 
}
\end{center}
\end{figure}

\section{Additional Experiment Details}\label{sec:add_detail}
\subsection{Hyperparamter Tuning}\label{sec:hyper_param}
The hyperparameters we consider for baselines are as follows: peak learning rate, $\beta_1$, $\beta_2$ and weight decay. All hyperparameters except the peak learning rate are tuned with grid search on a 30M GPT-2 trained for 50K steps. The peak learning rate is tuned on models of different sizes with grid search separately. We search $\beta_1$ in $[0.8,0.9, 0.95, 0.96, 0.99]$ and $\beta_2$ in $[0.9,0.95,0.98,0.99,0.995]$. Weight decay is chosen from $0.05, 0.1,0.2,0.5$. For Lion, we also include $\beta_1 = 0.95, \beta_2=0.98$ as suggested by~\citet{chen2023symbolic}. On 30M models, we found AdamW is sensitive to the choice of $\beta_1$ but not $\beta_2$. $\beta_1=0.9$ works the best for AdamW while $\beta_1=0.95$ works the best for Lion. We use $\beta_2=0.95$ for AdamW since this is the dominantly used configuration in the LLM pre-training literature, and $\beta_2=0.98$ for Lion as it is recommended by~\citet{chen2023symbolic}. We found that weight decay 0.1 works the best for AdamW, while 0.2 works the best for Lion and Sophia. 

For peak learning rate on 125M and 355M, we perform a grid search. For larger models, we search for the largest possible learning rate with which training does not blow up for each model in the following list: [6e-4, 4e-4, 3e-4, 2e-4, 1.5e-4, 1.2e-4, 1e-4, 8e-5, 6e-5, 4e-5]. For example, a 6.6B GPT NeoX model work with 1.2e-4 peak learning rate, but the loss will blow up if we increase the learning rate to 1.5e-4 as shown in Figure~\ref{fig:66b}. The result of grid search of peak learning rate is provided in Table~\ref{table3}. We provide additional fine-grained tuning of peak learning rate of the 355M GPT-2 and 1.5B GPT Neo-X in Figure~\ref{fig:66b}. Concretely, we tried 3.6e-4, 4.2e-4, 4.5e-4 and 4.8e-4 for the 355M model, 1.8e-4 for the 1.5B GPT Neo-X. Results indicate that the learning rate choices in Table~\ref{table3} is indeed the optimal for all the values we have tried.

We use $\beta_1=0.96$, $\beta_2=0.99$, $\epsilon =$1e-12 and $k=10$ for Sophia. We adopt the following procedure to obtain these hyper parameters. We first fix $\gamma = 0.01$, $k=10$, and tune $\beta_1$ and $\beta_2$ with grid search on a 30M model, and directly use $\beta_1$ and $\beta_2$ from the 30M model on models of larger sizes. Similar to AdamW, we find that Sophia is not sensitive to $\beta_2$. We then fix $\beta_1=0.96$, $\beta_2=0.99$ and tuning $k=10$. As shown in in Figure \ref{fig:ablation} (a), $k=10$ is better than $k=1$ or $k=100$ in terms of the balance between convergence speed and the computation overhead.

After finding out $\beta_1=0.96$, $\beta_2=0.99$, $\epsilon =1e-12$ and $k=10$ with the method above, we tune $\gamma$ and peak learning rate jointly. We first tune $\gamma$ to make the proportion of coordinates where the update is not clipped (i.e., $|m_t / \max\{\gamma \cdot h_t,\epsilon\}| < 1$) in the range of $10\%-50\%$. We search for $\gamma$ in the list of [0.005,0.01,0.02,0.05,0.1,0.2]. As a result we find out $\gamma = 0.01$ works the best for Sophia-H while $\gamma = 0.05$ works the best for Sophia-G. We then fix $\beta_1, \beta_2=0.99, \epsilon, \gamma, k$ for all larger models. 

To tune the peak learning rate, we adopt the same procedure as we use for baseline methods. The result of grid search of peak learning rate is also provided in Table~\ref{table3}. 
\begin{table}
\vspace{-20pt}
	\centering
	\caption{\small Model Configurations and Peak Learning Rate.}
	\label{table3}
	\begin{small}
 \addtolength{\tabcolsep}{-3pt} 
	\begin{tabular}{l|c|c|c|c|c|c|c|c}
	\toprule
	Acronym & Size & d$\_$model & n$\_$head & depth & AdamW lr & Lion lr & Sophia-H lr & Sophia-G lr\\
	\midrule
-- & 30M & 384 & 6 & 6 & 1.2e-3 & 4e-4 & 1e-3 & 1e-3 \\
Small & 125M & 768 & 12 & 12 & 6e-4 & 1.5e-4 & 6e-4 & 6e-4\\
Medium & 355M & 1024 & 16 & 24 & 3e-4 & 6e-5 & 4e-4 & 4e-4\\
-- & 540M & 1152 & 18 & 30 & 3e-4 & -- & 4e-4 & 4e-4\\
Large & 770M & 1280 & 20 & 36 & 2e-4 & -- & 3e-4 & 3e-4\\
NeoX 1.5B & 1.5B & 1536 & 24 & 48 & 1.5e-4 & -- & -- & 1.2e-4\\
NeoX 6.6B & 6.6B & 4096 & 32 & 32 & 1.2e-4 & -- & -- & 6e-5\\
 \bottomrule
	\end{tabular}
	\end{small}
	\vspace{-5pt}
\end{table}

\subsection{Model and Implementation Details}\label{sec:imp_detail}
We consider three sizes of GPT-2 corresponding to small, medium, and large in~\citet{radford2019language}. We also introduce a 30M model for efficient hyperparameter grid search and a 540M model for scaling law visualization. We provide the model specifications in Table~\ref{table3}. We use the nanoGPT (\url{https://github.com/karpathy/nanoGPT/}) code base. Following nanoGPT, we use GELU activations and disable bias and Dropout~\cite{srivastava2014dropout} during pre-training.  

GPT-2 models are trained on OpenWebText~\citep{Gokaslan2019OpenWeb}. The text is tokenized with the GPT-2 tokenizer~\citep{radford2019language}. We use the train and validation split from nanoGPT. The training set contains 9B tokens, and the validation set contains 4.4M tokens.

We use distributed data parallel with gradient accumulation to enable a batch size of 480. All models are trained with bfloat16. The 125M and 355M models are trained on machines with 10 A5000 GPUs, while the 770M models are trained on an AWS p4d.24xlarge instance with 8 A100 GPUs. 

We consider 1.5B and 6.6B GPT NeoX~\citep{black2022gpt} models trained on the Pile~\citep{gao2022pile}. The models use GPT NeoX tokenizers. We use levanter (\url{https://github.com/stanford-crfm/levanter/tree/main}) for GPT NeoX. We use fully sharded data parallel with gradient accumulation to enable a batch size of 512 for the 1.5B model and 1024 for the 6.6B model. These models are trained on a TPU v3-128 slice.

We observed AdamW and Lion does not perform well on standard transformers which are larger than 355M. The iterates become unstable when the learning rate is close to the choice of~\citet{radford2019language}. We introduce scaling attention by the inverse of layer index to address this issue following~\citet{mistral,wolf-etal-2020-transformers}. Note that Sophia does not need this trick as mentioned in Section~\ref{sec:analysis}.

\subsection{Downtream Evaluation}\label{sec:eval}
We perform few-shot evaluation of the models on 4 subtasks of SuperGLUE. We use 2-shot prompting and greedy decoding. The prompt consists of an instruction followed by two examples. The examples are sampled from the train split while we report the accuracy on validation split averaged over 5 selection of exemplars. Prompts for each subtask are illustrated in Figure~\ref{fig:prompts}.

\begin{figure}[t]
\begin{center}
\includegraphics[width=1.0\textwidth]{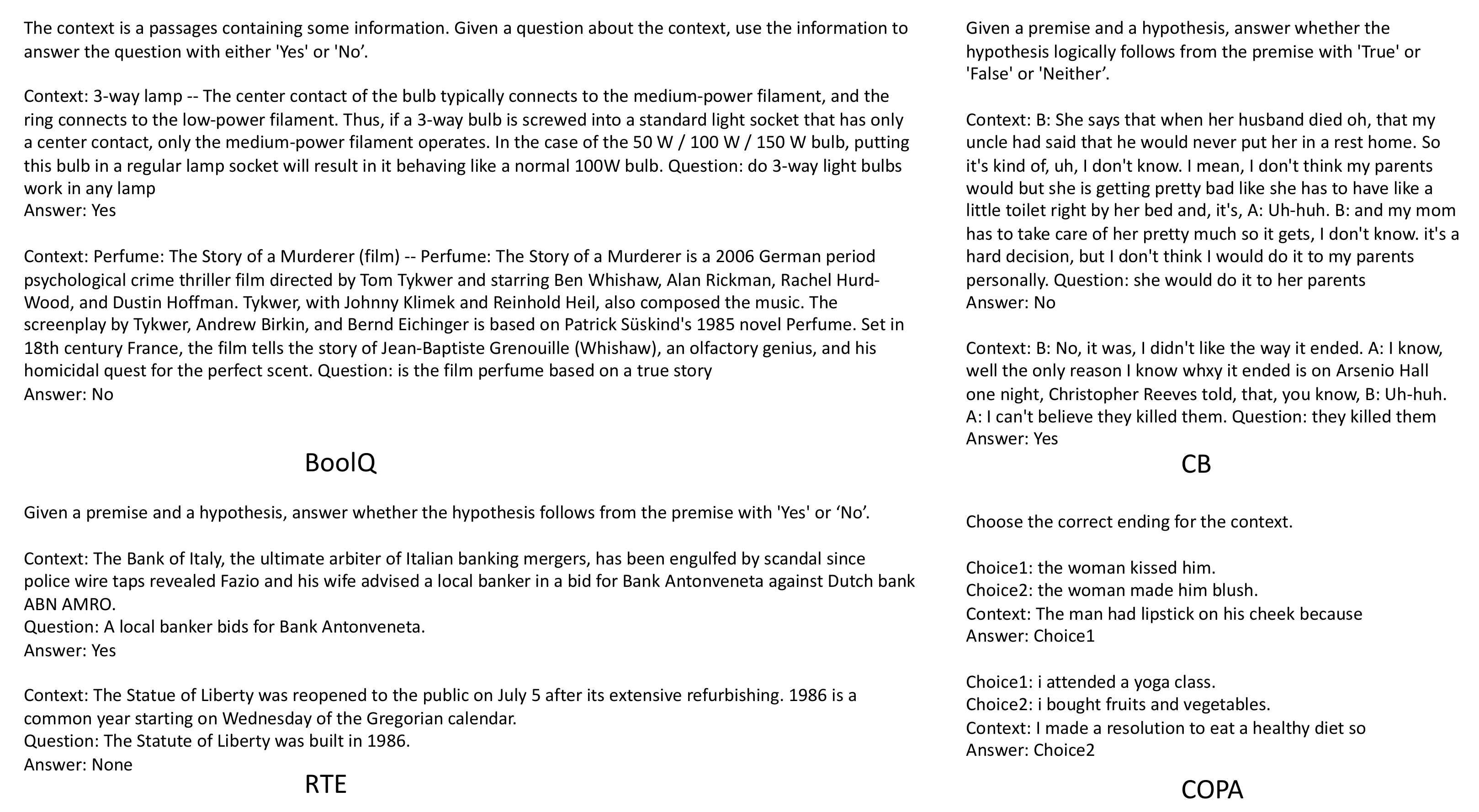}
\caption{\small{Prompts for SuperGLUE downstream evaluation. \label{fig:prompts}}
}
\vspace{-20pt}
\end{center}
\end{figure}

\begin{figure}[h]
\begin{center}
\includegraphics[width=0.252\textwidth]{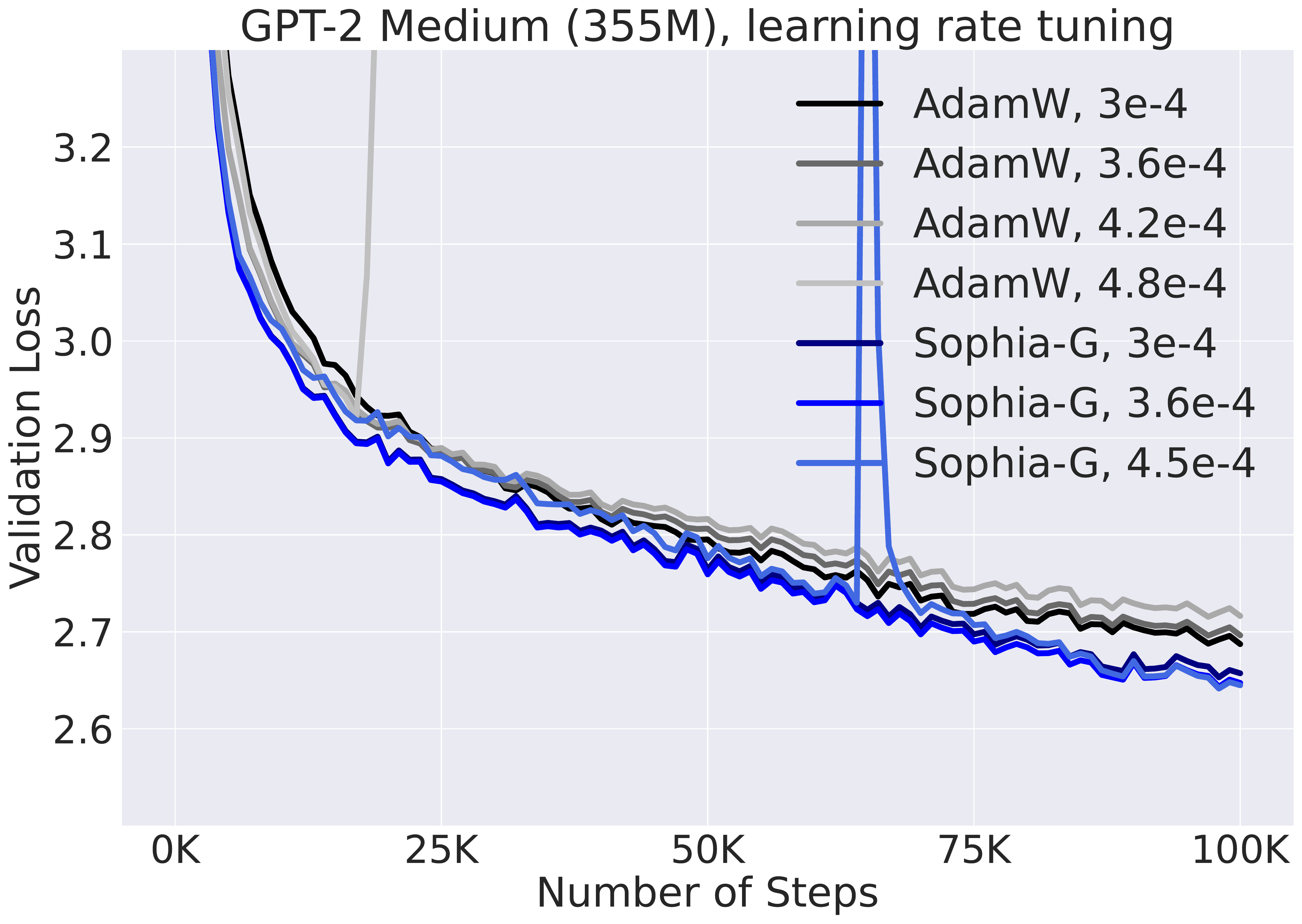}
\includegraphics[width=0.252\textwidth]{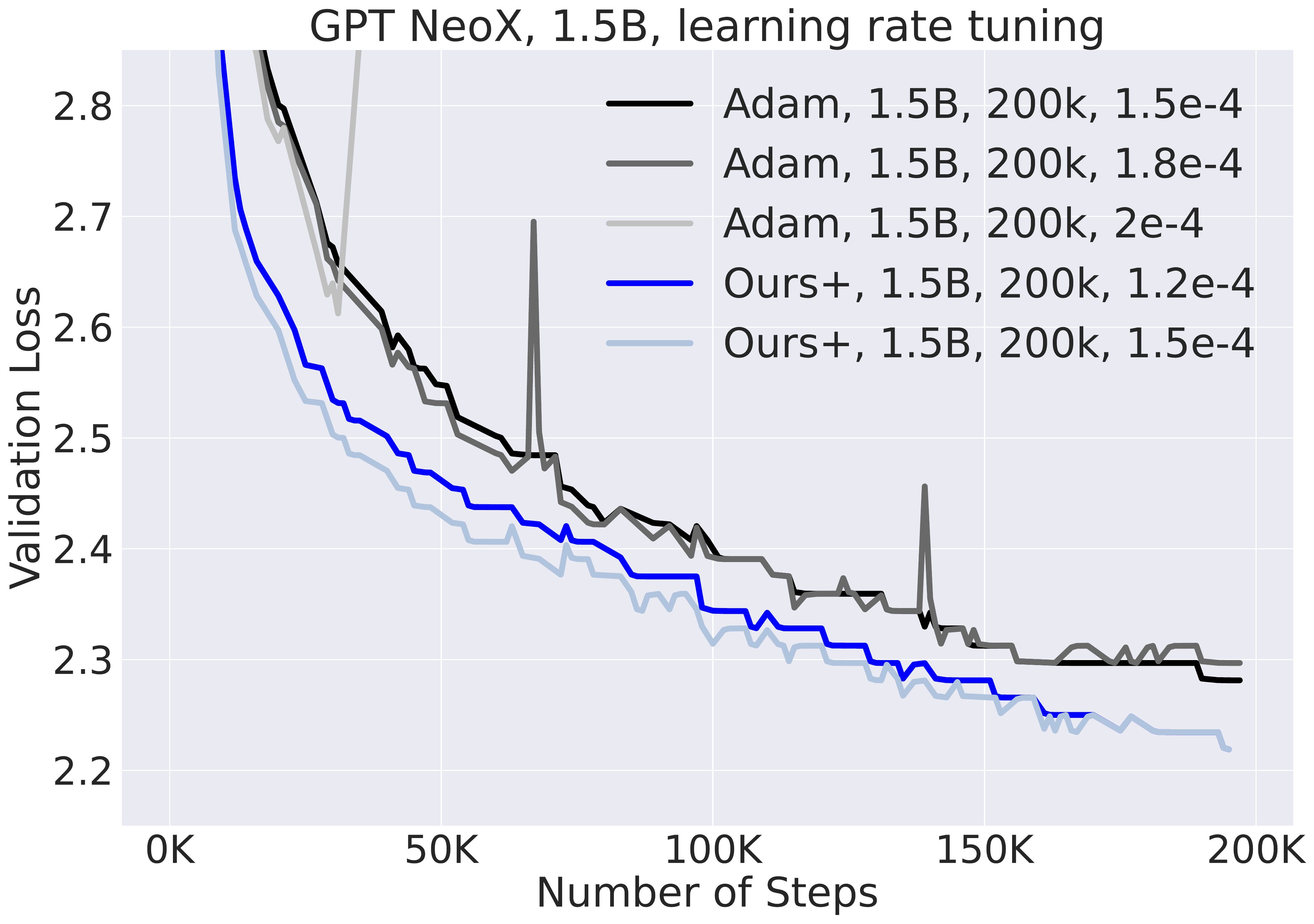}
\includegraphics[width=0.302\textwidth]{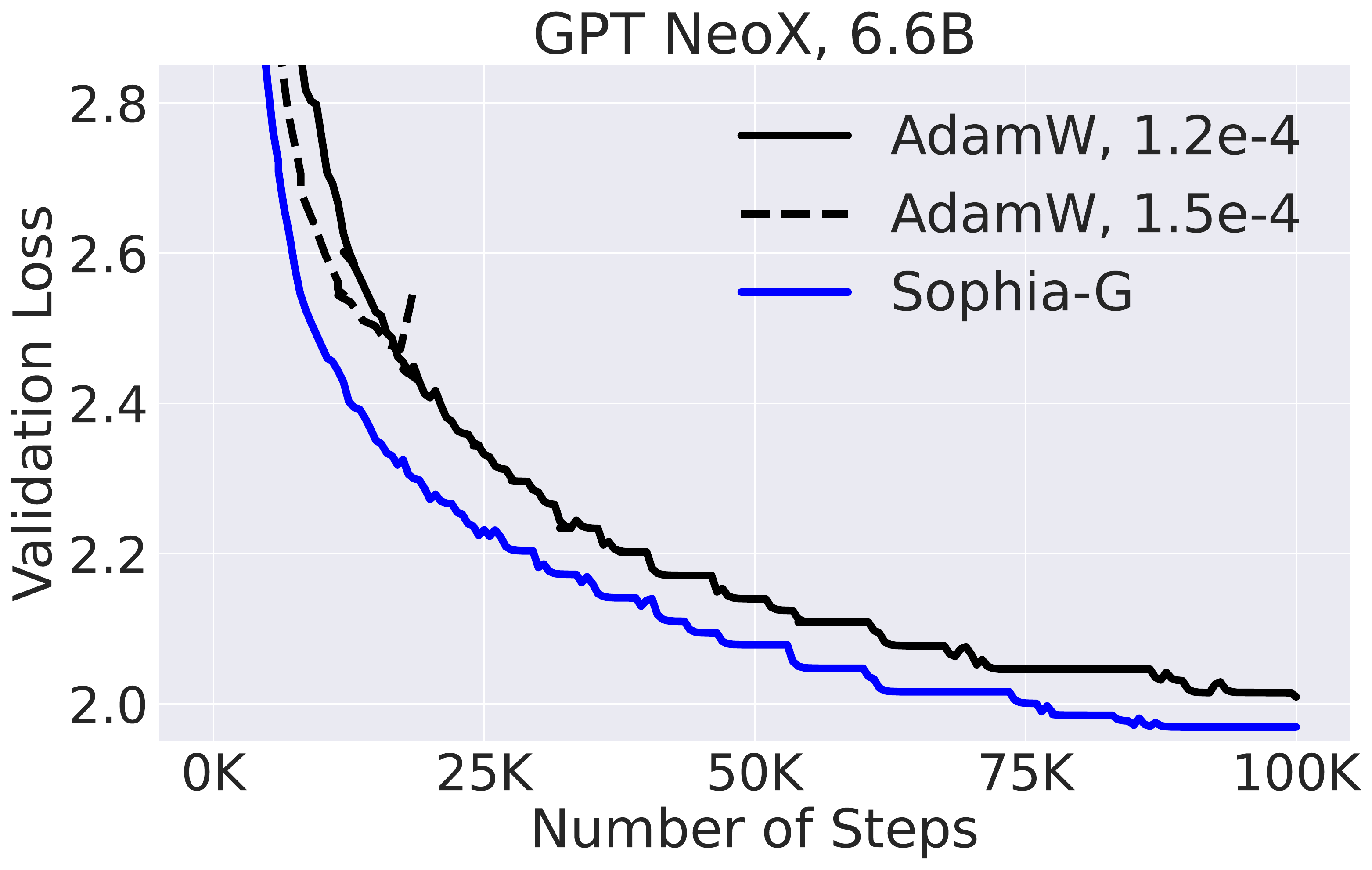}
\caption{Results peak learning rate tuning. \label{fig:66b} 
}
\end{center}
\end{figure}

\newpage

\section{Limitations}
\textbf{Scaling up to larger models and datasets.} Although \ours~demonstrates scalability up to 6.6B-parameter models, and there is no essential constraints from further scaling up, we do not compare with AdamW and Lion on larger models due to limited resources. We believe \ours~is faster than AdamW and Lion on larger models given the improvement in scaling laws and better pre-training stability.

\textbf{Holistic downstream evaluation.} We evaluate pre-trained checkpoints on the pre-training validation losses and only four SuperGLUE subtasks.
The limited evaluation in downstream evaluation is partly due to the limited model size, because language models at this scale do not have enough capabilities such as in-context learning, and mathematical reasoning.

\textbf{Evaluation on other domains.} While this paper focuses on optimizers for large language modeling, a more general optimizer should also be evaluated in other domains such as computer vision,  reinforcement learning, and multimodal tasks. Due to the limitation of computation resources, we leave the application to other domains and models to future works. 

\section{Theoretical Analyses: Details of Section~\ref{sec:theoretical_analysis}\label{sec:proofs}}

\Cref{thm:convex_main} is  a direct combination of the \Cref{lem:descent_lemma} (Descent Lemma), \Cref{lem:small_decrement_imply_small_loss} and \Cref{lem:2nd_phase_convex}. In the analysis, there will be two phases. In the first phase decrease loss to $\frac{\mu\rho^2}{8}$ in $8\frac{ L(\theta(0))-\min L}{\eta\mu\rho^2}$ steps. In the second phase, there will be an exponential decay of error.

\begin{lemma}\label{lem:strictly_convex_loss}
	Under \Cref{assum:existence_local_minimizer_w_hessian_pd}, we have that  $L(\theta)\to \infty$ whenever $\|\theta\|_2\to\infty$.
\end{lemma}
\begin{proof}[Proof of \Cref{lem:strictly_convex_loss}]
By convexity of $L$, we have $\forall \theta\in\R^d$ with $\norm{\theta-\theta^*}_2\ge 1$,
\begin{align}
	\frac{1}{\norm{\theta-\theta^*}_2}L(\theta) + \frac{\norm{\theta-\theta^*}_2 - 1}{\norm{\theta-\theta^*}_2}L(\theta^*) \ge L(\theta^* + \frac{\theta - \theta^*}{\norm{\theta - \theta^*}_2}) \ge \min_{\norm{\bar\theta}_2 =1  } L(\theta^* +\bar\theta).
\end{align} 

Since $L$ is strictly convex, $\Delta \triangleq \min_{\norm{\bar\theta}_2 =1  } L(\theta^* +\bar\theta) - L(\theta^*)>0$. Thus we conclude that 
\begin{align}
	L(\theta) \ge \norm{\theta-\theta^*}_2\Delta + L(\theta^*)\ge 	(\norm{\theta}_2 - \norm{\theta^*}_2) \Delta + L(\theta^*).
\end{align}
Therefore when $\norm{\theta}_2\to \infty$, $L(\theta)\to\infty$ as well.
\end{proof}

Note that we don't assume the Hessian of loss is Lipschitz. \Cref{assum:hessian_multiplicative_lipschitz} only assumes the Hessian in a neighborhood of constant radius only differs by a constant in the multiplicative sense. 
\begin{lemma}\label{lem:small_loss_imply_small_distance}
For any $\theta\in\R^d$ satisfying $L(\theta)-\min L\le \frac{\mu R^2}{4}$, it holds that $\norm{\theta-\theta^*}_2\le 2\sqrt{\frac{L(\theta) - \min L}{\mu}}\le R$.
\end{lemma}

\begin{proof}[Proof of \Cref{lem:small_loss_imply_small_distance}]
	We will prove by contradiction. Suppose there exists such $\theta$ with $L(\theta)\le \frac{\mu R^2}{4}$ but $\norm{\theta-\theta^*}_2> 2\sqrt{\frac{L(\theta)-\min L}{\mu}}$. We consider $\theta'\triangleq \theta^* + \sqrt{\frac{2L(\theta)}{\mu}}\cdot \frac{\theta-\theta^*}{\norm{\theta-\theta^*}_2}$. Since $\theta'$ is between $\theta$ and $\theta^*$ and that $L$ is strictly convex, we know that $L(\theta')<L(\theta)$. However, by Taylor expansion on function $f(t)\triangleq L(\theta^* + t(\theta'-\theta^*))$, we have that 
	\begin{align}
	f(1) = f(0)+f'(0)+ \frac{f''(t) 	}{2}, \quad \text{for some $t\in[0,1]$.}
	\end{align}
Note that $\norm{\theta' - \theta^*}_2\le \norm{\theta - \theta^*}_2 \le R$, by  \Cref{assum:hessian_multiplicative_lipschitz} and \Cref{assum:existence_local_minimizer_w_hessian_pd},
 we have $f''(t)  = (\theta'-\theta^*)^\top \nabla^2 L(t\theta' +(1-t) \theta^*)(\theta'-\theta^*) \ge \frac{1}{2}(\theta'-\theta^*)^\top \nabla^2 L(\theta^*)(\theta'-\theta^*) \ge \frac{\mu}{2} \norm{\theta'-\theta^*}_2^2 =  2(L(\theta)-\min L))$.
  Also note that $f(1) = L(\theta'), f(0)=L(\theta^*)$ and $ f'(0)=0$, we conclude that $L(\theta')-L(\theta^*) \ge L(\theta) - L(\theta^*)$, namely $(\theta')\ge L(\theta)$. Contradiction! 
\end{proof}

\begin{lemma}\label{lem:small_gradient_imply_small_loss}
	For any $\theta\in\R^d$ satisfying that $\norm{\nabla L(\theta)}_2\le \frac{R\mu}{2}$, it holds that $\norm{\theta-\theta^*}_2\le \frac{2\norm{\nabla L(\theta)}}{\mu}\le R$.
\end{lemma}

\begin{proof}[Proof of \Cref{lem:small_gradient_imply_small_loss}]
	We will prove by contradiction. We consider function $f(t)\triangleq \inner{\frac{\theta-\theta^*}{\norm{\theta-\theta^*}_2}}{\nabla L(\theta^*+ t\cdot \frac{\theta-\theta^*}{\norm{\theta-\theta^*}_2})}$. Because of the strict convexity of $L$, $f$ is a strict monotone increasing function. If $\norm{\theta-\theta^*}> \frac{2\norm{\nabla L(\theta)}}{\mu}$ but $\norm{\nabla L(\theta)}_2\le \frac{R\mu}{2}$, then we have $f(R)<f(\norm{\theta-\theta^*}_2)\le \norm{\nabla L(\theta)}_2 $. On the other hand, by \Cref{assum:hessian_multiplicative_lipschitz} and \Cref{assum:existence_local_minimizer_w_hessian_pd}, $f'(t)\ge \frac{\mu}{2}$ for $t\in[0,R]$. Thus $f(R)\ge f(0) + \int_{t=0}^{ \frac{2\norm{\nabla L(\theta)}}{\mu}} f'(t)\diff t =  \norm{\nabla L(\theta)} $. Contradiction! 
\end{proof}

\begin{lemma}\label{lem:soln_existence_mirror_ode}
	For any $\theta\in\R^d$, the following differential equation has at least one solution on interval $[0,1]$:
	\begin{align}\label{eq:mirror_ode}
		\frac{\diff \theta(t)}{\diff t} = - (\nabla^2 L(\theta(t)))^{-1} \nabla L(\theta), \quad \theta(0)=\theta,
	\end{align}
	and the solution satisfies that $\nabla L(\theta(t)) = (1-t)\nabla L(\theta)$ for all $t\in[0,1]$ and $\theta(0)=\theta^*$.
\end{lemma}
\begin{proof}[Proof of \Cref{lem:soln_existence_mirror_ode}]
	Since $\nabla^2 L$ is continuous and positive definite by \Cref{assum:existence_local_minimizer_w_hessian_pd} , $(\nabla^2 L)^{-1}$ is continuous and thus the above ODE~\eqref{eq:mirror_ode} has a solution over interval $[0,T)$ for some positive $T$ and we let $T_{\max}$ be the largest positive number such that the solution exists (or $T_{\max}=\infty$). Now we claim $T_{\max}\ge 1$, otherwise $\norm{\theta(t)-\theta^*}_2$ must diverge to infinity when $t\to T_{\max}$. However, for any $t\le 1$, we have 
	\begin{align}
	\frac{\diff \nabla L(\theta(t))}{\diff t} = -  \nabla L(\theta),
	\end{align}
	which implies that $\nabla L(\theta(t)) = (1-t)\nabla L(\theta)$ for all $t\in[0,1]$. Therefore,
	\begin{align}
	\frac{\diff L(\theta(t))}{\diff t} = - 	(\nabla L(\theta(t)))^\top (\nabla^2 L(\theta(t)))^{-1} \nabla L(\theta) =(1-t)	(\nabla L(\theta))^\top (\nabla^2 L(\theta(t)))^{-1} \nabla L(\theta)\le 0.
	\end{align}
	Thus $L(\theta(t))\le L(\theta(0))$. By \Cref{lem:strictly_convex_loss}, we know that $\|\theta(t)\|$ remains bounded for all $t\in[0,T_{\max}]$, thus $T_{\max}\ge 1$. Note that  $\theta(1)$ has zero gradient, $\theta(1)$ must be $\theta^*$. This completes the proof. 
\end{proof}

\begin{lemma}\label{lem:small_loss_or_gradient_imply_loss_approx}
For any $\theta\in\R^d$ satisfying (1) $L(\theta)-\min L\le \frac{\mu R^2}{16}$ or (2) $\norm{\nabla L (\theta)}_2 \le \frac{R\mu}{4}$, it holds that \begin{align}
L(\theta) - \min L	\le \nabla L(\theta)^\top (\nabla^2 L(\theta))^{-1}\nabla L(\theta) \le  4(L(\theta) - \min L).
 \end{align}
\end{lemma}

\begin{proof}[Proof of \Cref{lem:small_loss_or_gradient_imply_loss_approx}]

Let $\{\theta(t)\}_{t=0}^1$ be the solution of \Cref{eq:mirror_ode}.
We know that $\nabla L(\theta(t)) = (1-t)\nabla L(\theta)$ for all $t\in[0,1]$  and that $\theta(1) = \theta^*$  by \Cref{lem:soln_existence_mirror_ode}.
 For case (1), by \Cref{lem:small_loss_imply_small_distance}, we know that for any $t\in[0,1]$, $\norm{\theta(t)-\theta^*}_2\le R/2$. 
 For case (2), by \Cref{lem:small_gradient_imply_small_loss}, we know that for any $t\in[0,1]$, $\norm{\theta(t)-\theta^*}_2\le R/2$. Thus in both two cases, $\norm{\theta(t)-\theta}_2=\norm{\theta(t)-\theta(0)}_2=\le \norm{\theta(t)-\theta^*} + \norm{\theta(0)-\theta^*}\le R$. By~\Cref{assum:hessian_multiplicative_lipschitz}, it holds that 
\begin{align}\label{eq:small_loss_hessian_sanwidch}
	2(\nabla^2 L(\theta))^{-1} \succeq (\nabla^2 L(\theta(t)))^{-1}\succeq \frac{1}{2}(\nabla^2 L(\theta))^{-1}.
\end{align}
for all $t\in[0,1]$. Therefore, we have that 
\begin{align}\label{eq:small_loss_loss_integral}
L(\theta) - \min L = L(\theta(0))-L(\theta(1)) 
=&\int_{t=0}^1  	(\nabla L(\theta(t)))^\top (\nabla^2 L(\theta(t)))^{-1} \nabla L(\theta)\notag\\
=&\int_{t=0}^1 	(1-t)	(\nabla L(\theta))^\top (\nabla^2 L(\theta(t)))^{-1} \nabla L(\theta).
\end{align}
The proof is completed by plugging \Cref{eq:small_loss_hessian_sanwidch} into \Cref{eq:small_loss_loss_integral} and noting that $\int_{t=0}^1(1-t) = 1/2$.
\end{proof}

\begin{lemma}\label{lem:small_loss_or_gradient_imply_PL}
For any $\theta\in\R^d$ satisfying (1) $L(\theta)-\min L\le \frac{\mu R^2}{4}$ or (2) $\norm{\nabla L (\theta)}_2 \le \frac{R\mu}{2}$, it holds that \begin{align}
L(\theta) - \min L	\le \mu^{-1}\norm{\nabla L(\theta)}_2^2  \end{align}
\end{lemma}
\begin{proof}[Proof of \Cref{lem:small_loss_or_gradient_imply_PL}]
	The proof of \Cref{lem:small_loss_or_gradient_imply_PL} is almost the same as that of \Cref{lem:small_loss_or_gradient_imply_loss_approx} and thus omitted.
\end{proof}

\begin{lemma}\label{lem:small_loss_imply_no_clipping}
	For any $\theta\in\R^d$ satisfying $L(\theta)-\min L\le \frac{\mu R^2}{16}$, it holds that \begin{align}
 \norm{(\nabla^2 L(\theta))^{-1}\nabla L(\theta)}_2 \le \sqrt{\frac{8(L(\theta) - \min L)}{\mu}}.
 \end{align}

\end{lemma}
\begin{proof}[Proof of \Cref{lem:small_loss_imply_no_clipping}]
	By \Cref{lem:small_loss_imply_small_distance}, we have that $\norm{\theta-\theta^*}_2\le R$. By \Cref{assum:hessian_multiplicative_lipschitz}, we have $\nabla^2 L(\theta)\succeq \frac{1}{2}\nabla^2 L(\theta^*) \succeq \frac{\mu}{2} I_d$. 
	By \Cref{lem:small_loss_or_gradient_imply_loss_approx}, we have that 
	\begin{align}
	4(L(\theta) - \min L) \ge &	\nabla L(\theta)^\top (\nabla^2 L(\theta))^{-1}\nabla L(\theta)\\
	\ge &\nabla L(\theta)^\top (\nabla^2 L(\theta))^{-1} \nabla^2 L(\theta)(\nabla^2 L(\theta))^{-1}\nabla L(\theta)\\
	\ge & \frac{\mu}{2}\norm{\nabla L(\theta)^\top (\nabla^2 L(\theta))^{-1}}_2^2.
		\end{align}
This completes the proof.
\end{proof}

\begin{lemma}\label{lem:no_clipping_imply_loss_approximation}
	For any $\theta\in\R^d$ satisfying that $\norm{((\nabla^2 L(\theta))^{-1}\nabla L(\theta)}_2\le \frac{R}{2}$, it holds that 
	\begin{align}
L(\theta) - \min L	\le \nabla L(\theta)^\top (\nabla^2 L(\theta))^{-1}\nabla L(\theta) \le  4(L(\theta) - \min L).
 \end{align}
\end{lemma}
\begin{proof}[Proof of \Cref{lem:no_clipping_imply_loss_approximation}]
	Let $\{\theta(t)\}_{t=0}^1$ be the solution of \Cref{eq:mirror_ode} and we claim that for all $t\in[0,1]$, $\norm{\theta(t) -\theta}_2\le R$. Otherwise, let $T$ be the smallest positive number such that $\norm{\theta(T) -\theta}_2=R$. Such $T$ exists because $\norm{\theta(t) -\theta}_2$ is continuous in $t$ and $\norm{\theta(0) -\theta}_2=0$. We have that 
	\begin{align}
	R=&\norm{\theta(T)-\theta(0)}_2 \le \int_{t=0}^T \norm{\frac{\diff \theta(t)}{\diff t}}_2\diff t \\
	=&\int_{t=0}^T \norm{((\nabla^2 L(\theta(t)))^{-1}\nabla L(\theta)}_2\diff t 	\\
	\le &  \int_{t=0}^T \norm{(\nabla^2 L(\theta(t)))^{-1} \nabla^2 L(\theta)}_2\norm{((\nabla^2 L(\theta))^{-1}\nabla L(\theta)}_2\diff t 	 \\
	\le & 2 \int_{t=0}^T \norm{((\nabla^2 L(\theta))^{-1}\nabla L(\theta)}_2\diff t 	\label{eq:local_1} \\
	\le &2 T\frac{R}{2} =RT,
	\end{align}
which implies $T=1$. Here in \Cref{eq:local_1}, we use \Cref{assum:hessian_multiplicative_lipschitz}. Thus we conclude that for all $t\in[0,1]$, $\norm{\theta(t) -\theta}_2\le R$.  By~\Cref{assum:hessian_multiplicative_lipschitz}, it holds that 
\begin{align}\label{eq:no_clipping_hessian_sanwidch}
	2(\nabla^2 L(\theta))^{-1} \succeq (\nabla^2 L(\theta(t)))^{-1}\succeq \frac{1}{2}(\nabla^2 L(\theta))^{-1}.
\end{align}

 Therefore, we have that 
\begin{align}\label{eq:no_clipping_loss_integral}
L(\theta) - \min L = L(\theta(0))-L(\theta(1)) 
=&\int_{t=0}^1  	(\nabla L(\theta(t)))^\top (\nabla^2 L(\theta(t)))^{-1} \nabla L(\theta)\notag\\
=&\int_{t=0}^1 	(1-t)	(\nabla L(\theta))^\top (\nabla^2 L(\theta(t)))^{-1} \nabla L(\theta).
\end{align}
The proof is completed by plugging \Cref{eq:no_clipping_hessian_sanwidch} into \Cref{eq:no_clipping_loss_integral} and noting that $\int_{t=0}^1(1-t) = 1/2$.
\end{proof}

\begin{lemma}\label{lem:small_decrement_imply_small_loss}
	If $\rho \le \frac{R}{2\sqrt{d}}$, then for any $\Delta\le \frac{R\rho\mu}{10}$  and any $\theta\in\R^d$ satisfying
	 \begin{align}
\sum_{i=1}^d \min\{ \rho \abs{v_i^\top \nabla L(\theta)}, \sigma_i^{-1}\abs{v_i^\top \nabla L(\theta)}^2\} \le  \Delta,
	 \end{align}
 where $\nabla^2  L(\theta) = V^\top \Sigma V$	  is the eigendecomposition of $\nabla ^2 L(\theta)$, $v_i$ is the $i$th row of $V$ and $\Sigma = \diag(\sigma_1,\ldots,\sigma_d)$, it holds that 
 \begin{align}
     L(\theta) -\min L\le \Delta+ \frac{25\Delta^2}{\rho^2\mu}
 \end{align}
 
 In particular, if we set $\Delta\triangleq \frac{\mu\rho^2}{20}$, we have $L(\theta) -\min L\le \frac{\mu\rho^2}{8}$.
\end{lemma}

\begin{proof}[Proof of \Cref{lem:small_decrement_imply_small_loss}]
	Let $I_\theta\triangleq \{i\in[d]\mid \abs{v_i^\top \nabla L(\theta)}\sigma_i^{-1} \le \rho \}$ be the set of indices where clipping does not happen. Then we have that 
	\begin{align}
		\sum_{i\in I_\theta}  \sigma_i^{-1}\abs{v_i^\top \nabla L(\theta)}^2\le \Delta\\
		\sum_{i\notin I_\theta}  \rho \abs{v_i^\top \nabla L(\theta)}\le \Delta
	\end{align}
Now we consider a new strictly convex loss function in $R^{|I_\theta|}$, which is $L$ restricted on the space of $\{\theta + \sum_{i\in I_\theta} w_{[i]}v_i\mid w\in\R^{|I_\theta|}\}$, that is, $ L_\theta(w) = L(\theta+ \sum_{i\in I_\theta}w_{[i]}v_i)$. This new loss function ${L_\theta}$ clearly satisfy \Cref{assum:hessian_multiplicative_lipschitz} since it is a restriction of $L$ into some subspace of $\R^d$. By  \Cref{lem:strictly_convex_loss}, we know that $\inf_w L_\theta(w)$ can be attained and we denote it by $w^*$. By \Cref{assum:existence_local_minimizer_w_hessian_pd}, we know that $L_\theta$ is strictly convex and thus $\nabla^2 L_\theta(w)\succ 0$, which means \Cref{assum:existence_local_minimizer_w_hessian_pd} also holds for $L_\theta$. 

Next we will apply \Cref{lem:no_clipping_imply_loss_approximation} on $L_\theta$ at $w=0$. We  use $V_{I_\theta}\in\R^{|I|\times d}$ to denote the submatrix of $V$ containing rows in $I$ for any $I\subset [d]$. One can verify by chain rule that $\nabla L_\theta(w) = V_{I_\theta} \nabla L(\theta + V_{I_\theta}^\top w)$ and that $\nabla^2 L_\theta(w) = V_{I_\theta} \nabla^2 L(\theta + V_{I_\theta}^\top w) V_{I_\theta}^\top$. Thus we have that 
\begin{align}
(\nabla^2 L_\theta(0)	)^{-1}\nabla L_\theta(0) = V_{I_\theta}(\nabla^2L(\theta))^{-1}\nabla L(\theta).
\end{align}
By the definition of $I_\theta$, we know that $\norm{V_{I_\theta}(\nabla^2L(\theta))^{-1}\nabla L(\theta)}\infty\le \rho$.  Thus $\norm{(\nabla^2 L_\theta(0)	)^{-1}\nabla L_\theta(0)}_2 \le \sqrt{d} \norm{V_{I_\theta}(\nabla^2L(\theta))^{-1}\nabla L(\theta)}_\infty = \sqrt{d}\cdot \rho \le \frac{R}{2}$. Thus we can apply \Cref{lem:no_clipping_imply_loss_approximation} on $L_\theta$ at $w=0$ and conclude that 
\begin{align}\label{eq:intermediate_loss_approx}
	L_\theta(0)-  L_\theta(w^*) 
	\le \nabla L_\theta(0)^\top (\nabla^2 L_\theta(0))^{-1}\nabla L_\theta(0)
	=  \sum_{i\in I_\theta}  \sigma_i^{-1}\abs{v_i^\top \nabla L(\theta)}^2
	\le  \Delta 
\end{align}
Thus $L(\theta) - L(\theta+ V_{I_\theta}^\top w^*)   = L_\theta(0)- L_\theta(w^*)\le \Delta $.

It remains to show that $L(\theta+ V_{I_\theta}^\top w^*)-L(\theta^*)\le \frac{25\Delta^2}{\rho^2\mu}$. To do so, our strategy is to first show that $\norm{\nabla L(\theta+V_{I_\theta}^\top w^*)}_2$ is small and then to use \Cref{lem:small_loss_or_gradient_imply_PL}. We will use $I_\theta^c$ to denote the complement of $I_\theta$ in $[d]$ and $V_{I_\theta^c}\in\R^{(d-|I_\theta|)\times d}$ to denote the submatrix of $V$ which contains all the rows that do not belong to $I_\theta$. Note that $w^*$ is the minimizer of $L_\theta$, we know that $V_{I_\theta} \nabla L(\theta+V_{I_\theta}^\top w^*) = 0$ and that $\norm{\nabla L(\theta+V_{I_\theta}^\top w^*)}_2 = \norm{V_{I_\theta^c}\nabla L(\theta+V_{I_\theta}^\top w^*)}_2$. 

Now we consider the following ODE 
	\begin{align}\label{eq:mirror_ode}
		\frac{\diff w(t)}{\diff t} = - (\nabla^2 L_\theta(w(t)))^{-1} \nabla L_\theta(0), \quad w(0)=0.
	\end{align}
By \Cref{lem:soln_existence_mirror_ode}, we know this ODE has solution $w(t)$ over interval $[0,1]$ with $w(1) = w^*$. With the same argument in the proof of \Cref{lem:no_clipping_imply_loss_approximation}, we know that $\norm{w(t)}_2\le R$ for all $t\in[0,1]$. Thus we have for any $t\in[0,1]$,
\begin{align}\label{eq:gradient_increasing_speed}
	&\norm{V_{I_\theta^c}\frac{\diff \nabla L(\theta+ V_{I_\theta} w(t))}{\diff t}}_2 \\\
	= &\norm{V_{I_\theta^c}  \nabla^2 L(\theta+ V_{I_\theta} w(t))V_{I_\theta}(\nabla^2 L_\theta(w(t)))^{-1} \nabla L_\theta(0)}_2 \\
	= &  \norm{V_{I_\theta^c}  \nabla^2 L(\theta+ V_{I_\theta} w(t))V_{I_\theta}V_{I_\theta}^\top(\nabla^2 L(\theta+ V_{I_\theta} w(t)))^{-1} \nabla L(\theta)}_2\\
	\le & \norm{V_{I_\theta^c}  \sqrt{\nabla^2 L(\theta+ V_{I_\theta} w(t))}}_F \label{eq:term1}\\
	\cdot & \norm{\sqrt{\nabla^2 L(\theta+ V_{I_\theta} w(t))}V_{I_\theta}V_{I_\theta}^\top(\nabla^2 L(\theta+ V_{I_\theta} w(t)))^{-1} \nabla L(\theta) }_2\label{eq:term2}
	\end{align}
	
For the first term (\Cref{eq:term1}), by \Cref{assum:hessian_multiplicative_lipschitz}, we have that 
\begin{align}
\norm{V_{I_\theta^c}  \sqrt{\nabla^2 L(\theta+ V_{I_\theta} w(t))}}_F ^2 \le 2V_{I_\theta^c} \nabla^2 L(\theta) V_{I_\theta^c} = 2\sum_{i\notin I_\theta} \sigma_i \le 2\sum_{i\notin I_\theta} \frac{v_i^\top\nabla L(\theta)}{\rho}\le \frac{2\Delta}{\rho^2}.
\end{align}

For the second term (\Cref{eq:term2}), by \Cref{assum:hessian_multiplicative_lipschitz}, we have that 
\begin{align}
&\norm{\sqrt{\nabla^2 L(\theta+ V_{I_\theta} w(t))}V_{I_\theta}V_{I_\theta}^\top(\nabla^2 L(\theta+ V_{I_\theta} w(t)))^{-1} \nabla L(\theta) }_2^2 \\
\le &8 \norm{\sqrt{\nabla^2 L(\theta)}V_{I_\theta}V_{I_\theta}^\top(\nabla^2 L(\theta))^{-1} \nabla L(\theta) }_2^2\\
= &8 \nabla L(\theta) ^\top V_{I_\theta}V_{I_\theta}^\top(\nabla^2 L(\theta))^{-1} V_{I_\theta}V_{I_\theta}^\top\nabla L(\theta) \\
= & 8\sum_{i\in I_\theta}  \sigma_i^{-1}\abs{v_i^\top \nabla L(\theta)}^2 \le 8\Delta.
\end{align}
Thus we conclude that $\norm{V_{I_\theta^c}\frac{\diff \nabla L(\theta+ V_{I_\theta} w(t))}{\diff t}}_2\le  \frac{4\Delta}{\rho}$, which implies that 
\begin{align}
&\norm{\nabla L(\theta+V_{I_\theta}^\top w^*)}_2 = \norm{V_{I_\theta^c}\nabla L(\theta+V_{I_\theta}^\top w^*)}_2\\
 = &\norm{V_{I_\theta^c}\nabla L(\theta)+\int_{t=0}^1 V_{I_\theta^c}\frac{\diff \nabla L(\theta+ V_{I_\theta} w(t))}{\diff t}\diff t}_2\\
 \le & \norm{V_{I_\theta^c}\nabla L(\theta)}_2 +\int_{t=0}^1\norm{ V_{I_\theta^c}\frac{\diff \nabla L(\theta+ V_{I_\theta} w(t))}{\diff t}}_2\diff t \\
 \le & \frac{\Delta}{\rho} + \frac{4\Delta}{\rho} = \frac{5\Delta}{\rho}. 	
\end{align}

Applying \Cref{lem:small_loss_or_gradient_imply_PL}, we have that 
\begin{align}
	L(\theta+V_{I_\theta}^\top w^*) -\min L\le \mu^{-1} \norm{\nabla L(\theta+V_{I_\theta}^\top w^*)}_2^2 = \frac{25\Delta^2}{\rho^2\mu}.
\end{align}
This completes the proof.
\end{proof}

\begin{lemma}[Descent Lemma]\label{lem:descent_lemma}
	For any $\eta,\rho>0$ with $\eta\rho\le R/\sqrt{d}$, $\theta\in\R^d$ and any eigendecomposition of 
$\nabla^2 L(\theta)$, where $V_tV_t^\top=I_d$, $\sigma_t$ is diagonal $\nabla^2 L(\theta) = V^\top \Sigma V$, define 
	\begin{align}
	\theta_+ \triangleq 	 \theta - \eta V^\top\clip(V(\nabla^2L(\theta))^{-1}\nabla L(\theta),\rho),
	\end{align}
it holds that 
	\begin{align}
			L(\theta_+) - L(\theta)\le -(\eta-\eta^2)\sum_{i=1}^d \min\{ \rho \abs{v_i^\top \nabla L(\theta)}, \sigma_i^{-1}\abs{v_i^\top \nabla L(\theta)}^2\},
	\end{align}
where $v_i$ is the $i$th row of matrix $V$.
\end{lemma}

\begin{proof}[Proof of \Cref{lem:descent_lemma}]
	Let $ u \triangleq \clip(V(\nabla^2L(\theta))^{-1}\nabla L(\theta),\rho)$. By the definition of $\clip$ operation, we know that $\norm{V^\top u}_2 = \norm{u}_2 \le \sqrt{d} \rho$. Thus  we have $\norm{\theta_+ - \theta}  = \eta \norm{V^\top u}_2 \le \eta\rho\sqrt{d}$. Define $f(t) = L(t\theta_+ + (1-t)\theta)$. By Assumption~\ref{assum:hessian_multiplicative_lipschitz}, we know that $f''(t)\le 2f''(0)$ for all $t\in[0,1]$ and thus 
	\begin{align}
	f(1) = f(0) + f'(0) + \int_{s=0}^1\int_{t=0}^s f''(s)\diff s\diff t	\le f(0)+f'(0)+f''(0).
	\end{align}
	It remains to show that 
	\begin{enumerate}
		\item $f'(0) = -\eta \sum_{i=1}^d \min\{ \rho \abs{v_i^\top \nabla L(\theta)}, \sigma_i^{-1}\abs{v_i^\top \nabla L(\theta)}^2\}$;
		\item $f''(0) \le  \eta^2 \sum_{i=1}^d \min\{ \rho \abs{v_i^\top \nabla L(\theta)}, \sigma_i^{-1}\abs{v_i^\top \nabla L(\theta)}^2\}$;
	\end{enumerate}
	First, by chain rule, we have $f'(0) = \inner{\nabla L(\theta)}{-\eta V^\top u}=\inner{V\nabla L(\theta)}{-\eta u} = -\eta \inner{V\nabla L(\theta)}{ \clip(\Sigma^{-1}V\nabla L(\theta),\rho)}=-\eta \sum_{i=1}^d \min\{ \rho \abs{v_i^\top \nabla L(\theta)}, \sigma_i^{-1}\abs{v_i^\top \nabla L(\theta)}^2\}$.\\
	Second, again by chain rule, we have $f''(0) = \eta^2 \inner{V^\top u}{\nabla^2L(\theta) V^\top u} = \eta^2 \inner{u}{\Sigma u} = \sum_{i=1}^d \abs{u_i}^2\sigma_i$. Note that by definition $\abs{u_i}=\min\{\abs{v_i^\top \nabla L(\theta)}/\sigma_i,\rho\}$, we have $\abs{u_i}^2\sigma_i\le \min\{\abs{v_i^\top \nabla L(\theta)}/\sigma_i,\rho\}\cdot \abs{v_i^\top \nabla L(\theta)}/\sigma_i \cdot \sigma_i = \min\{\abs{v_i^\top \nabla L(\theta)}^2/\sigma_i,\rho\abs{v_i^\top \nabla L(\theta)}\}$, which completes the proof.
\end{proof}

\begin{lemma}\label{lem:2nd_phase_convex}
	If $\eta\rho\le R/\sqrt{d}$ and for some $T\in\mathbb{N}$, $L(\theta_T)-\min L\le \frac{\mu\rho^2}{8}$, then if holds that for all $t\ge T$, 
	\begin{enumerate}
		\item $\theta_{t+1} = \theta_t - \eta (\nabla^2 L(\theta_t))^{-1}\nabla L(\theta_t)$;
		\item $L(\theta_t)-\min L \le (1-\eta(1-\eta))^{t-T}(L(\theta_T)-\min L)$.
	\end{enumerate}
\end{lemma}
\begin{proof}[Proof of \Cref{lem:2nd_phase_convex}]
	First by \Cref{lem:descent_lemma}, we have for all $t\ge T$, $\L(\theta_t) - \min L\le L(\theta_T)-\min L\le \frac{\mu\rho^2}{8}$, therefore by \Cref{lem:small_loss_imply_no_clipping}, we have $ \norm{(\nabla^2 L(\theta_t))^{-1}\nabla L(\theta_t)}_2 \le \rho$ for all $t\ge T$, which implies clipping will not happen. This completes the proof of the first claim.
	
	For the second claim, by \Cref{lem:descent_lemma,lem:small_loss_or_gradient_imply_loss_approx}, we have that 
	\begin{align}
		L(\theta_{t+1}) - L(\theta_t)
		\le& -(\eta-\eta^2)\sum_{i=1}^d \sigma_i^{-1}\abs{v_i^\top \nabla L(\theta_t)}^2 \\
		= &-(\eta-\eta^2)\nabla L(\theta_t) (\nabla^2 L(\theta_t))^{-1}\nabla L(\theta_t)\\
		\le & - \eta(1-\eta)( L(\theta_t)-\min L),
	\end{align}
which completes the proof.
\end{proof}

\newcommand{\nL}{L_{\mu,\beta}}
\subsection{Lower bound for SignGD on 2-dimensional quadratic loss} \label{sec:proofs:lowbound}

Define $L_{\mu,\beta}:\R^2\to\R$ as a quadratic function with parameter $\mu,\beta$ as $L_{\mu,\beta}(\theta)\triangleq \frac{\mu}{2}\theta_{[1]}^2 + \frac{\beta}{2}\theta_{[2]}^2$.  We have the following lower bound, which shows signGD's convergence rate has to depend on the condition number $\beta/\mu$. 
\begin{theorem}\label{thm:signgd_lower_bound}
    For any $\mu,\beta,\Delta,\eps>0$, suppose there exist a learning rate $\eta$ and a time $T$ such that for all $\theta_0$ satisfying that $L_{\mu,\beta}(\theta_0)\le \Delta$, signGD reaches loss at most $\eps$ at step $T-1$ and $T$ (in the sense that $\nL(\theta_{T})\le \eps$ and $\nL(\theta_{T-1})\le \eps$). Then, $T$ must satisfy $T\ge \frac{1}{2}(\sqrt{\frac{\Delta}{\epsilon}}-\sqrt{2})\sqrt{\frac{\beta}{\mu}}$.
\end{theorem}

\begin{proof}[Proof of \Cref{thm:signgd_lower_bound}]
    We consider two initialization: $\theta_0 = (0,\sqrt{\frac{2\Delta}{\beta}})$ and $\theta'_0 = (\sqrt{\frac{2\Delta}{\mu}},0)$, and let $\theta_t$ and $\theta_t'$ be the iterates under the two initializations. For each coordinate $i\in\{1,2\}$, because $|(\theta_t)_{[i]}- (\theta_{t+1})_{[i]}|=\eta$, we have that $|(\theta_t)_{[i]}| + |(\theta_{t+1})_{[i]}|\ge \eta$. Thus $2\eps\ge \nL(\theta_T)+ \nL(\theta_{T-1}) \ge \frac{\beta}{2}((\theta_T)_{[2]}^2 + (\theta_{T-1})_{[2]}^2)\ge \frac{\beta\eta^2}{4}$, which implies $\eta \le \sqrt{\frac{8\eps}{\beta}}$.

    The fact that $\nL(\theta_T')+ \nL(\theta_{T-1}')\le 2\eps$ implies $(\theta_{T}')_{[1]}\le \sqrt{\frac{4\eps}{\mu}}$. Because SignGD can only move each coordinate by $\eta$ at most, we have $(T-1)\eta\ge \sqrt{2\Delta/\mu} - \sqrt{\frac{4\eps}{\mu}}$. Using the fact that $\eta \le \sqrt{\frac{8\eps}{\beta}}$, we have that $2(T-1)\ge (\sqrt{\frac{\Delta}{\epsilon}}-\sqrt{2})\sqrt{\frac{\beta}{\mu}}$, which completes the proof.
\end{proof}